\documentclass[10pt]{article}
\usepackage[left=1in,bottom=1in,top=1in,right=1in]{geometry}

\title{Understanding Gradient Descent on  Edge of Stability \\in Deep Learning}
\usepackage{times}
\usepackage{math_commands}

\usepackage[colorlinks,
            linkcolor=red,
            anchorcolor=blue,
            citecolor=blue
            ]{hyperref} 
\usepackage{url}            
\usepackage{booktabs}       
\usepackage{amsfonts}       
\usepackage{nicefrac}       
\usepackage{xcolor}         

\usepackage{amsmath, amssymb, amsthm, bm}
\usepackage{multirow}
\usepackage[ compress]{natbib}
\usepackage{graphicx}
\usepackage{makecell}
\usepackage{array}
\usepackage{algorithm}
\usepackage{algorithmic}
\usepackage{dsfont}
\usepackage{csquotes}
\usepackage{cancel}
\usepackage[capitalize,noabbrev]{cleveref}
\usepackage{enumitem}
\usepackage{thmtools} 
\usepackage{thm-restate}
\usepackage{caption}
\usepackage{subcaption}
\usepackage{nicefrac}
\usepackage[disable]{todonotes}

\theoremstyle{plain}
\newtheorem{theorem}{Theorem}[section]
\newtheorem{lemma}[theorem]{Lemma}

\newtheorem{corollary}[theorem]{Corollary}
\theoremstyle{definition}
\newtheorem{definition}[theorem]{Definition}

\newcommand{\abhishek}[1]{{\color{red}[AP: #1]}}
\newcommand{\zhiyuan}[1]{{\color{red}[ZL: #1]}}
\renewcommand{\abhishek}[1]{}
\renewcommand{\zhiyuan}[1]{}

\newcommand{\tL}{\widetilde{L}}

\author{%
	Sanjeev Arora \qquad Zhiyuan Li \qquad Abhishek Panigrahi \\
	Department of Computer Science\\
	Princeton University\\
	\texttt{\{arora,zhiyuanli,ap34\}@cs.princeton.edu} \\
}

\begin{document}

\date{}
\maketitle
\sloppy

\begin{abstract}
\normalsize
    Deep learning experiments by \citet{cohen2021gradient} using deterministic Gradient Descent (GD) revealed an {\em Edge of Stability (EoS)} phase when learning rate (LR) and sharpness (\emph{i.e.}, the largest eigenvalue of Hessian) no longer behave as in traditional optimization. Sharpness stabilizes around $2/$LR and loss goes up and down  across iterations, yet still with an overall downward trend. The current paper mathematically analyzes a new mechanism of implicit regularization in the EoS phase, whereby GD updates due to non-smooth loss landscape  turn out to evolve along some deterministic flow on the manifold of minimum loss. This is in contrast to many previous results about implicit bias either relying on infinitesimal updates or noise in gradient. Formally, for any smooth function $L$ with certain regularity condition, this effect is demonstrated for (1) {\em Normalized GD}, i.e., GD  with a varying 
   LR $\eta_t =\frac{\eta}{\norm{\nabla L(x(t))}}$ and loss $L$;  (2) GD with constant LR and loss $\sqrt{L- \min_x L(x)}$.  Both provably enter the Edge of Stability, with the associated flow on the manifold minimizing $\lambda_{1}(\nabla^2 L)$. The above theoretical results have been corroborated by an experimental study.
\end{abstract}



\section{Introduction}
Traditional convergence analyses of gradient-based algorithms assume learning rate $\eta$ is set according to the basic relationship $\eta < 2/\lambda$ where $\lambda$ is the largest eigenvalue of the Hessian of the objective, called {\em sharpness}\footnote{Confusingly, another traditional name for $\lambda$ is {\em smoothness}.}. {\em Descent Lemma} says that if this relationship holds along the trajectory of Gradient Descent, loss drops during each iteration. In deep learning where objectives are nonconvex and have multiple optima, similar analyses can show convergence towards stationary points and local minima.  In practice, sharpness is unknown and $\eta$ is set by trial and error. Since deep learning works, it has been generally assumed that this trial and error allows $\eta$ to adjust to sharpness so that the theory applies. But recent empirical studies~\citep{cohen2021gradient,ahn2022understanding} showed compelling evidence to the contrary. On a variety of popular architectures and training datasets, GD with fairly small values of $\eta$ displays following phenomena that they termed {\em Edge of Stability (EoS)}: (a) Sharpness rises beyond  $2/\eta$, thus violating the above-mentioned relationship. (b) Thereafter  sharpness stops rising but hovers noticeably above $2/\eta$ and even decreases a little. (c) Training loss behaves non-monotonically over individual iterations, yet consistently decreases over long timescales. 

Note that (a) was already pointed out by~\citet{li2020reconciling}. Specifically, in modern deep nets, which use some form of normalization combined with weight decay, training to near-zero loss must lead to arbitrarily high sharpness.  (However, \citet{cohen2021gradient} show that the EoS phenomenon appears even without normalization.) Phenomena (b), (c) are more mysterious, suggesting that GD with finite $\eta$ is able to continue decreasing loss despite violating $\eta < 2/\lambda$, while at the same time regulating further increase in value of sharpness and even causing a decrease. These striking inter-related phenomena suggest a radical overhaul  of our thinking about optimization in deep learning. At the same time, it appears mathematically challenging to  analyze such phenomena, at least for realistic settings and losses (as opposed to toy examples with 2 or 3 layers). The current paper introduces frameworks for doing such analyses.


We start by formal  definition of stableness, ensuring that if a point + LR combination is stable then 
a gradient step is guaranteed to decrease the loss by the local version of Descent Lemma. 

\begin{definition}[Stableness]\label{defi:stableness}
Given a loss function $L$,  a parameter $x\in\RR^d$ and LR $\eta>0$ we define the \emph{stableness} of $L$ at $(x,\eta)$ be $\stableness{L}{x}{\eta}:=\eta\cdot\sup_{0\le s\le \eta} \lambda_{1}(\nabla^2 L(x-s\nabla L(x)))$. We say $L$ is \emph{stable} at $(x,\eta)$  iff the stableness of $L$ at $(x,\eta)$ is smaller than or equal to  $2$; otherwise we say $L$  is \emph{unstable} at $(x,\eta)$. 
\end{definition}

The above defined stableness is a better indicator for EoS than only using the sharpness at a specific point $x$, \emph{i.e.} $\eta\lambda_{1}(\nabla^2 L(x))<2$, because the loss can still oscillate in the latter case. \footnote{ See such experiments (e.g., ReLU CNN (+BN), Figure 75) in Appendix of in \citet{cohen2021gradient}.}  A concrete example is $L(x) = |x|, x\in\RR$. For any $c \in (0,1)$ and LR $\eta>0$, the GD iterates $x(2k)=c\eta$ and $x(2k+1) =-(1-c)\eta$,  always have zero  sharpness for all $k\in\NN$, but  Descent Lemma doesn't apply because the gradient is not continuous around $x=0$ (\emph{i.e.} the sharpness is infinity when $x=0$). As a result, the loss is not stable and oscillates between $c\eta$ and $(1-c)\eta$.

\begin{figure*}[t]
\centering
 \begin{subfigure}[b]{0.49\textwidth}
     \centering
     \includegraphics[width=\textwidth]{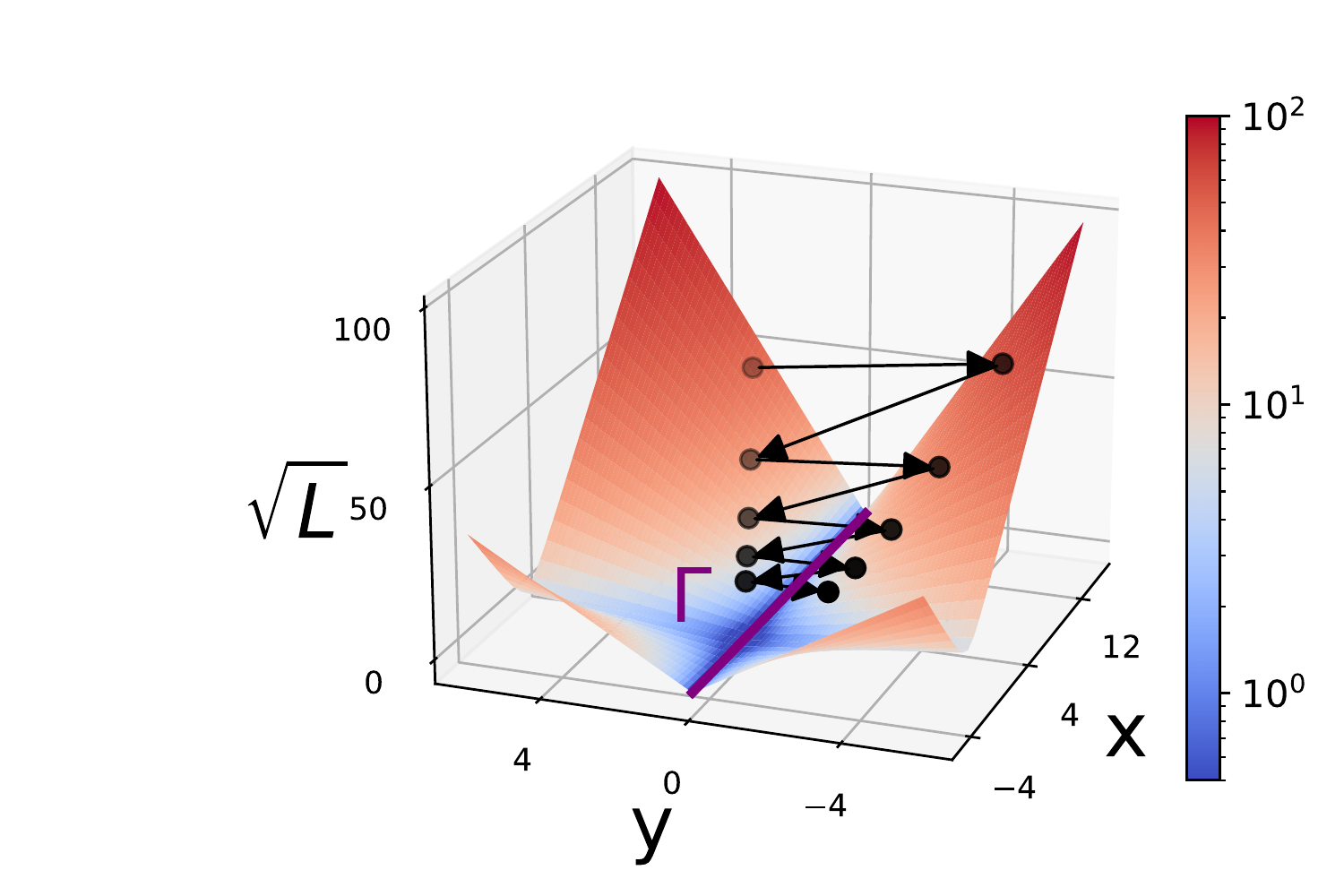}
     \caption{ GD on $\sqrt{L}$}
     \label{subfig:sqrt_L}
 \end{subfigure}
  \begin{subfigure}[b]{0.49\textwidth}
     \centering
     \includegraphics[width=\textwidth]{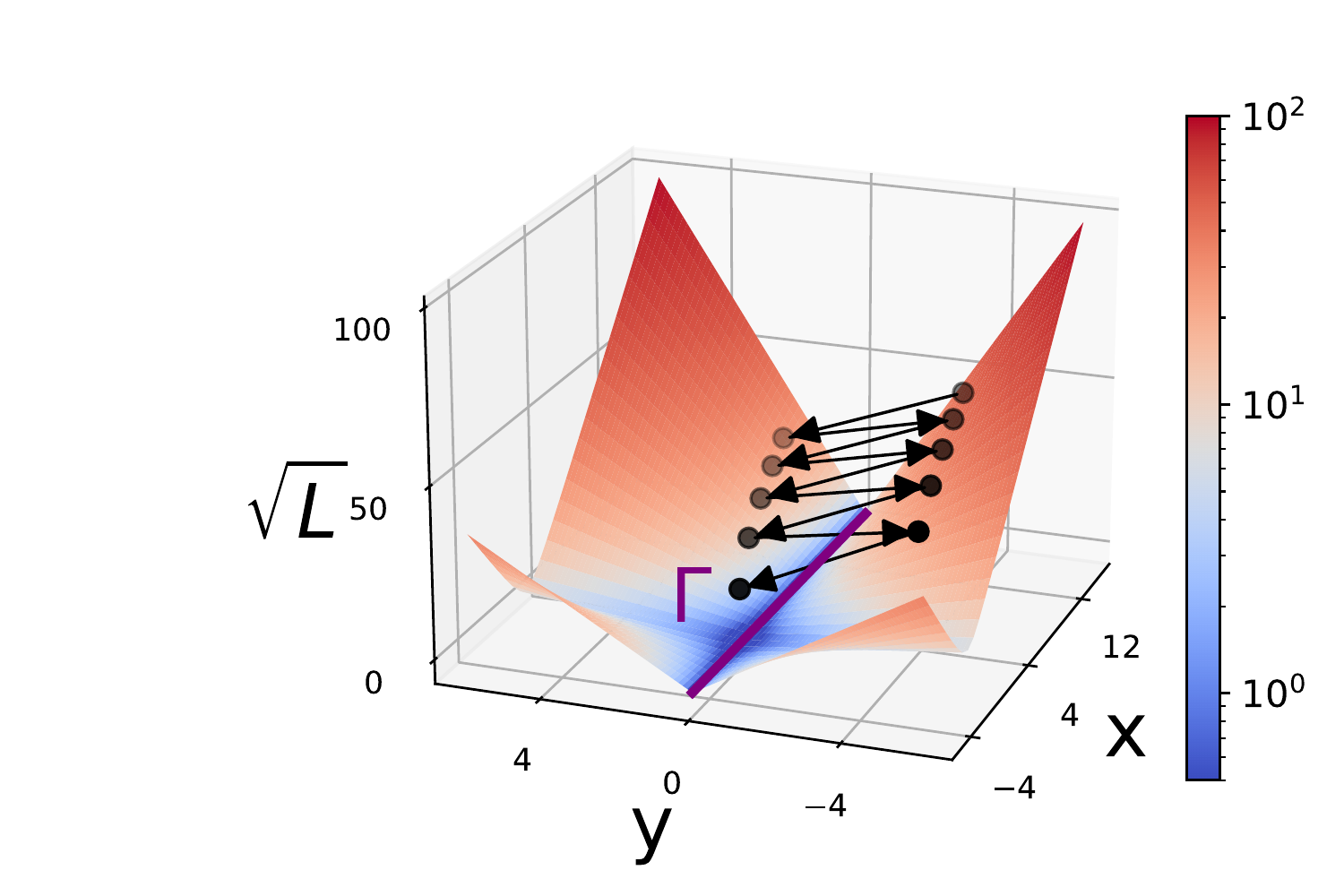}
     \caption{ Normalized GD on $L$}
     \label{subfig:ngd}
 \end{subfigure}
\caption{GD operating on EoS  oscillates around the zero loss manifold $\Gamma = \{(x,y)\mid y=0\}$ while slowly moving towards flatter local minima. Here $L(x,y)=(1+x^2)y^2$ and the sharpness of $L$ decreases as $|x|$ decreases. 
}
\label{fig:ngd_traj_3d}
\end{figure*}

\subsection{Two Provable Mechanisms for Edge of Stability: Non-smoothness and Adaptivity}


In this paper we identify two settings where GD provably operates on Edge of Stability. The intuition is from \Cref{defi:stableness}, which suggests that either sharpness or learning rate has to increase to avoid GD converge and stays at Edge of Stability. 

 The first setting, which is simple yet quite general, is to consider a modified training loss $f(L)$ where $f:\RR \to \RR$ is a monotone increasing but non-smooth function. For concreteness, assume GD is performed on $\tL:=\sqrt{L}$ where $L$ is a smooth loss function with $\min_x L(x) =0$ and   $\nabla^2 L\neq 0$ at its minimizers. Note that $\nabla \tL = \frac{\nabla L}{2\sqrt{L}}$ and $\nabla^2 \tL = \frac{2L\nabla^2 L - \nabla L \nabla L^\top}{4\sqrt{L}^3}$, which implies $\nabla^2 \tL$ must diverge whenever $x$ converges to any minimizer where $\nabla^2 L$ has rank at least $2$, since $\nabla L \nabla L^\top$ is rank-$1$. (An analysis is also possible when $\nabla^2 L$ is rank-$1$, which is the reason for \Cref{defi:stableness}.) 

The second setting assumes that the loss is smooth but learning rate is effectively adaptive. We focus a concrete example, {\em Normalized Gradient Descent}, $x \leftarrow x -\eta \nabla L/\|\nabla L\|$, which exhibits EoS behavior as  $\nabla L \rightarrow 0$.
 We can view Normalized GD as GD with a varying LR $\eta_t = \frac{\eta}{\norm{\nabla L(x(t))}}$, which goes to infinity when $\nabla L \rightarrow 0$.


 These analyses will  require (1) The zero-loss solution set $ \{x\mid L(x) =0\}$  \footnote{Without loss of generality,  we  assume  $\min_{x'} L(x')= 0$ throughout the paper. The main results for Normalized GD still hold if we relax the assumption and only assume $\Gamma$ to be a manifold of local minimizers. For GD on $\sqrt{L}$, we need to replace $\sqrt{L}$ by $\sqrt{L} - \sqrt{L_{\min}}$ where $L_{\min}$ is the local minimum.} contains  a $(D-M)$ dimensional submanifold of $\RR^D$ for some $1\le M\le D$ and we denote it by $\Gamma$ and (2)  $\nabla^2 L(x)$ is rank-$M$ for any $x\in\Gamma$. Note that while modern deep learning evolved using non-differentiable losses, the recent use of activations such as Swish \citep{ramachandran2017searching} instead of ReLU has allowed differentiable losses without harming performance.

\paragraph{Our Contribution:} We  show that Normalized GD   on $L$ (\Cref{sec:ngd_mainpaper}) and GD on $\sqrt{L}$ (\Cref{sec:gdsqrtL_mainpaper}) exhibit similar two-phase dynamics with sufficiently small LR $\eta$. In the first phase, GD tracks gradient flow (GF), with a monotonic decrease in loss until getting $O(\eta)$-close to the manifold (\Cref{thm:ngd_phase_1,thm:sqrtL_phase_1}) and the stableness becomes larger than $2$. In the second  phase, GD no longer tracks GF and loss is not monotone decreasing due to the high stableness. Repeatedly overshooting, GD iterate jumps back and forth across the manifold while moving slowly along the direction in the tangent space of the manifold which decreases the sharpness. (See \Cref{fig:ngd_traj_3d} for a graphical illustration) Formally, we prove when $\eta\to 0$, the trajectory of GD converges to some limiting flow on the manifold. (\Cref{thm:ngd_phase_2,thm:PhaseII_sqrtL}) We further prove that in both settings GD in the second phase  operates on EOS, and loss decreases in a non-monotone manner. Formally, we show that  the average stableness over any two consecutive steps is at least 2 and that the average of $\sqrt{L}/\eta$ over  two consecutive is proportional to sharpness or square root of sharpness. (\Cref{thm:stableness_ngd,thm:stableness_gd_sqrt_L})

Though many works have suggested (primarily via experiments and some intuition) that the training algorithm in deep learning implicitly selects out solutions of low sharpness in some way, we are not aware of a formal setting where this had ever been made precise.  Note that our result requires no stochasticity as in SGD~\citep{li2022what}, though we need to inject tiny noise (\emph{e.g.,} of magnitude $O(\eta^{100})$ ) to GD iterates occasionally (\Cref{alg:PNGD,alg:PsqrtL}). We believe that this is due the technical limitation of our current analysis and can be relaxed with a more advanced analysis. Indeed, in experiments, our theoretical predictions hold for the deterministic GD directly without any perturbation.

\paragraph{Novelty of Our Analysis:} Our analysis is inspired by the mathematical framework of studying limiting dynamics of SGD around manifold of minimizers by \citet{li2022what}, where the high-level idea is to introduce a projection function $\Phi$ mapping the current iterate $x_t$ to the manifold and it suffices to understand the dynamics of $\Phi(x_t)$. It turns out that the one-step update of $\Phi(x_t)$ depends on the second moment of (stochastic) gradient at $x_t$, $\EE[\nabla L(x_t)(\nabla L(x_t))^\top]$. While for SGD the second moment converges to the covariance matrix of stochastic gradient  \citep{li2022what} as   $x_t$ gets close to the manifold when $\eta\to 0$, for GD operating on EOS, the updates $\nabla \sqrt{L}(x_t)$ or $\frac{\nabla L(x_t)}{\norm{\nabla L(x_t)}}$ is non-smooth and not even defined at the manifold of the minimizers! To show $\Phi(x_t)$ moves in the direction which decreases the sharpness, the main technical difficulty is to show that $\nabla \sqrt{L}(x_t)$ or $\frac{\nabla L(x_t)}{\norm{\nabla L(x_t)}}$  aligns to the top eigenvector of the Hessian $\nabla^2 L(x_t)$ and then the analysis follows from the framework by \citet{li2022what}. 

To prove the alignment between the  gradient and the top eigenvector of Hessian, it boils down to analyze Normalized GD on quadratic functions~\eqref{eq:ngd_tilde_quadratic}, which to the best of our knowledge has not been studied before. The dynamics is like chaotic version of power iteration, and we manage to show that the iterate will always align to the top eigenvector of  Hessian of the quadratic loss. The proof is based on identifying a novel potential (\Cref{sec:quadratic}) and might be of independent interest.

\section{Related Works}
\paragraph{Sharpness:} Low  sharpness has long been related to  flat minima  and thus to good generalization~\citep{hochschmidflat,keskar2016large}. Recent study on predictors of generalization~\citep{jiangfantastic} does show sharpness-related measures as being good predictors, leading to SAM algorithm that improves generalization by explicitly controlling a parameter related to sharpness~\citep{ForetSAM}. However, \citet{dinhsharp} show that due to the positive homogeneity in the network architecture, networks with rescaled parameters can have very different sharpness yet be the same to the original one in function space. This observation weakens correlation between sharpness and and generalization gap and makes the definition of sharpness ambiguous. In face of this challenge,  multiple notions of scale-invariant sharpness have been proposed~\citep{yi2019positively,yi2019bninvariant,tsuzuku20normalized,rangamani2021invariant}. Especially, \citet{yi2021towards,kwon2021asam} derived new algorithms with better generalization by explicitly regularizing new sharpness notions aware of the symmetry and invariance in the network. \citet{he2019asymmetric} goes beyond the notion of sharpness/flatness and argues that the local minima of modern deep networks can be asymmetric, that is, sharp on one side, but flat on the other side.


\paragraph{Limiting Diffusion/Flow around Manifold of Minimizers:} The idea of analyzing the behavior of SGD with small LR along the the manifold originates from \cite{blanc2020implicit}, which gives a local analysis on a special noise type named label noise, \emph{i.e.} noise covariance is equal to Hessian at minimizers. \citet{damian2021label} extends this analysis and show SGD with label noise finds approximate stationary point for original loss plus some Hessian-related regularizer.
The formal mathematical framework of approximating the limiting dynamics of SGD with arbitrary noise by Stochastic Differential Equations is later established by
\citet{li2022what}, which is built on the convergence result for solutions of  SDE with large-drift~\citep{katzenberger1991solutions}.


\paragraph{Implicit Bias:} The notion that training algorithm plays an active role in selecting the solution (when multiple optima exist)
has been termed the {\em implicit bias} of the algorithm~\citep{gunasekar2018implicit} and studied in a large number of papers~\citep{soudry2018implicit,li2018algorithmic,arora2018optimization,arora2019implicit,gunasekar2018implicit_conv,gunasekar2018characterizing,lyu2019gradient,li2020towards, woodworth2020kernel,razin2020implicit,lyu2021gradient,azulay2021implicit,gunasekar2021mirrorless}. In the infinite width limit, the implicit bias of Gradient Descent is shown to be the solution with the minimal RKHS norm with respect to the Neural Tangent Kernel (NTK)~\citep{jacot2018neural,li2018learning,  du2019gradient,arora2019fine,arora2019exact,allen2019convergence,allen2019learning,zou2020gradient,chizat2018lazy,yang2019scaling}.
The implicit bias results from these papers are typically proved by performing a trajectory analysis for (Stochastic) Gradient Descent. Most of the results can be directly extended to the continuous limit (i.e., GD infinitesimal LR) and even some heavily relies on the conservation property which only holds for the continuous limit. In sharp contrast, the implicit bias shown in this paper -- reducing the sharpness along the minimizer manifold -- requires finite LR and doesn't exist for the corresponding continuous limit. Other implicit bias results that fundamentally relies on the finiteness of LR includes stability analysis~\citep{wu2017towards,ma2021linear} and implicit gradient regularization~\citep{barrett2021implicit}, which is a special case of approximation results for stochastic modified equation by \citet{li2017stochastic,li2019stochastic}.

\paragraph{Non-monotone Convergence of Gradient Descent :} Recently, a few convergence results for Gradient Descent have been made where the loss is not monotone decreasing, meaning at certain steps the stableness can go above $2$ and the descent lemma breaks. These results typically involve a two-phase analysis where in the first phase the sharpness decreases and the loss can oscillate and in the second phase the sharpness is small enough and thus the loss monotone decreases. Such settings include scale invariant functions~\citep{arora2018theoretical,li2022robust} and 2-homogeneous models with $\ell_2$ loss~\citep{lewkowycz2020large,wang2021large}. Different to the previous works, the non-monotone decrease of loss shown in our work happens at Edge of Stability and doesn't require a entire phase where descent lemma holds.



\section{Warm-up: Quadratic Loss Functions}\label{sec:quadratic}

To introduce ideas that will be used in the main results,  we sketch analysis of Normalized GD \eqref{eq:ngd_quadratic} on quadratic loss function $L(x) = \frac{1}{2}x^\top A x$ 
where $A \in \RR^{D \times D}$ is positive definite with eigenvalues $\lambda_1 > \lambda_2 \ge \ldots \geq  \lambda_D$ and $v_1, \ldots, v_D$ are the corresponding eigenvectors.
\begin{align}\label{eq:ngd_quadratic}
     x(t\!+\!1) = x(t)\! -\!\eta \frac{ \nabla L(x(t))}{\norm{\nabla L(x(t))}} = x(t)\!-\!\eta \frac{Ax(t)}{\norm{Ax(t)}}.
\end{align}

Our main result \Cref{thm:main_lemma_ngd_quadratic} is that the iterates of Normalized GD $x(t)$ converge to $v_1$ in direction, from which the loss oscillation~\Cref{cor:main_lemma_ngd_quadratic} follows, suggesting that GD is operating in EoS. Since in quadratic case there is only one local minima, there is of course no need to talk about implicit bias. However, the observation that the GD iterates always align to the top eigenvector as well as the technique used in its proof play a very important role  for deriving the sharpness-reduction implicit bias for the case of general loss functions.  

Define $\Tilde{x}(t) = \frac{Ax(t)}{\eta}$, and  the following update rule~\eqref{eq:ngd_tilde_quadratic} holds.  It is clear that  the convergence of $\Tilde{x}_t$  to $v_1$ in direction implies the convergence of  $x_t$ as well.
\begin{align}
    \restatableeq{\quadraticxtilde}{\Tilde{x}(t+1) = \Tilde{x}(t) - A \frac{ \Tilde{x}(t) }{ \norm{ \Tilde{x}(t) }}.}{eq:ngd_tilde_quadratic}
\end{align}

\begin{restatable}{theorem}{mainlemmaquadratic}\label{thm:main_lemma_ngd_quadratic}
If $\abs{ \langle v_1, \Tilde{x}(t) \rangle }\neq 0$, $\forall t\ge 0$, then there exists $0<C<1$ and $s\in\{\pm1\}$ such that $\lim_{t\to \infty} \Tilde{x}(2t) =Cs \lambda_1 v_1 $ and $\lim_{t\to \infty} \Tilde{x}(2t+1) = (C-1)s \lambda_1 v_1  $. 
\end{restatable}

As a direct corollary, the loss oscillates as between time step $2t$ and time step $2t+1$ as $t \to \infty$. This shows that the behavior of loss is not monotonic and hence indicates the edge of stability phenomena for the quadratic loss.
\begin{corollary}\label{cor:main_lemma_ngd_quadratic}
If $\abs{ \langle v_1, \Tilde{x}(t) \rangle }\neq 0$, $\forall t\ge 0$, then there exists $0<C<1$ such that $\lim_{t\to \infty} L(x(2t)) = \frac{1}{2} C^2 \lambda_1 \eta^2$ and $\lim_{t\to \infty} L(x(2t+1)) = \frac{1}{2} (C-1)^2 \lambda_1 \eta^2$.  
\end{corollary}





We  analyse the trajectory of the iterate $\Tilde{x}(t)$ in two phases. For convenience, we define $\projd{j}$ as the projection matrix into the space spanned by $\{v_i\}_{i=j}^D$, \emph{i.e.}, $\projd{j}:= \sum_{i=j}^D v_iv_i^\top$. 
In the first \emph{preparation phase}, $\Tilde{x}(t)$ enters the intersection of $D$ invariant sets $\{\mathcal{I}_j\}_{j=1}^D$ around the origin, where $\mathcal{I}_j := \{ \Tilde{x}\mid \norm{\projd{j}\Tilde x} \le \lambda_j\}$. (\Cref{lem:prepphase})
In the second  \emph{alignment phase}, the projection of $\Tilde{x}(t)$ on the top eigenvector, $|\inner{\Tilde{x}(t)}{v_1}|$, is shown to increase monotonically among the steps among the steps $\{t\in\mathbb{N}\mid \norm{\Tilde{x}(t)} \le 0.5 \lambda_1\}$. Since it is bounded, it must converge. The vanishing increment over steps turns out to suggest the $\Tilde x(t)$ must converge to $v_1$ in direction.

\begin{lemma}[Preparation Phase]\label{lem:prepphase}
For any  $j\in[D]$ and $t \ge \frac{  \lambda_1} { \lambda_j }\ln \frac{  \lambda_1} { \lambda_j }+ \max\{\frac{\norm{\Tilde x(0)} - \lambda_1}{\lambda_D}, 0\}$, it holds that $\Tilde x(t)\in \mathcal{I}_j$. 
\end{lemma}

\begin{proof}[Proof of \cref{lem:prepphase}] 
First, we show for any $j\in [D]$, $\mathcal{I}_j$ is indeed an invariant set for update rule \eqref{eq:ngd_tilde_quadratic} via \Cref{lem:quadratic_norm_invariant}. With straightforward calculation, one can show that for any $j\in [D]$, $\norm{\projd{j}\Tilde x(t)}$ decreases by $\frac{\lambda_D \norm{\projd{j}\Tilde x(t) }}{\norm{\Tilde x(t)}}$ if $\norm{\projd{j}\Tilde x(t)}\ge \lambda_j$ (\Cref{lem:quadratic_prep_decrease}). Setting $j=1$, we have $\norm{\Tilde x(t)}$ decreases by $\lambda_D$ if $\norm{\Tilde x(t)}\ge \lambda_1$ (\Cref{lem:invariant1}). Thus for all $t\ge  \max\{\frac{\norm{\Tilde x(0)} - \lambda_1}{\lambda_D}, 0\}$, $\Tilde x(t)\in \mathcal{I}_1$. Finally once $\Tilde x(t)\in \mathcal{I}_1$, we can  upper bound $\norm{\Tilde x(t)}$ by $\lambda_1$, and thus $\norm{\projd{j}\Tilde x(t) }$ shrinks at least by a factor of $\frac{\lambda_D}{\lambda_1}$ per step, which implies $\Tilde x(t)$ will be in $\mathcal{I}_j$ in another $\frac{  \lambda_1} { \lambda_j }\ln \frac{  \lambda_1} { \lambda_j }$ steps.(\Cref{lem:invariantj})
\end{proof}

Once the component of $\Tilde{x}(t)$ on an eigenvector becomes $0$, it stays $0$. So  without loss of generality we can assume that after the preparation phase, the projection of $\Tilde{x}(t)$ along the top eigenvector $v_1$ is non-zero, otherwise we can study the problem in the subspace excluding the top eigenvector. 

\begin{lemma}[Alignment Phase]\label{lem:alignmentphase}
If $\Tilde x(T)\in\cap_{j=1}^D \mathcal{I}_j$ holds for some $T$, then for any $t', t$ such that $T\le t\le t'$ and  $\norm{ \Tilde{x} (t)}\le 0.5 \lambda_1$, it holds $\abs{ \langle v_1, \Tilde{x} (t) \rangle }\le \abs{ \langle v_1, \Tilde{x} (t') \rangle }$.
\end{lemma}

\begin{figure*}[t]
\centering
\begin{minipage}[t]{0.33\textwidth}
\centering\includegraphics[width=\linewidth]{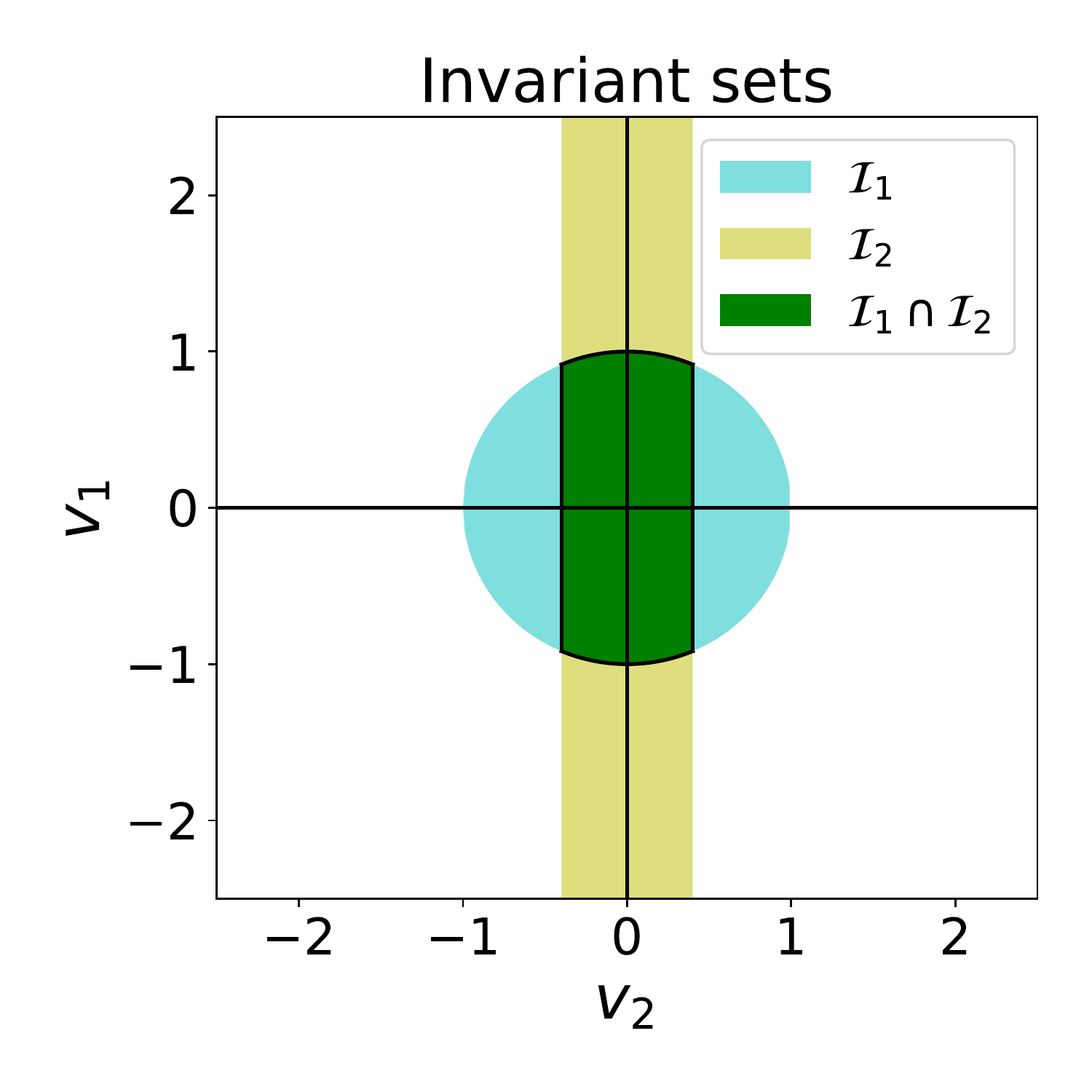} 
\end{minipage}
\centering
\begin{minipage}[t]{0.33\textwidth}
\centering\includegraphics[width=\linewidth]{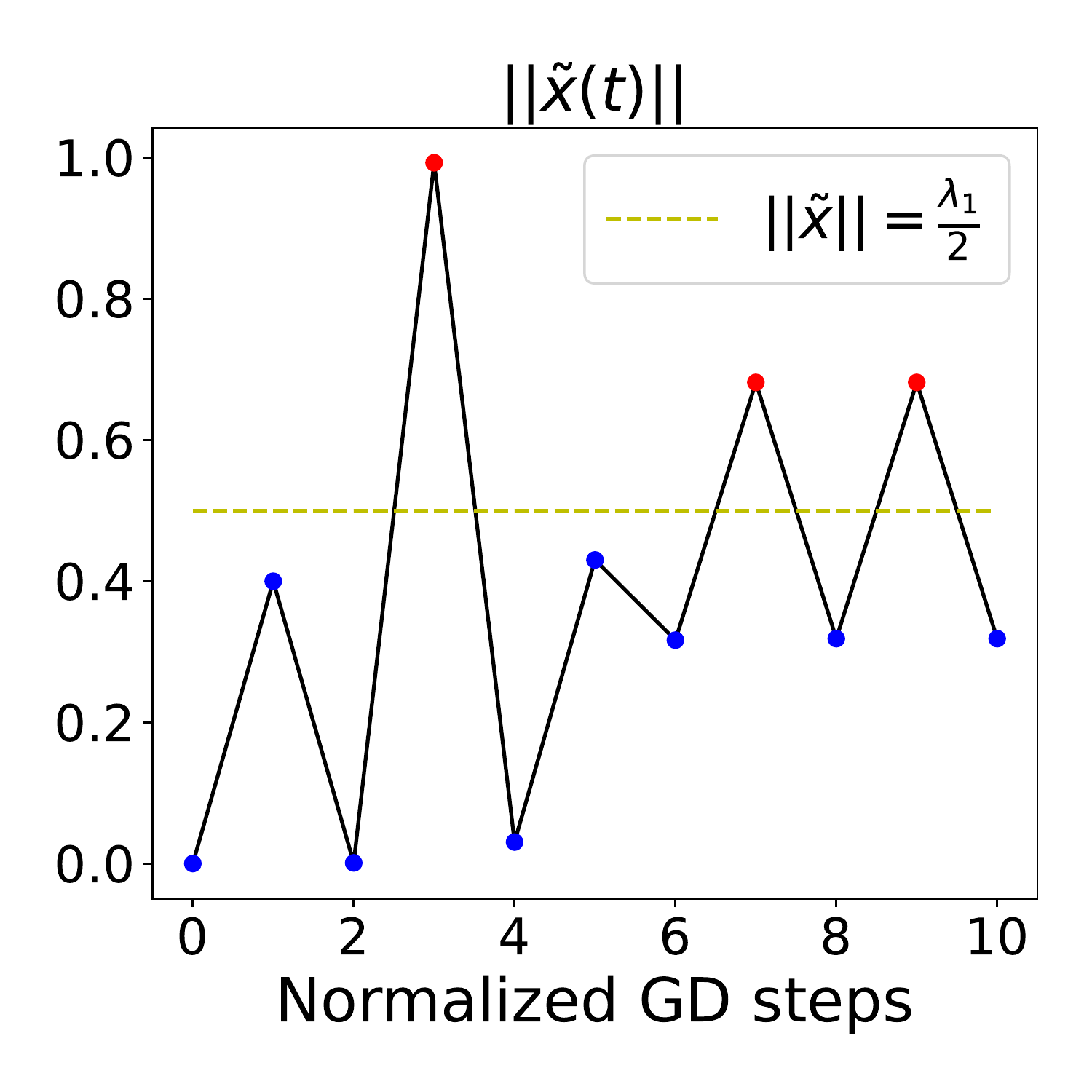} 
\end{minipage}
\begin{minipage}[t]{0.33\textwidth}
\centering\includegraphics[width=\linewidth]{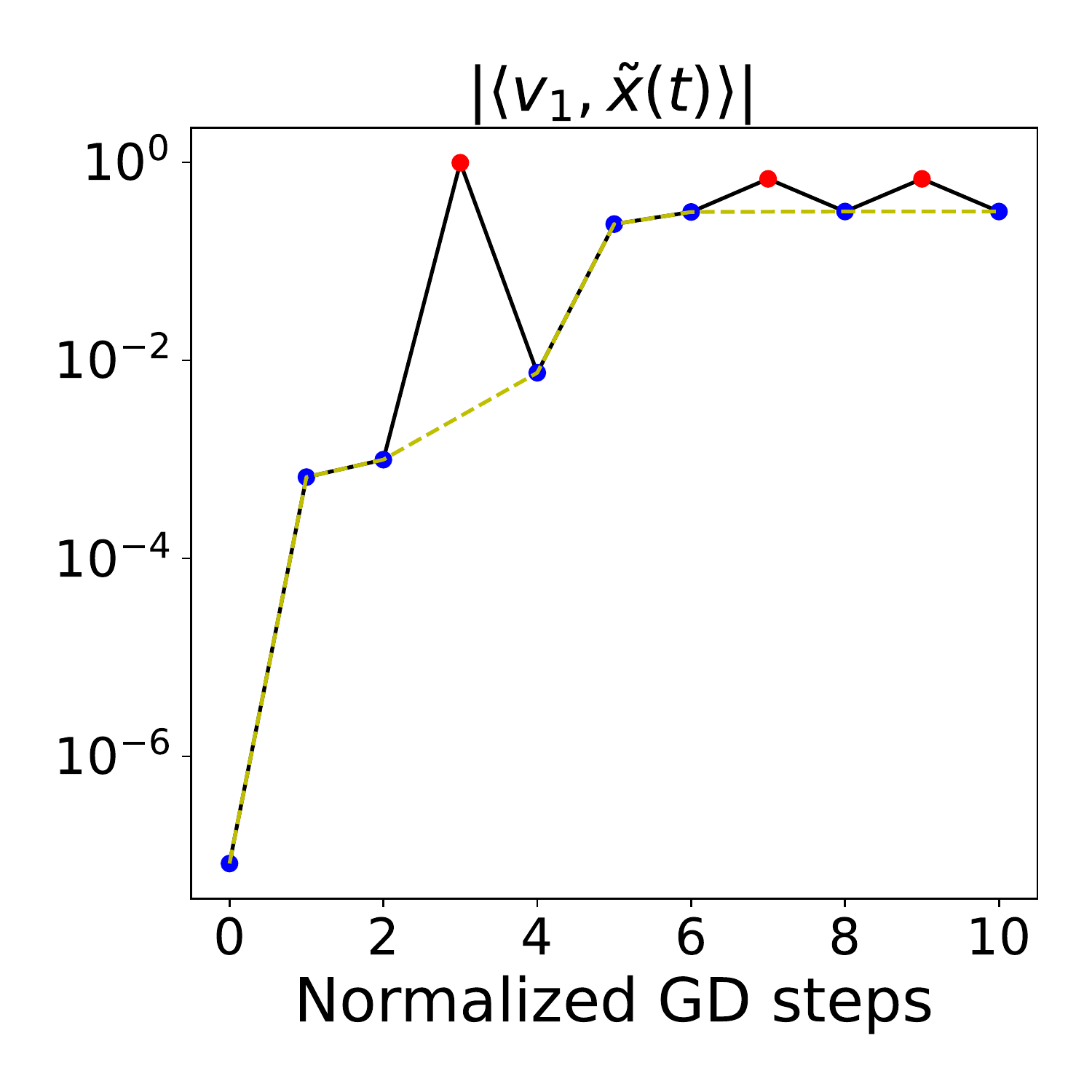} 
\end{minipage}
\caption{Visualization of key concepts and lemmas in the analysis for Normalized GD on a 2D quadratic loss with $\lambda_1=1,\lambda_2=0.4$. \textbf{Left}: invariant sets (defined in \Cref{lem:prepphase}). \textbf{Middle}:  $\norm{\Tilde{x}(t)}$ drops below $\frac{\lambda_1}{2}$ in the next step whenever it is above $\frac{\lambda_1}{2}$ (\Cref{lem:iteratenorm_drops}). \textbf{Right}: $\abs{ \langle v_1, \Tilde{x}(t) \rangle}$ monotone increases among all the steps with  norm  below $\frac{\lambda_1}{2}$. (\Cref{lem:one_two_step_grow})}
\label{fig:Quadratic}
\end{figure*}

Below we sketch the proof of \Cref{lem:alignmentphase}. 
\begin{proof}[Proof of \Cref{lem:alignmentphase}]
First, \Cref{lem:iteratenorm_drops} (proved in \Cref{appsec:quadraticmodel}) shows that the norm of the iterate $\Tilde{x}(t)$ remains above $0.5 \lambda_1$ for only one time-step.  



\begin{restatable}{lemma}{lemeteratenormdrops}\label{lem:iteratenorm_drops}
For any  $t$ with $\Tilde x(t)\in \cap_{j=1}^D \mathcal{I}_j$, if $\norm{\Tilde{x}(t)} > \frac{\lambda_1}{2}$, then $\norm{\Tilde{x}(t+1)} \le \max \left( \frac{\lambda_1}{2} - \frac{\lambda_D^2}{2\lambda_1}, \lambda_1 - \norm{\Tilde{x} (t)} \right)$.
\end{restatable}


Thus, for any $t$ with $\Tilde x(t)\in \cap_{j=1}^D \mathcal{I}_j$ and $\norm{\Tilde{x}(t)} \le \frac{\lambda_1}{2}$, either $\norm{\Tilde{x}(t+1)} \le \frac{\lambda_1}{2}$, or $\norm{\Tilde{x}(t+1)} > \frac{\lambda_1}{2}$, which in turn implies that  $\norm{\Tilde{x}(t+2)} \le \frac{\lambda_1}{2}$ by \Cref{lem:iteratenorm_drops}. The proof of \Cref{lem:alignmentphase} is completed by induction on \Cref{lem:one_two_step_grow}. 

\begin{lemma}\label{lem:one_two_step_grow}
    For any step $t$ with $\norm{\Tilde{x}(t)} \le \frac{\lambda_1}{2}$, for any $k\in \{1,2\}$,  $\abs{ \langle v_1, \Tilde{x}(t+k) \rangle} \ge \abs{ \langle v_1, \Tilde{x}(t) \rangle}$.
\end{lemma}
Proof of case $k=1$ in \Cref{lem:one_two_step_grow} follows directly from plugging the assumption  $\norm{\Tilde{x}(t)} \le \frac{\lambda_1}{2}$ into \eqref{eq:ngd_tilde_quadratic} (See \Cref{lem:behaviornormhalf}). The case of $k=2$ in \Cref{lem:one_two_step_grow} follows from \Cref{lem:firscoordincr_ondrop}. We defer the complete proof of \Cref{lem:one_two_step_grow} into \Cref{appsec:quadraticmodel}. 
\end{proof}

To complete the proof for \cref{thm:main_lemma_ngd_quadratic}, we relate the increase in the projection along $v_1$ at any step $t$, $\abs{ \langle v_1, \Tilde{x}(t) \rangle}$, to the magnitude of the angle between $\Tilde{x}(t)$ and the top eigenspace, $\theta_t$. Briefly speaking, we show that if $\norm{\Tilde{x}(t)} \le \frac{\lambda_1}{2}$, $\abs{ \langle v_1, \Tilde{x}(t) \rangle}$ has to increase by a factor of $ \Theta(\theta^2_t)$ in two steps.
Since $\abs{ \langle v_1, \Tilde{x}(t) \rangle}$ is bounded and monotone increases among $\{t\mid \norm{\Tilde x(t)} \le \frac{\lambda_1}{2}\}$ by \Cref{lem:alignmentphase},
 we conclude that $\theta_t$ gets arbitrarily small for sufficiently large $t$ with $\norm{\Tilde x(t)} \le \frac{\lambda_1}{2}, \norm{\Tilde x(t+2)} \le \frac{\lambda_1}{2}$ satisfied. Since the one-step normalized GD update \Cref{eq:ngd_tilde_quadratic} is continuous when bounded away from origin, with a careful analysis, we conclude $\theta_t\to 0$ for all iterates. Please see \cref{sec:main_lemma_quadratic} for details. 
\paragraph{Equivalence to GD on $\sqrt{\frac{1}{2}x^\top A x}$:} Below we show  GD on loss $\sqrt{L}(x) = \sqrt{\frac{1}{2}x^\top A x}$, \Cref{eq:gd_quadratic_square_smooth}, follows the same update rule  as  Normalized GD on $L(x) = \frac{1}{2}x^\top A x$, up to a linear transformation. 
\begin{align}\label{eq:gd_quadratic_square_smooth}
x(t+1) = x(t) - \eta \nabla \sqrt{L}(x(t)) = x(t)- \eta\frac{Ax(t)}{\sqrt{2x(t)^\top Ax(t)}}.
\end{align}
Denoting $\Tilde{x}(t) = \frac{1}{\eta} (2A)^{1/2} x(t)$, we can easily check $\Tilde{x}(t) $ also satisfies update rule~\eqref{eq:ngd_tilde_quadratic}.


\section{Main Results}

In this section we present the main results of this paper. \Cref{subsec:notations} is for preliminary and notations. In \Cref{subsec:assump_on_manifold}, we make our key assumptions that the minimizers of the loss function  form a manifold. In \Cref{sec:ngd_mainpaper,sec:gdsqrtL_mainpaper} we present our main results for Normalized GD and GD on $\sqrt{L}$ respectively. In \Cref{subsec:EoSoperation} we show the above two settings for GD do enter the regime of Edge of Statbility.


\subsection{Preliminary and Notations}\label{subsec:notations}

For any integer $k$, $\mathcal{C}^{k}$ denotes  the set of the $k$ times continuously differentiable functions and $L\mathcal{C}^k$ denotes the set of the $k$ times lipshitz differentiable functions, \emph{i.e.}., the $k$th derivative are locally lipschitz. It holds that $\mathcal{C}^{k+1}\implies L\mathcal{C}^{k}\implies \mathcal{C}^{k}$ and that $\mathcal{C}^\infty \iff L\mathcal{C}^\infty$. For any mapping $F$, 
we use $\partial F(x)[u]$ and $\partial^{2} F(x)[u, v]$ to denote the first and second order directional derivative of $F$ at $x$ along the derivation of $u$ (and $v)$. 
 Given the loss function $L$, the gradient flow (GF) governed by $L$ can be described through a mapping $\phi: \mathbb{R}^{D} \times[0, \infty) \rightarrow \mathbb{R}^{D}$ satisfying $\phi(x, \conttime)=x-\int_{0}^{\conttime} \nabla L(\phi(x, s)) \mathrm{d} s$. We further define the limiting map of gradient flow as $\Phi$, that is, $\Phi(x) = \lim_{\conttime \to \infty} \phi(x, \conttime).$

For a matrix $A \in \mathbb{R}^{D \times D}$, we denote its eigenvalue-eigenvector pairs by $\{\lambda_i(A), v_i(A) )\}_{i \in [D]}$. For simplicity, 
whenever $\Phi$ is defined at point $x$, we use $\{(\lambda_i(x), v_i(x))\}_{i=1}^{D}$ to denote the eigenvector-eigenvalue pairs of $\nabla^2 L ( \Phi(x) )$, with $\lambda_1(x) > \lambda_2(x) \ge \lambda_3(x) \ldots \ge \lambda_D(x)$. As an analog to the quadratic case, we use $\tilde x$ to denote $\nabla^2 L(\Phi ( x  )) (x - \Phi(x ))$ for Normalized GD on $L$ and  $\left( 2\nabla^2 L(\Phi ( x  )) \right)^{1/2} (x  - \Phi(x ))$ for GD on $\sqrt{L}$. Furthermore, when the iterates $x(t)$ are clear in the context, we also use shorthand $\lambda_i(t): =\lambda_i(x(t))$,  $v_i(t):= v_i(x(t))$ and  $\theta_t\in [0,\frac{\pi}{2}]$ to denote the angle between $\tilde x(t)$ and top eigenspace of $\nabla^2 L(\Phi(x(t)))$. Given a differentiable submanifold $\Gamma$ of $\mathbb{R}^D$ and point $x\in\Gamma$, we use $P_{x,\Gamma}: \Gamma \to \mathbb{R}^{D}$ to denote the projection operator onto the normal space of $\Gamma$ at $x$, and $P_{x,\Gamma}^\perp :=I_D - P_{ x,\Gamma }$.  As before, for notational convenience, we  use the shorthand $P_{t,\Gamma} := P_{\Phi(x(t)),\Gamma}$ and $P_{t,\Gamma}^\perp := P_{\Phi(x(t)),\Gamma}^{\perp} $. 

In this section, we focus on the setting where LR $\eta$ goes to 0 and we fix  the initialization $\xinit$ and the loss function $L$ throughout this paper. We use $O(\cdot),\Omega(\cdot)$ to hide constants about  $\xinit$ and $L$.


\subsection{Key Assumptions on Manifold of Local Minimizers}\label{subsec:assump_on_manifold}

Following \citet{fehrman2020convergence,li2022what}, we make the following assumption throughout the paper. 

\begin{restatable}{assumption}{assumptiononL}\label{ass:manifold}
Assume that the loss $L: \mathbb{R}^{D} \rightarrow \mathbb{R}$ is a $\mathcal{C}^{4}$ function, and that $\Gamma$ is a $(D-M)$-dimensional $\mathcal{C}^{2}$-submanifold of $\mathbb{R}^{D}$ for some integer $1 \leq M \leq D$, where for all $x \in \Gamma, x$ is a local minimizer of $L$ with $L(x)=0$ and $\operatorname{rank}\left(\nabla^{2} L(x)\right)=M$.
\end{restatable}

Let $U$ be the attraction set of $\Gamma$, that is, the set of points starting from which gradient flow w.r.t. loss $L$ converges to some point in $\Gamma$, that is, $U:=\{ x \in \RR^D \mid \Phi(x) \textrm{ exists and } \Phi(x) \in \Gamma \}$. \Cref{ass:manifold} implies that $U$ is open and $\Phi$ is $L\mathcal{C}^2$ on $U$. (By \Cref{lem:U_open_Phi_smooth})

The smoothness assumption is satisfied for networks with smooth activation functions like tanh and GeLU \citep{hendrycks2016gaussian}. The existence of manifold is due to the vast overparametrization in modern deep networks and preimage theorem. (See a discussion in section 3.1 of \citet{li2022what}) The assumption $\operatorname{rank}\left(\nabla^{2} L(x)\right)=M$ basically say $\nabla^2 L(x)$ always attains the maximal rank in the normal space of the manifold, which ensures the differentiability of $\Phi$ and is crucial to our current analysis, though it's not clear if it is necessary.  We also make the following assumption to ensure that $\lambda_1(\nabla^ 2 L(\cdot))$ is differentiable, which is necessary for our main results, \Cref{thm:ngd_phase_2,thm:PhaseII_sqrtL}.

\begin{restatable}{assumption}{assumptiononEigengap}\label{ass:eigen_gap}
	For any $x\in \Gamma$, $\nabla^2 L(x)$ has a positive eigengap, \emph{i.e.}, $\lambda_1(\nabla^2 L(x))>\lambda_2(\nabla^2 L(x))$.
\end{restatable}




\subsection{Results for Normalized GD}\label{sec:ngd_mainpaper}
We first denote the iterates of Normalized GD with LR $\eta$ by $x_\eta(t)$, with $x_\eta(0)\equiv \xinit$ for all $\eta$:
\begin{equation}\label{eq:normalized_gd}
\textrm{Normalized GD:}\quad  x_\eta(t+1) = x_\eta(t) -\eta \frac{\nabla L(x_\eta(t))}{\norm{\nabla L(x_\eta(t))}}	
\end{equation}

The first theorem demonstrates the movement in the manifold, when the iterate travels from $\xinit$ to a position that is $O(\eta)$ distance closer to the manifold (more specifically, $\Phi(\xinit)$). Moreover, just like the result in the quadratic case, we have more fine-grained bounds on the projection of $x_\eta(t) - \Phi(x_\eta(t))$ into the bottom-$k$ eigenspace of $\nabla^2 L(\Phi(x_\eta(t)))$ for every $k\in [D]$. 
For convenience, we define the following quantity  for all $j\in [d]$ and $x\in U$:
$$R_j(x) := \sqrt{ \sum_{i=j}^{M}\langle v_i(x),  \tilde x\rangle^2 } - \lambda_j(x) \eta  = \sqrt{ \sum_{i=j}^{M} \lambda^2_i(x) \langle v_i(x),  x - \Phi(x)\rangle^2 } - \lambda_j(x) \eta $$

In the quadratic case, \Cref{lem:prepphase} shows that $R_j(x)$ will eventually become non-positive for normalized GD iterates. Similarly, for the general loss, the following theorem shows that $R_j(x_\eta(t))$ eventually becomes approximately non-positive (smaller than $O(\eta^2)$) in $O(\frac{1}{\eta})$ steps.

\begin{figure}[t]
\centering
\includegraphics[width=0.55\linewidth]{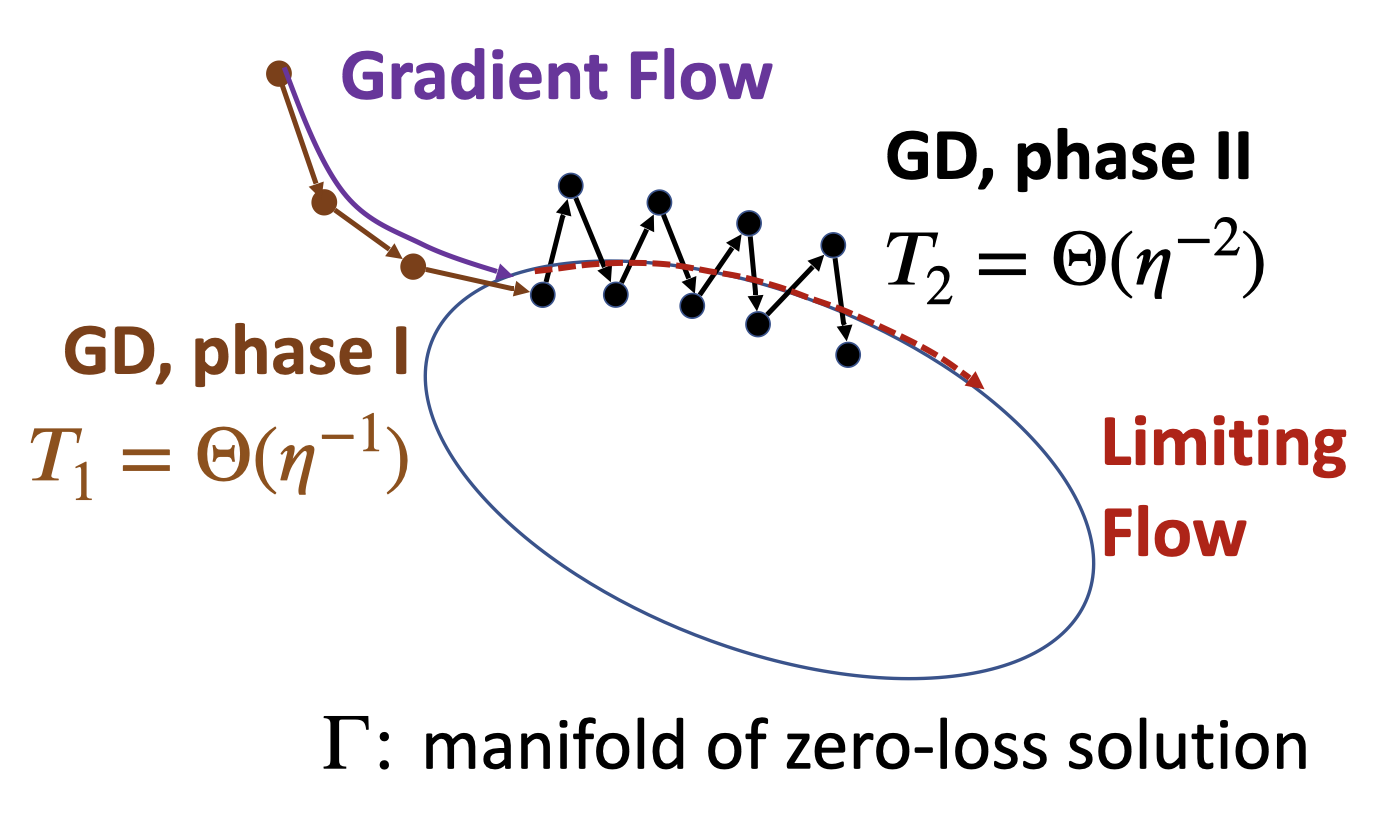}
\caption{Illustration for two-phase dynamics of Normalized GD and GD on $\sqrt{L}$ on a 1D zero loss manifold $\Gamma$. For
sufficiently small LR $\eta$, Phase I is close to Gradient Flow and lasts for $\Theta(\eta^{-1})$ steps, while
Phase II is close to the limiting flow which decreases the sharpness of the loss and lasts for $\Theta(\eta^{-2})$ steps. GD iterate oscillates along the top eigenvector of the Hessian with the period equal to two steps. (cf. Figure 2 in \citep{li2022what})}
\label{fig:limiting_flow}
\end{figure}

\begin{restatable}[Phase I]{theorem}{theoremngdfirstphase}\label{thm:ngd_phase_1}
 Let $\{x_\eta(t)\}_{t\in \mathbb{N}}$ be the iterates of Normalized GD~\eqref{eq:normalized_gd} with LR $\eta$ and $x_\eta(0)= \xinit\in U$. There is  $T_1>0$ such that for any $T_1'>T_1$, it holds that for sufficiently small $\eta$ that (1) $\max\limits_{ T_1\le \eta t\le T_1'}\norm{ x_\eta(t)-\Phi(\xinit)} \le O(\eta)$ and (2) $\max\limits_{ T_1\le \eta t\le T_1', j\in[D]} R_j(x_\eta(t))\le O(\eta^2)$.
\end{restatable}

Our main contribution is the analysis for the second phase~(\Cref{thm:ngd_phase_2}), which says just like the quadratic case, the angle between $\tilde x_\eta(t)$ and the top eigenspace of $\nabla^2 L(\Phi(x_\eta(t)))$, denoted by $\theta_t$, will be $O(\eta)$ on average. And as a result, the dynamics of Normalized GD tracks the riemannian gradient flow with respect to $\log(\lambda_1(\nabla^2 L (\cdot)))$ on manifold, that is, the unique solution of \Cref{eq:log_limiting_flow}, where $P^{\perp}_{x, \Gamma}$ is the  projection matrix onto the tangent space of manifold $\Gamma$ at $x\in\Gamma$. 
    \begin{align}
       \textrm{Limiting Flow:}\quad  \restatableeq{\loglimitingflow}{X(\conttime)=   \Phi(\xinit) - \frac{1}{4}\int_{s=0}^\conttime P^{\perp}_{X(s), \Gamma} \nabla \log  \lambda_1( X(s) )\diff s, \quad X(\conttime) \in \Gamma }{eq:log_limiting_flow}
    \end{align}
    
Note \Cref{eq:log_limiting_flow} is not guaranteed to have a global solution, \emph{i.e.}, a well-defined  solution for all $\conttime\ge0$, for the following two reasons: (1).  when the  multiplicity of top eigenvalue is larger than $1$, $\lambda_1(\nabla^2 L (\cdot))$ may be not differentiable and (2). the projection matrix is only defined on $\Gamma$ and the equation becomes undefined when the solution leaves $\Gamma$, \emph{i.e.}, moving across the boundary of $\Gamma$. For simplicity, we make \Cref{ass:eigen_gap} that every point on $\Gamma$ has a positive eigengap. Or equivalently, we can work with a slightly smaller manifold $\Gamma' = \{x\in\Gamma\mid  \lambda_1(x)>\lambda_2(x)\}$.

Towards a mathematical rigorous characterization of the dynamics in the second phase, we need to make the following modifications: (1). we add negligible noise of magnitude $O(\eta^{100})$ every $\eta^{-\freq}$ steps, (2). we assume for each $\eta>0$, there exist some step  $t=\Theta(1/\eta)$ in phase I, except the guaranteed condition (1) and (2) (by \Cref{thm:ngd_phase_1}, the additional condition (3) also holds. This assumption is mild because we only require (3) to hold for one step among $\Theta(1/\eta)$ steps from $\frac{T_1}{\eta}$ to $\frac{T_1'}{\eta}$, where $T_1$ is the constant given by \Cref{thm:ngd_phase_1} and $T_1'$ is arbitrary constant larger than $T_1$. This assumption also holds empirically for all our experiments in \Cref{sec:exp}.

\begin{algorithm}[tb]
   \caption{Perturbed Normalized Gradient Descent}
   \label{alg:PNGD}
\begin{algorithmic}

   \STATE {\bfseries Input:} loss function $L: \mathbb{R}^{D} \to \mathbb{R}$, initial point $\xinit$, maximum number of iteration $T$, LR $\eta$, Frequency parameter $T_{\mathrm{freq}} = \Theta(\eta^{-\freq})$, noise parameter $\noiseparameter = \Theta(\eta^{100})$.
   \FOR{$t=1$ {\bfseries to} $T$}
   \STATE Generate $n(t) \sim B_0(\noiseparameter)$ \textrm{if} $t \mod T_{\mathrm{freq}} = 0$, \textrm{else set} $n(t)=0$.
  \STATE $x(t) \gets x(t-1) - \eta \frac{\nabla L(x(t))}{\norm{\nabla L(x(t))}} +n(t).$  \\
   \ENDFOR
\end{algorithmic}
\end{algorithm}

\begin{restatable}[Phase II]{theorem}{theoremngdsecondphase}\label{thm:ngd_phase_2}
   Let $\{x_\eta(t)\}_{t\in \mathbb{N}}$ be the iterates of perturbed Normalized GD (\Cref{alg:PNGD}) with LR $\eta$. Under \Cref{ass:manifold} and \Cref{ass:eigen_gap}, if the initialization $x_\eta(0)$ satisfy  that  
   (1) $\norm{ x_\eta(0)-\Phi(\xinit)} \le O(\eta)$ where $\xinit \in U$, 
   (2) $\max_{ j\in[D]} R_j(x_\eta(t))\le O(\eta^2)$,
    and additionally 
    (3) $\min\{ \abs{ \langle v_1(x_\eta(0)), x_\eta(0) - \Phi(x_\eta(0)) \rangle }, -R_1(x_\eta(0))\} \ge \Omega(\eta)$,  then for any time $T_2 > 0$ till which the solution of \eqref{eq:log_limiting_flow} exists, it holds for sufficiently small $\eta$, with probability at least $1-O(\eta^{10})$, that   $\norm{ \Phi(x_\eta(\lfloor T_2/\eta^2 \rfloor))-X(T_2)}= O(\eta)$ and   $ \frac{1}{\lfloor T_2/\eta^2 \rfloor}\sum_{t=0}^{ \lfloor T_2/\eta^2 \rfloor} \theta_t \le O(\eta )$, where $\theta_t\in [0,\frac{\pi}{2}]$  denotes the angle between $\nabla^2 L(\Phi(x_\eta(t)))(x_\eta(t)-\Phi(x_\eta(t)))$ and top eigenspace of $\nabla^2 L(\Phi(x_\eta(t)))$.
\end{restatable}

\subsection{Results for GD on $\sqrt{L}$}\label{sec:gdsqrtL_mainpaper}
In this subsection, we denote the iterates of GD on $\sqrt{L}$ with LR $\eta$ by $x_\eta(t)$, with $x_\eta(0)\equiv \xinit$ for all $\eta$:
\begin{equation}\label{eq:gd_sqrtL}
\textrm{GD on $\sqrt{L}$:}\quad  x_\eta(t+1) = x_\eta(t) -\eta \nabla \sqrt{L}(x_{\eta}(t))	
\end{equation}

Similar to Normalized GD, we will have two phases.
The first theorem demonstrates the movement in the manifold, when the iterate travels from $\xinit$ to a position that is $O(\eta)$ distance closer to the manifold. For convenience, we will denote the quantity $\sqrt{ \sum_{i=j}^{M} \lambda_i(x) \langle v_i(x),  x - \Phi(x)\rangle^2 } -\eta \sqrt{1/2} \lambda_j(x)  $ by $\overline{R}_j(x)$ for all $j\in [M]$ and $x\in U$.
\begin{theorem}[Phase I]\label{thm:sqrtL_phase_1}
    Let $\{x_\eta(t)\}_{t\in \mathbb{N}}$ be the iterates of Normalized GD~\eqref{eq:gd_sqrtL} with LR $\eta$ and $x_\eta(0)= \xinit\in U$. There is  $T_1\in \mathbb{R}^+$ such that for any $T_1'\in\mathbb{R}^+$, it holds  for sufficiently small $\eta$ that (1) $\max\limits_{ T_1\le \eta t\le T_1'}\norm{ x_\eta(t)-\Phi(\xinit)} \le O(\eta)$ and (2) $\max\limits_{ T_1\le \eta t\le T_1', j\in[D]} \overline{R}_j(x_\eta(t))\le O(\eta^2)$.
\end{theorem}

\begin{algorithm}[tb]
   \caption{Perturbed Gradient Descent on $\sqrt{L}$}
   \label{alg:PsqrtL}
\begin{algorithmic}

   \STATE {\bfseries Input:} loss function $L: \mathbb{R}^{D} \to \mathbb{R}$, initial point $\xinit$, maximum number of iteration $T$, LR $\eta$, Frequency parameter $T_{\mathrm{freq}} = \Theta(\eta^{-\freq})$, noise parameter $\noiseparameter = \Theta(\eta^{100})$.
   \FOR{$t=1$ {\bfseries to} $T$}
   \STATE Generate $n(t) \sim B_0(\noiseparameter)$ \textrm{if} $t \mod T_{\mathrm{freq}} = 0$, \textrm{else set} $n(t)=0$.
  \STATE $  x(t) \gets x(t-1) - \eta \nabla \sqrt{L}(x(t)) + n(t).$
   \ENDFOR
\end{algorithmic}
\end{algorithm}

The next result demonstrates that close to the manifold, the trajectory implicitly minimizes sharpness.
\begin{theorem}[Phase II]\label{thm:PhaseII_sqrtL}
Let $\{x_\eta(t)\}_{t\in \mathbb{N}}$ be the iterates of perturbed  GD on $\sqrt{L}$ (\Cref{alg:PsqrtL}). Under \Cref{ass:manifold} and \Cref{ass:eigen_gap},  if the initialization $x_\eta(0)$ satisfy  that  
   (1) $\norm{ x_\eta(0)-\Phi(\xinit)} \le O(\eta)$, where $\xinit \in U$, 
   (2) $\max_{ j\in[D]} \overline R_j(x_\eta(t))\le O(\eta^2)$,
    and additionally (3) $\min\{ \abs{ \langle v_1(x_\eta(0)), x_\eta(0) - \Phi(x_\eta(0)) \rangle }, -R_1(x_\eta(t))\} \ge \Omega(\eta)$,  then for any time $T_2 > 0$ where the solution of \eqref{eq:nonlog_limiting_flow} exists, it holds for sufficiently small $\eta$, with probability at least $1-O(\eta^{10})$, that   $\norm{ \Phi(x_\eta(\lfloor T_2/\eta^2 \rfloor))-X(T_2)}= O(\eta^{1/2})$ and   $ \frac{1}{\lfloor T_2/\eta^2 \rfloor}\sum_{t=0}^{ \lfloor T_2/\eta^2 \rfloor} \theta_t \le O(\eta^{1/2} )$, where $\theta_t\in [0,\frac{\pi}{2}]$  denotes the angle between $\sqrt{\nabla^2 L(\Phi(x_\eta(t)))}(x_\eta(t)-\Phi(x_\eta(t)))$ and top eigenspace of $\nabla^2 L(\Phi(x_\eta(t)))$.
    \begin{align}
    \restatableeq{\nonloglimitingflow}{X(\conttime)=  \Phi(\xinit)- \frac{1}{8}\int_{s=0}^\conttime P^{\perp}_{X(s), \Gamma} \nabla  \lambda_1(X(s))\diff s, \quad  X(\conttime) \in \Gamma.}{eq:nonlog_limiting_flow}
    \end{align}
\end{theorem}

\vspace{-5pt}
\subsection{Operating on the Edge of Stability}
\label{subsec:EoSoperation}

In this section, we  show that both Normalized GD on $L$ and GD on $\sqrt{L}$ is on Edge of Stability in their phase II, that is, at least in one of every two consecutive  steps, the stableness is at least $2$ and the loss oscillates in every two consecutive steps. Interestingly, the average loss over two steps decreases over time, even when operating on the edge of Stability (see \Cref{fig:ngd_traj_3d} for illustration), as indicated by the following theorems. Note that \Cref{thm:ngd_phase_2,thm:PhaseII_sqrtL} ensures that the average of $\theta_t$ are $O(\eta)$ and $O(\sqrt{\eta})$. 
We defer their proofs into \Cref{subsec:proof_ngd_eos,subsec:proof_sqrtL_eos} respectively.

\begin{theorem}[Stableness, Normalized GD]\label{thm:stableness_ngd}
 Under the setting of \Cref{thm:ngd_phase_2}, by viewing Normalized GD as GD with time-varying LR $\eta_t := \frac{\eta}{\norm{\nabla L(x_\eta(t)) }}$, we have 
$[\stableness{L}{x_\eta(t)}{\eta_t}]^{-1} +  [\stableness{L}{x_\eta(t+1)}{\eta_{t+1}}]^{-1} = 1 + O(\theta_t+\eta).$	Moreover, we have $\sqrt{L(x_\eta(t))}+ \sqrt{L(x_\eta(t+1))} = \eta\sqrt{\frac{\lambda_1(\nabla^2 L(x_\eta(t)))}{2}} +O(\eta\theta_t)$.
\end{theorem}

\begin{theorem}[Stableness, GD on $\sqrt{L}$]\label{thm:stableness_gd_sqrt_L}
 Under the setting of \Cref{thm:PhaseII_sqrtL}, we have 
$[\stableness{\sqrt{L}}{x_\eta(t)}{\eta_t}]\ge \Omega(\frac{1}{\theta_t})$. 	Moreover, we have $\sqrt{L(x_\eta(t))}+ \sqrt{L(x_\eta(t+1))} = \eta\lambda_1(\nabla^2 L(x_\eta(t)))+O(\eta\theta_t)$.
\end{theorem}

\section{Proof Overview}
We sketch the proof of the Normalized GD in phase I and II respectively in \Cref{subsec:proof_sketch_ngd}. 
Then we briefly discuss how to prove the results for GD with $\sqrt{L}$ with same analysis in \Cref{subsec:proof_sketch_sqrt_loss}. We start by introducing the properties of limit map of gradient flow $\Phi$ in \Cref{subsec:property_phi}, which plays a very important role in the analysis.

\subsection{Properties of $\Phi$}\label{subsec:property_phi}

The limit map of gradient flow $\Phi$ lies at the core of our analysis. When LR $\eta$ is small, one can show $x_\eta(t)$ will be $O(\eta)$ close to manifold and $\Phi(x_\eta(t))$. Therefore, $\Phi(x_\eta(t))$ captures the essential part of the implicit regularization of Normalized GD and characterization of the trajectory of $\Phi(x_\eta(t))$ immediately gives us that of $\Phi(x_\eta(t))$ up to $O(\eta)$. 

Below we first recap a few important properties of $\Phi$ that will be used later this section, which makes the analysis of $\Phi(x_\eta(t))$ convenient. 

\begin{lemma}\label{lem:prop_Phi_mainpaper}
    Under \Cref{ass:manifold}, $\Phi$ satisfies the following two properties: 
    \begin{enumerate}
        \item  $\partial \Phi (x) \nabla L(x) = 0$ for any $x \in U.$ (\cref{lem:Phi_grad_dot})
        
        
        \item For any $x \in \Gamma$, if $\lambda_1(x)>
        \lambda_2(x)$, $
            \partial^2 \Phi(x) [v_1(x), v_1(x)]= -\frac{1}{ 2 } P^\perp_{x, \Gamma} \nabla \log \lambda_1(x)  $. (\Cref{lem:Phi_projection,lem:nabla2_tangentperp})
        
    \end{enumerate}
\end{lemma}
Note that $x_{\eta}(t+1) - x_{\eta}(t) = - \eta \frac{\nabla L(x_\eta(t))}{\norm{\nabla L(x_\eta(t))}}$, using a second order taylor expansion of $\Phi$, we have
\begin{align}
    \Phi ( x_{\eta}(t+1) ) - \Phi ( x_{\eta}(t) )
    &= -\eta \partial \Phi(x_\eta(t))  \frac{\nabla L(x_\eta(t))}{\norm{\nabla L(x_\eta(t))}} + \frac{\eta^2}{2} \partial^2 \Phi ( x_{\eta}(t) )  \left[ \frac{\nabla L(x_\eta(t))}{\norm{\nabla L(x_\eta(t))}}, \frac{\nabla L(x_\eta(t))}{\norm{\nabla L(x_\eta(t))}} \right] + O(\eta^3) \nonumber\\&
    =  \frac{\eta^2}{2} \partial^2 \Phi ( x_{\eta}(t) )  \left[ \frac{\nabla L(x_\eta(t))}{\norm{\nabla L(x_\eta(t))}}, \frac{\nabla L(x_\eta(t))}{\norm{\nabla L(x_\eta(t))}} \right] + O(\eta^3), \label{eq:movement_Phi_mainpaper}
\end{align}
where we use the first claim of \cref{lem:prop_Phi_mainpaper} in the final step. Therefore, we have  $\Phi ( x_{\eta}(t+1) ) - \Phi ( x_{\eta}(t) ) = O(\eta^2)$, which means $\Phi(x_\eta(t))$ moves slowly along the manifold, at a rate of at most $O(\eta^2)$ step. The Taylor expansion of $\Phi$, \eqref{eq:movement_Phi_mainpaper} plays a crucial role in our analysis for both Phase I and II and will be used repeatedly.

\subsection{Analysis for Normalized GD}\label{subsec:proof_sketch_ngd}

\paragraph{Analysis for Phase I, \Cref{thm:ngd_phase_1}:} 
The Phase I itself can be divided into two subphases: (A). Normalized GD iterate $x_\eta(t)$ gets $O(\eta)$ close to manifold; (B). counterpart of preparation phase in the quadratic case: local movement in the $O(\eta)$-neighborhood of the manifold which decreases $R_j(x_\eta(t))$ to $O(\eta^2)$. Below we sketch their proofs respectively:
\begin{itemize}
\item \textbf{Subphase (A):} First,  with a very classical result in ODE approximation theory, normalized GD with small LR will track the normalized gradient flow, which is a time-rescaled version of standard gradient flow, with $O(\eta)$ error, and enter a small neighborhoods of the manifold where Polyak-Łojasiewicz (PL) condition holds.  Since then, Normalized GD  decreases the fast loss with PL condition and the gradient has to be $O(\eta)$ small in $O(\frac{1}{\eta})$ steps. (See details in \cref{appsec:phaseIconv}).

\item \textbf{Subphase (B):} The result in subphase (B) can be viewed as a generalization of \Cref{lem:prepphase} when the loss function is $O(\eta)$-approximately quadratic, in both space and time. More specifically, it means $\norm{\nabla^2 L(\Phi(x_\eta(t))) - \nabla ^2L (x)}\le O(\eta)$ for all $x$ which is $O(\eta)$-close to some $\Phi(x_{\eta}(t'))$ with $t'-t\le O(1/\eta)$.  This is because by Taylor expansion \eqref{eq:movement_Phi_mainpaper}, $\norm{\Phi(x_\eta(t)) - \Phi(x_\eta(t'))} = O(\eta^2 (t'-t)) = O(\eta)$, and again by Taylor expansion of $\nabla^2 L$, we know  $\norm{\nabla^2 L(x)- \nabla^2 L(\Phi(x_\eta(t)))} = O (\norm{x- \Phi(x_\eta(t))}) =O(\eta)$. 

 With a similar proof technique, we show $x_{\eta}(t)$  enters ainvariant set around the manifold $\Gamma$, that is, $\{ x \in U \mid R_j(x) \le O(\eta^2), \forall j\in[D]  \}$.  Formally, we show the following analog of \Cref{lem:prepphase}:

\begin{lemma}[Preparation Phase, Informal version of \Cref{lem:phase0}]\label{lem:preparation_phase_gen}
Let  $\{x_{\eta}(t)\}_{t\ge 0}$ be the iterates of Normalized GD~\eqref{eq:normalized_gd} with LR $\eta$. If for some step $t_0$, $\norm{x_{\eta}(t_0) - \Phi(x_{\eta}(t_0))} = O(\eta)$, then for sufficiently small LR $\eta$ and all steps $t \in [t_0 + \Theta(1), \Theta(\eta^{-2})]$ steps, the iterate ${x}_{\eta}(t)$ satisfy $\max_{j\in [M]} R_j(x_\eta(t))\le O(\eta^2)$.
\end{lemma}

\end{itemize}








\paragraph{Analysis for Phase II, \Cref{thm:ngd_phase_2}:} 

Similar to the subphase (B) in the Phase I, the high-level idea here is again that $x_\eta(t) $ locally evolves like normalized GD with quadratic loss around $\Phi(x_\eta(t))$ and with an argument similar to the alignment phase of quadratic case (though technically more complicated),  we show  $x_\eta(t)-\Phi(x_\eta(t))$ approximately aligns to the top eigenvector of $\nabla^2 L(\Phi(x_\eta(t)))$, denoted by $v_1(t)$ and so does $\nabla L(x_\eta(t))$. More specifically, it corresponds to the second claim in \Cref{thm:ngd_phase_2}, that $ \frac{1}{\lfloor T_2/\eta^2 \rfloor}\sum_{t=0}^{ \lfloor T_2/\eta^2 \rfloor} \theta_t \le O(\eta )$. 

We now have a more detailed look at the movement in $\Phi$. 
Since $\Phi(x_{\eta}(t))$ belongs to the manifold, we have $\nabla L(\Phi(x_{\eta}(t)))=0$ and so $\nabla L(x_{\eta}(t)) = \nabla^2 L(\Phi(x_{\eta}(t))) (x_{\eta}(t)-\Phi(x_{\eta}(t))) + O(\eta^2)$ using a Taylor expansion. This helps us derive a relation between the Normalized GD update and the top eigenvector of the hessian (simplified version of \Cref{lem:appr_gradient_norm}):
\begin{align}
   \exists s \in \{\pm 1\}, \quad  \frac{\nabla L(x_{\eta}(t))}{ \norm{ \nabla L( x_{\eta}(t) ) } } = s v_1(t) + O(\theta_t + \eta). \label{eq:observe_normalizedGD}
\end{align}
  
Incorporating the above into the movement in $\Phi(x_{\eta}(t))$ from \cref{eq:movement_Phi_mainpaper} gives:
\begin{align}\label{eq:main_one_step_phi_pre}
    \Phi ( x_{\eta}(t+1) ) - \Phi ( x_{\eta}(t) ) = \frac{\eta^2}{2} \partial^2 \Phi ( x_{\eta}(t) )  [ v_1(t), v_1(t) ] + O(\eta^2 \theta_t + \eta^3)
\end{align}

Applying the second property of \cref{lem:prop_Phi_mainpaper} on \Cref{eq:main_one_step_phi_pre} above yields \Cref{lem:main_PhiGt}.

%
%
%


\begin{lemma}[Movement in the manifold, Informal version of \cref{lem:PhiGt}]\label{lem:main_PhiGt}
    Under the setting in \cref{thm:ngd_phase_2}, for sufficiently small $\eta$, we have at any step $t \le \lfloor T_2 / \eta^2 \rfloor$
    \begin{align*}
        \Phi( x_{\eta}(t+1) ) - \Phi( x_{\eta}(t) ) = - \frac{\eta^2}{4} P^{\perp}_{t, \Gamma} \nabla \log \lambda_1 ( t ) + O( \eta^3 + \eta^2 \theta_t ).
    \end{align*}
\end{lemma}

To complete the proof of \cref{thm:ngd_phase_2}, we show that for small enough $\eta$, the trajectory of $\Phi(x_{\eta}(\tau/\eta^2))$ is $O(\eta^3 \lfloor T_2 / \eta^2 \rfloor  + \eta^2 \sum_{t=0}^{\lfloor T_2 / \eta^2 \rfloor} \theta_t)$-close to $X(\tau)$ for any $\tau\le T_2$, where $X(\cdot)$ is the flow given by \cref{eq:log_limiting_flow}. This error is $O(\eta)$, since $\sum_{t=0}^{\lfloor T_2 / \eta^2 \rfloor} \theta_t = O(\lfloor T_2 / \eta^2 \rfloor \eta)$. 

One technical difficulty towards showing the average of $\eta_t$ is only $O(\eta)$  is that our current analysis requires $\abs{ \langle v_1(x_\eta(t)), x_\eta(t) - \Phi(x_\eta(t)) \rangle }$ doesn't vanish, that is, it remains $\Omega(\eta)$ large throughout the entire training process. This is guaranteed by \Cref{lem:alignmentphase} in quadratic case -- since the alignment monotone increases whenever it's smaller $\frac{\lambda_1}{2}$, but the analysis breaks when the loss is only approximately quadratic and the alignment $\abs{ \langle v_1(x_\eta(t)), x_\eta(t) - \Phi(x_\eta(t)) \rangle }$could decrease decrease by $O(\theta_t\eta^2)$ per step. Once the alignment becomes too small, even if the angle $\theta_t$ is small, the normalized GD dynamics become chaotic and super sensitive to any perturbation. Our current proof technique cannot deal with this case and that's the main reason we have to make the additional assumption in \Cref{thm:ngd_phase_2}. 

\textbf{Role of $\eta^{100}$ noise.} Fortunately, with the additional assumption that the initial alignment is at least $\Omega(\eta)$, we can show adding any $\text{poly}(\eta)$ perturbation (even as small as $\Omega(\eta^{100})$) suffices to prevent the aforementioned bad case, that is, $\abs{ \langle v_1(x_\eta(t)), x_\eta(t) - \Phi(x_\eta(t)) \rangle }$ stays $\Omega(\eta)$ large. The intuition why $\Omega(\eta^{100})$ perturbation works again comes from quadratic case -- it's clear that  $\Tilde x = cv_1$ for any $|c|\le1$ is a stationary point for two-step normalized GD updates for quadratic loss under the setting of \Cref{sec:quadratic}. But if $c$ is smaller than critical value determined by the eigenvalues of the hessian, the stationary point is unstable, meaning any deviation away from the top eigenspace will be amplified until the alignment increases above the critical threshold.   Based on this intuition, the formal argument, \cref{lem:Behavior_explode2nd} uses the techniques from the `escaping saddle point' analysis \citep{jin2017escape}. Adding noise is not necessary in experiments to observe the predicted behavior (see `Alignment' in \Cref{fig:cifar} where no noise is added). On one hand, it might be because the floating point errors served the role of noise. On the other hand, we suspect it's not  necessary even for theory, just like GD gets stuck at saddle point only when initialized from a zero measure set even without noise~\citep{lee2016gradient,lee2017first}.

\subsection{Analysis for GD on $\sqrt{L}$}\label{subsec:proof_sketch_sqrt_loss}
In this subsection we will make an additional assumption that   $ L(x)=0$ for all $x\in\Gamma$. 
The analysis then will follow a very similar strategy as the analysis for (Normalized) GD. 
However, the major difference from the analysis for Normalized GD comes from the update rule for $x_{\eta}(t)$ when it is $O(\eta)$-close to the manifold:
\begin{align*}
    \exists s \in \{\pm 1\}, \quad  \nabla \sqrt{L}(x_{\eta}(t)) = s \sqrt{\lambda_1(t)} v_1(t)   + O( \eta + \theta_t ).
\end{align*}
Thus, the effective learning rate is $\sqrt{ \lambda_1(t) } \eta$ at any step $t$. This shows up, when we compute the change in the function $\Phi$. Thus, we have the following lemma showcasing the movement in the function $\Phi$ with the GD update on $\sqrt{L}$:
\begin{lemma}[Movement in the manifold, Informal version of \cref{lem:phase0_sqrtL}]
     Under the setting in \cref{thm:PhaseII_sqrtL}, for sufficiently small $\eta$, we have at any step $t \le \lfloor T_2 / \eta^2 \rfloor$,
    $
        \Phi( x_{\eta}(t+1) ) - \Phi( x_{\eta}(t) ) = - \frac{\eta^2}{8} P^{\perp}_{t, \Gamma} \nabla \lambda_1 ( t ) + O( \eta^3 + \eta^2 \theta_t ).
    $

\end{lemma}

\section{Experiments}\label{sec:exp}

	Though our main theorems characterizes the dynamics of Nomalized GD and GD on $\sqrt{L}$ for sufficiently small LR, it's not clear if the predicted phenomena is related to the training with practical LR as the function and initialization dependent constants are hard to compute and could be huge. Neverthesless, in this section we show the phenomena predicted by our theorem does occur for real-life models like VGG-16. We further verify the predicted convergence to the limiting flow for Normalized GD on a two-layer fully-connected network trained on MNIST.

\begin{figure*}[!htbp]
\begin{center}
\includegraphics[width=\linewidth]{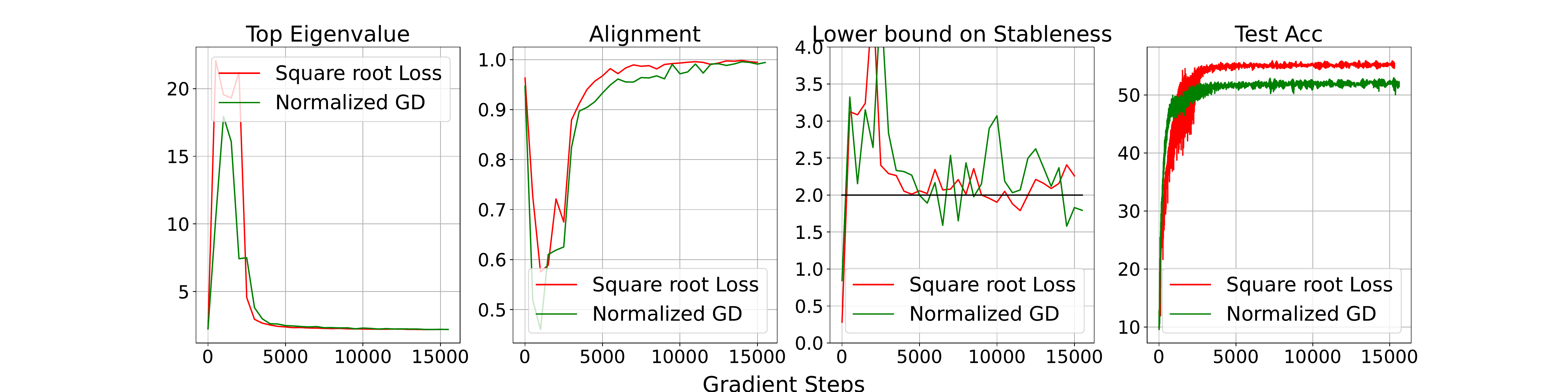}
\caption{We verify our theoretical claims in the second phase ---(a) the sharpness decreases; (b) gradient aligns with the top eigenvector of Hessian; (c) stableness will be higher than $2$ --- under the setting of training  VGG-16 on CIFAR-10 dataset with Normalized GD on $L$ and GD with $\sqrt{L}$ loss respectively.}
\label{fig:cifar}
\end{center}
\end{figure*}

\paragraph{Verification for Predicted Phenomena on Real-life Models:} We first observe the behavior of different test functions throughout the training to verify our theoretical findings.
We perform our experiments on a VGG-16 model \citep{simonyan2014very} trained on CIFAR-10 dataset \citep{CIFAR} with Normalized GD and GD with $\sqrt{L}$. For efficient full-batch training, we trained the model on a sample of randomly chosen $5000$ examples from the training dataset. To meet the smoothness requirement by our theory, we modified our network in two ways, (a) we used GeLU activation ~\citep{hendrycks2016gaussian} in place of the non-smooth ReLU activation, and (b) we used average pooling in place of the non-smooth max-pooling ~\citep{boureau2010theoretical}. We used $\ell_2$ loss instead of softmax loss to ensure the existence of minimizers and thus the manifold. 
We plot the behavior of the following four functions in \cref{fig:cifar}: Top eigenvalue of the Hessian, Alignment, Stableness, and  Test accuracy. Alignment is defined as $\frac{1}{\lambda_1 \norm{g}^2} g^{\top} (\nabla^2 L) g$, where $\nabla^2 L$ is the Hessian, $g$ is the gradient and $\lambda_1$ is the top eigenvalue of the Hessian. To check the behavior for Stableness, we plot $\frac{\eta}{\norm{g}} \times \lambda_1$ for Normalized GD and $\frac{\eta}{2\sqrt{L}} \times \lambda_1$ for GD with $\sqrt{L}$, which are lower bounds on the  Stableness of the Hessian (\ref{defi:stableness}).


We observe that the alignment function reaches close to $1$, towards the end of training. The top eigenvalue decreases over time (as predicted by\cref{thm:ngd_phase_2} and \cref{thm:PhaseII_sqrtL}), and the stableness hovers around $2$ at the end of training. 

\begin{figure*}[!htbp]
\begin{center}
\centerline{\includegraphics[width=\linewidth]{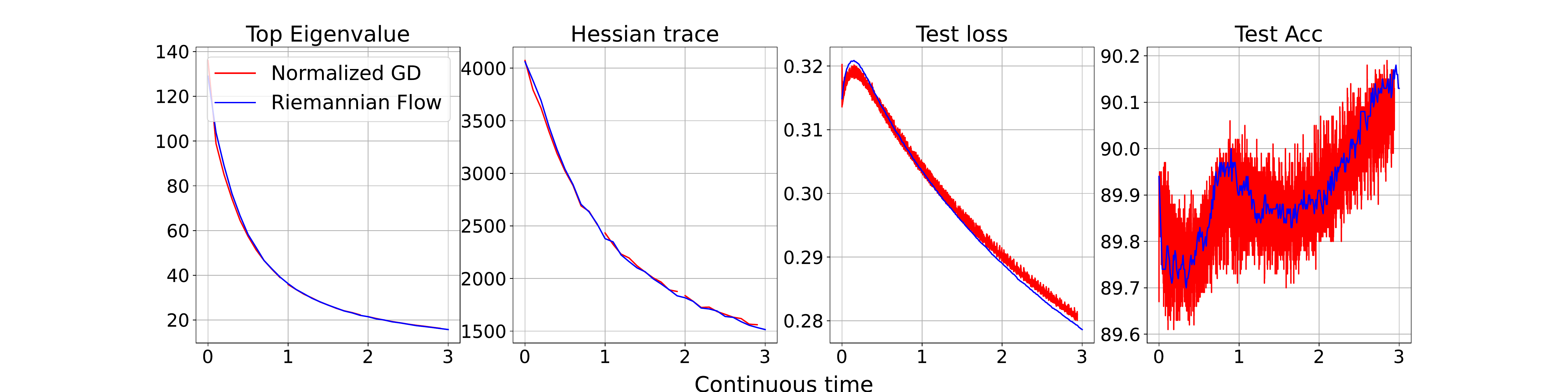}}
\caption{Normalized GD and Riemannian flow have almost the same behavior under proper time scalings, for a 2-layer network on MNIST initialized with tiny loss.}
\label{fig:MNIST}
\end{center}
\end{figure*}

\begin{figure*}[!htbp]
\begin{center}
\centerline{\includegraphics[width=0.4\linewidth]{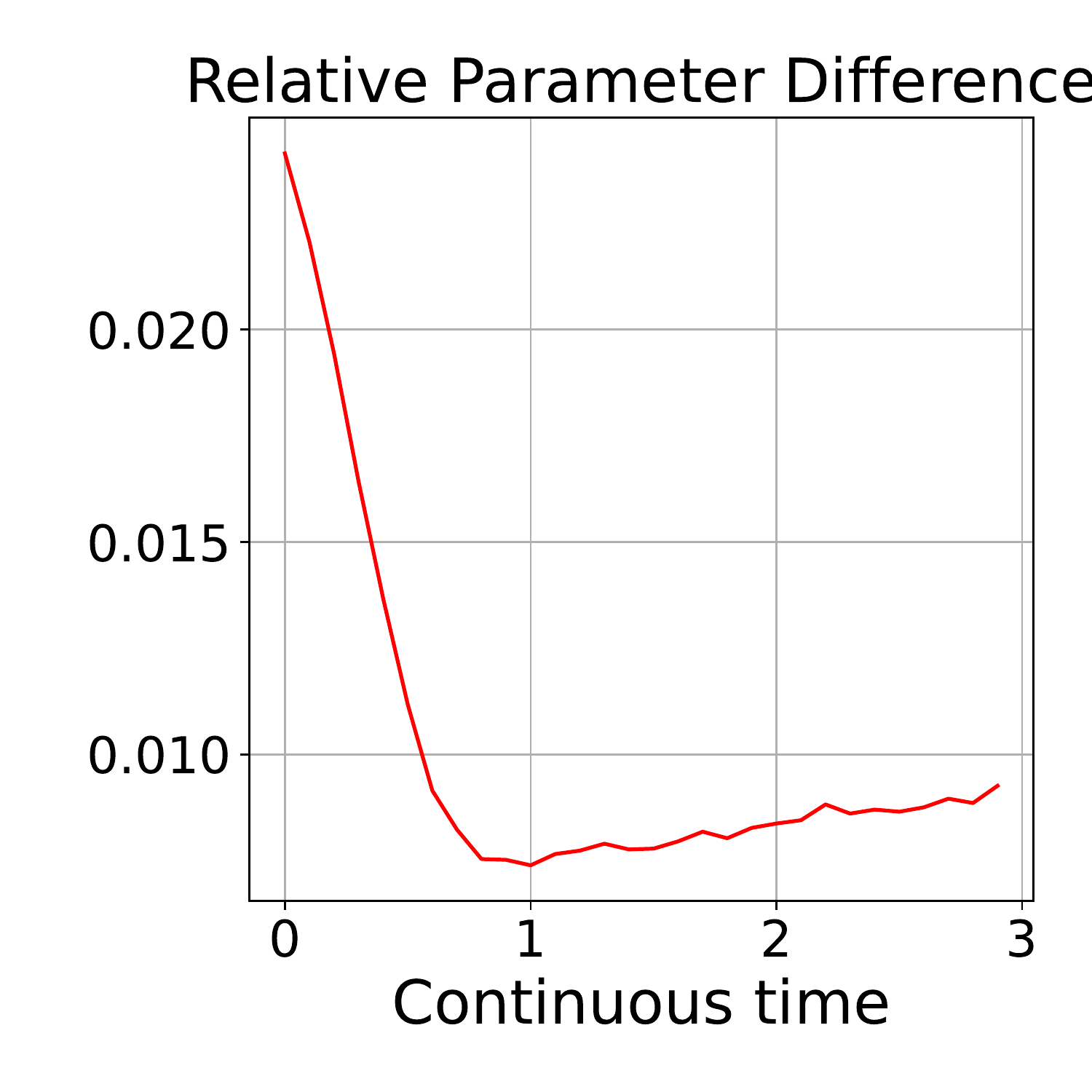}}
\caption{The trajectory of Normalized GD is very close to that of the limiting flow minimizing the sharpness on manifold, as predicted by our theory. Absolute difference is the norm of the difference between the parameters of the two trajectories at the same continuous time, while relative parameter difference is the ratio of the norm of the difference to the norm of parameters of each runs.}
\label{fig:MNIST_paramdiff}
\end{center}
\end{figure*}

\paragraph{Verifying Convergence to Limiting Flow on MNIST:}
We further verify the closeness between the Riemannian gradient flow w.r.t. the top eigenvalue and Normalized GD, as predicted by  \cref{thm:ngd_phase_2}, on a $1$ hidden-layer fully connected network on MNIST~\citep{lecun-mnisthandwrittendigit-2010}. The network had $784$ hidden units, with GeLU activation function. We use $\ell_2$ loss to ensure the existence of minimizers, which is necessary for the existence of the manifold. For efficient training on a single GPU, we train on  a random subset of training data of size $1000$.

We first trained the model with full  to reach loss of order $10^{-3}$. Starting from this checkpoint, we make two different runs, one for Normalized GD and another for Riemannian gradient flow w.r.t. the top eigenvalue (see \cref{sec:add_expt_details} for details). We plot the behavior of the network w.r.t. continuous time defined for Normalized GD as $\#\mathrm{Gradient Steps} \times \eta^2 / 4$, and for Riemannian flow as $\#\mathrm{Gradient Steps} \times \eta$, where $\eta$ is the learning rate. We track the behavior of Test Loss, Test accuracy, the top eigenvalue of the Hessian and also the trace of the Hessian in \cref{fig:MNIST}. We see that there is an exact match between the behavior of the four functions, which supports our theory. Moreover, \cref{fig:MNIST_paramdiff} computes the norm of the difference in the parameters between the two runs, and shows that the runs stay close to each other in the parameter space throughout training.


\section{Conclusion}
The recent discovery of Edge of Stability phenomenon in \citet{cohen2021gradient} calls for a reexamination of how we understand optimization in deep learning. The current paper gives two concrete settings with fairly general loss functions,  where gradient updates can be shown to decrease loss over many iterations even after stableness is lost. Furthermore, in one setting the trajectory is shown to amount to reduce the sharpness (i.e., the maximum eigenvalue of the Hessian of the loss), thus rigorously establishing an effect that has been  conjectured for decades in deep learning literature and was definitively documented for GD in \citet{cohen2021gradient}. Our analysis crucially relies upon learning rate $\eta$ being finite, in contrast to many recent results on implicit bias that required an infinitesimal LR. 
Even the alignment analysis of Normalized GD to the top eigenvector for quadratic loss in \Cref{sec:quadratic} appears to be new. 

One limitation of our analysis is that it only applies close to the manifold of local minimizers. By contrast, in experiments 
the EoS phenomenon, including the control of sharpness, begins much sooner. Addressing this gap, as well as analysing the EoS for the loss $L$ itself (as opposed to $\sqrt{L}$ as done here) is left for future work. Very likely this will require novel understanding of properties of deep learning losses, which we were able to circumvent by looking at $\sqrt{L}$ instead. Exploration of EoS-like effects in  SGD setting would also be interesting, although we first need definitive experiments analogous to \citet{cohen2021gradient}.

\section*{Acknowledgement}
We thank Kaifeng Lyu for helpful discussions.
The authors acknowledge support from NSF, ONR, Simons Foundation, Schmidt Foundation,
Mozilla Research, Amazon Research, DARPA and SRC. ZL is also supported by Microsoft Research
PhD Fellowship.

\bibliographystyle{plainnat}
\bibliography{reference.bib}


\newpage 
\appendix
\onecolumn

\tableofcontents

\section{Omitted Proofs for Results for Quadratic Loss Functions}\label{appsec:quadraticmodel}



We first recall the settings and notations. Let $A$ be a positive definite matrix. Without loss of generality, we can assume $A$ is diagonal, \emph{i.e.}, $A = \diag(\lambda_1,\lambda_2,\ldots, \lambda_D)\in \RR^{D \times D}$, where $\lambda_1>\lambda_2\ge\lambda_3\ge\ldots\ge \lambda_D>0$ and the eigenvectors are the standard basis vectors $e_1, \cdots, e_D$ of the $D$-dimensional space. We will  denote $\projd{j}= \sum_{i=j}^D e_ie_i^\top$ as the projection matrix onto the subspace spanned by $e_j, \ldots, e_D$.

Recall the loss function $L$ is defined as $  L(x) = \frac{1}{2} x^{\top} A x$. The Normalized GD update (LR= $\eta$ )is given by
$    x(t+1) = x(t) - \eta \frac{A x(t)}{\norm{A x(t)}}$.
A substitution $\Tilde x(t):= \frac{Ax(t)}{\eta}$ gives the following update rule:
\begin{align*}
\quadraticxtilde
\end{align*}
      
Note Normalized GD~\eqref{eq:ngd_tilde_quadratic} is not defined at $\norm{\Tilde x(t)} = 0$. Moreover, it's easy to check that if at some time step $t$ $| \langle v_1, \Tilde x(t)\rangle| =0 $,  $| \langle v_1, \Tilde x(t')\rangle| =0 $ holds for any $t'\ge t$. 
Thus it's necessary to assume $| \langle v_1, \Tilde x(t)\rangle|\neq 0$ for all $t\in\mathbb{N}$ in order to prove alignment to the top eigenvector of $A$ for Normalized GD~\eqref{eq:ngd_tilde_quadratic}. 


Now we recall the main theorem for Normalized GD on quadratic loss functions:
\mainlemmaquadratic*

We also note that GD on $\sqrt{L}$ with any LR $\eta$ can also be reduced to update rule~\eqref{eq:ngd_tilde_quadratic}, as shown in the discussion at the end of \Cref{sec:quadratic}.

\subsection{Proofs for Preparation Phase}


In this subsection, we  show (1). $\mathcal{I}_j$ is indeed an invariant set for normalized GD $\forall j\in [D]$ and (2). from any initialization, normalized GD will eventually go into their intersection $\cap_{j=1}^D \mathcal{I}_j$.

\begin{lemma}\label{lem:quadratic_norm_invariant}
For any $t\in\mathbb{N}$ and $j\in [D]$, $\norm{\projd{j}\Tilde x(t)}\le \lambda_j \implies \norm{\projd{j}\Tilde x(t+1)}\le \lambda_j$. In other words, $\{\mathcal{I}_j\}_{j=1}^D$ are invariant sets of update rule \Cref{eq:ngd_tilde_quadratic}.
\end{lemma}

\begin{proof}[Proof of \Cref{lem:quadratic_norm_invariant}]
Note $ \projd{j} A =  \projd{j} A  \projd{j}$, by definition of Normalized GD~\eqref{eq:ngd_tilde_quadratic}, we have 	
\begin{align*}
    \projd{j}\Tilde x(t+1) 
    = \projd{j}\Tilde x(t)- \projd{j} A \frac{\Tilde x(t)}{\norm{\Tilde x(t)} }
    = \left(I -  \frac{\projd{j} A}{\norm{\Tilde x(t)} }\right) \projd{j} \Tilde x(t),
    \end{align*}
which implies 
\begin{align}\label{eq:quadratic_ngd_proj_norm}
	\norm{\projd{j}\Tilde x(t+1) } \le  \norm{I -  \frac{\projd{j} A}{\norm{\Tilde x(t)} }}  \norm{\projd{j} \Tilde x(t)}.
\end{align}

Note that $\projd{j}A \preccurlyeq \lambda_j I$, $\norm{\projd{j}\Tilde x(t)}\le \norm{\Tilde x(t)}$ and $\norm{\projd{j} \Tilde x(t)}\le \lambda_j$ by assumption, we have
\begin{align*}
 - \frac{\lambda_j}{\norm{\projd{j}\Tilde x(t) }} I 
 \preccurlyeq -  \frac{\projd{j} A}{\norm{\Tilde x(t)} } 
 \preccurlyeq I -  \frac{\projd{j} A}{\norm{\Tilde x(t)} } 
 \preccurlyeq I 
 \preccurlyeq  \frac{\lambda_j}{\norm{\projd{j}\Tilde x(t) }} I.
\end{align*}
Therefore $\norm{I -  \frac{\projd{j} A}{\norm{\Tilde x(t)} } } \le \frac{\lambda_j}{\norm{\projd{j}\Tilde x(t) }}  $ and thus we conclude  ${\norm{\projd{j}\Tilde x(t+1) }}\le \lambda_j$.
\end{proof}

\begin{lemma}\label{lem:quadratic_prep_decrease}
For any $t\in\mathbb{N}$ and $j\in [D]$, if  $\norm{\projd{j}\Tilde{x}(t)} \ge \lambda_j$, then $\norm{ \projd{j}\Tilde{x}(t+1) }\le (1- \frac{\lambda_D}{\norm{\Tilde x(t)}})\norm{ \projd{j}\Tilde{x}(t) }$. 
\end{lemma}

\begin{proof}[Proof of \Cref{lem:quadratic_prep_decrease}]
Since $\lambda_j\le \norm{\projd{j}\Tilde x(t)}\le \norm{\Tilde x(t)}$, 
we have $0\preccurlyeq I -  \frac{\projd{j} A}{\norm{\Tilde x(t)} } \preccurlyeq 1- \frac{\lambda_D}{\norm{\Tilde x(t)} }$. 
Therefore $\norm{I -  \frac{\projd{j} A}{\norm{\Tilde x(t)} }} \le  1- \frac{\lambda_D}{\norm{\Tilde x(t)} }$. The proof is completed by plugging this into \Cref{eq:quadratic_ngd_proj_norm}.
\end{proof}

\Cref{lem:quadratic_prep_decrease} has the following two direct corollaries.
\begin{corollary}\label{lem:invariant1}
For any initialization $\Tilde x(0)$ and $t\ge \frac{\norm{\Tilde x(0)} - \lambda_1}{\lambda_D}$, $\norm{\Tilde x(t)}\le \lambda_1$, that is, $\Tilde x(t)\in \mathcal{I}_1$.
\end{corollary}
\begin{proof}[Proof of \Cref{lem:invariant1}]
	Set $j=1$ in \Cref{lem:quadratic_prep_decrease}, it holds that $\norm{\Tilde x(t+1)}\le \norm{\Tilde x(t)} - \lambda_D$ whenever $\norm{\Tilde x(t)}\ge \lambda_1$. Thus  $\norm {\Tilde x(\big \lceil\frac{\norm{\Tilde x(0)}- \lambda_1}{\lambda_D} \big \rceil)}\le \lambda_1$. The proof is completed as $\mathcal{I}_1$ is an invariant set by \Cref{lem:quadratic_norm_invariant}.
\end{proof}

\begin{corollary}\label{lem:invariantj} For any coordinate $j \in [D]$ and initial point $\Tilde x(0)\in \mathcal{I}_1$, if $t \ge  \frac{\lambda_1}{\lambda_D}\ln \frac{  \lambda_1} { \lambda_j }$ then $\norm{\projd{j} \Tilde x(t)}\le \lambda_j$.
\end{corollary}

\begin{proof}[Proof of \Cref{lem:invariantj}]

Since $\mathcal{I}_1$ is an invariant set, we have $\norm{\Tilde x(t)}\le \lambda_1$ for all $t\ge 0$. Thus let $T= \lfloor\frac{\lambda_1}{\lambda_D}\ln \frac{  \lambda_1} { \lambda_j }\rfloor$, we have 
\begin{equation*}
\norm{\projd{j}\Tilde x(T)} \le e^{-T\frac{\lambda_D}{\lambda_1}} \norm{\projd{j}\Tilde x(0)}\le \frac{\lambda_j}{\lambda_1}\norm{\Tilde x(0)} \le \lambda_j.
\end{equation*}
The proof is completed since $I_j$ is a invariant set for any $j\in [D]$ by \Cref{lem:quadratic_norm_invariant}.
\end{proof}

\subsection{Proofs for Alignment Phase}

In this subsection, we  analyze how normalized GD  align to the top eigenvector once it goes through the preparation phase, meaning $\Tilde x(t)\in \cap_{j=1}^D\mathcal{I}_j$ for all $t$ in alignment phase. 

\lemeteratenormdrops*
\begin{proof}
 The update at step $t$ as:
\begin{align*}
    \Tilde{x}(t+1) = \frac{1}{\norm{\Tilde{x}(t)}} \left( \norm{\Tilde{x}(t)} I - A \right) \Tilde{x}(t) 
    = \frac{1}{\norm{\Tilde{x}(t)}} \begin{bmatrix}
           ( \norm{\Tilde{x}(t)} - \lambda_1 ) \Tilde{x}_1(t)  \\
           ( \norm{\Tilde{x}(t)} - \lambda_2 ) \Tilde{x}_2(t) \\
           \vdots \\
           ( \norm{\Tilde{x}(t)} - \lambda_D ) \Tilde{x}_D(t)
         \end{bmatrix}.
\end{align*}

Let the index $k$ be the smallest integer such that $\lambda_{k+1} < 2 \norm{\Tilde{x}(t)} - \lambda_1$.  If no such index exists, then one can observe that $\norm{ \Tilde{x}(t+1) } \le \lambda_1 - \norm{ \Tilde{x} (t) }$. Assuming that such an index exists in $[D]$, we have $\lambda_{k} \ge 2 \norm{\Tilde{x}(t)} - \lambda_1$ and $\norm{\Tilde x(t)}-\lambda_j\le \lambda_1- \norm{\Tilde x(t)}$, $\forall j\le k$.  Now consider the following vectors:  
\begin{align*}
&v^{(1)}(t):= (\lambda_1 -\norm{\Tilde x(t)}) \Tilde x(t),\\
&v^{(2)}(t):=(2\norm{\Tilde x(t)}-\lambda_1-\lambda_k) \projd{k}\Tilde x(t), \\
&v^{(2+j)}(t):=(\lambda_{k+j-1}-\lambda_{k+j}) \projd{k+j}\Tilde x(t), \forall 1\le j\le D-k.
\end{align*}

By definition of $k$, $|\norm{\Tilde x(t)}-\lambda_j| \le |\norm{\Tilde x(t)}-\lambda_1|$. Thus
\begin{align*}
    \norm{ \Tilde{x}(t+1) } & 
    \le \frac{1}{\norm{\Tilde{x}(t)}} \norm{ \begin{bmatrix}
           ( \norm{\Tilde{x}(t)} - \lambda_1 ) \Tilde{x}_1(t)  \\
           \vdots \\
           ( \norm{\Tilde{x}(t)} - \lambda_1 ) \Tilde{x}_k(t) \\
           ( \norm{\Tilde{x}(t)} - \lambda_{k+1} ) \Tilde{x}_{k+1}(t) \\
                      \vdots \\
           ( \norm{\Tilde{x}(t)} - \lambda_D ) \Tilde{x}_D(t)
         \end{bmatrix} } \\
    & =  \frac{1}{\norm{\Tilde{x}(t)}}\norm{v^{(1)}(t) + v^{(2)}(t) + \ldots + v^{(D-k+2)}(t)}     \\
    &\le \frac{1}{\norm{\Tilde{x}(t)}} \left( \norm{ v^{(1)}(t) } + \norm{ v^{(2)} (t) } + \ldots + \norm{ v^{(D-k+2)}(t) } \right).     
\end{align*}

By assumption, we have $\Tilde x(t)\in \cap_{j=1}^D \mathcal{I}_j$. Thus
\begin{align*}
    &\norm{v^{(1)}(t)} = ( \lambda_1 - \norm{\Tilde{x}(t)} )\norm{\Tilde{x}(t)} \\&
    \norm{v^{(2)}(t)} \le ( 2\norm{\Tilde{x}(t)} - \lambda_1 - \lambda_{k} ) \lambda_k \\&
    \norm{v^{(2+j)}(t)} \le ( \lambda_{k - 1 + j} - \lambda_{k  + j} ) \lambda_{k+ j}, \text{ for all } j \ge 1.
\end{align*}
Hence,
\begin{align*}
    \sum_{j \ge 2} \norm{v^{(j)}(t)} &= ( 2\norm{\Tilde{x}(t)} - \lambda_1 - \lambda_{k} ) \lambda_k + \sum_{j \ge k} ( \lambda_{j} - \lambda_{j + 1} ) \lambda_{j + 1} \\& 
    = (2\norm{\Tilde{x}(t)} - \lambda_1) \lambda_k + \sum_{j \ge k} \lambda_j \lambda_{j + 1} - \sum_{j \ge k} \lambda_j^2 \\&
    \le \frac{ (2\norm{\Tilde{x}(t)} - \lambda_1)^2 + \lambda_k^2 }{2} + \sum_{j \ge k} \frac{\lambda_j^2 +  \lambda_{j + 1}^2}{2} - \sum_{j \ge k} \lambda_j^2 \\&
    \le \frac{ (2\norm{\Tilde{x}(t)} - \lambda_1)^2  }{2} - \frac{\lambda_D^2}{2},
\end{align*}
where we applied AM-GM inequality multiple times in the pre-final step.

Thus,
\begin{align*}
    \norm{ \Tilde{x}(t+1) } &\le \frac{1}{\norm{\Tilde{x}(t)}} \left( \norm{ v^{(1)}(t) } + \norm{ v^{(2)} (t) } + \ldots + \norm{ v^{(D-k+1)}(t) } \right) \\&
    \le \frac{ (2\norm{\Tilde{x}(t)} - \lambda_1)^2  }{2 \norm{\Tilde{x}(t)} } - \frac{\lambda_D^2}{2 \norm{\Tilde{x}(t)} } + \lambda_1 - \norm{\Tilde{x}(t)} \\&
    = \norm{\Tilde{x}(t)} + \frac{\lambda_1^2 - \lambda_D^2}{2 \norm{\Tilde{x}(t)}} - \lambda_1\\&
    \le \frac{\lambda_1}{2} - \frac{\lambda_D^2}{2\lambda_1},
\end{align*}
where the final step is because $\frac{\lambda_1}{2} \le \norm{\Tilde{x}(t)} \le \lambda_1$ and that the maximal value of a convex function is attained at the boundary of an interval.

\end{proof}

\begin{lemma}\label{lem:behaviornormhalf}
At any step $t$ and $i\in[D]$, if $\norm{\Tilde{x}(t)} \gtreqqless \frac{\lambda_i}{2}$, then $\abs{\Tilde{x}_i(t+1)} \lesseqqgtr \abs{\Tilde{x}_i(t)}$, where $\gtreqqless$ denotes larger than, equal to and smaller than respectively. (Same for $\lesseqqgtr$, but in the reverse order)
\end{lemma}

\begin{proof}
From the Normalized GD update rule, we have
      $\Tilde{x}_i(t+1) = \Tilde{x}_i(t) \left(1 - \frac{\lambda_i}{\norm{\Tilde{x}(t)}} \right), \text{ for all } i \in [D]$.
Thus
\begin{align*}
	\frac{\lambda_1}{\norm{\Tilde{x}(t)}} \lesseqqgtr 2
	 \Longleftrightarrow \abs{1 - \frac{\lambda_1}{\norm{\Tilde{x}(t)}}} \lesseqqgtr 1 
	\Longleftrightarrow \abs{\Tilde{x}_i(t+1)} \lesseqqgtr \abs{\Tilde{x}_i(t)},
\end{align*}
 which completes the proof.
\end{proof}

\begin{lemma}\label{lem:normbound_onincrease} 
At any step $t$, if $\norm{\Tilde{x}(t)} \le \frac{\lambda_1}{2}$, then 
\begin{align*}
    ( \lambda_1 - \norm{\Tilde{x}(t)} ) \cos \theta_t \le \norm{\Tilde{x}(t+1)} \le \lambda_1 - \norm{\Tilde{x}(t)} - \frac{\lambda}{2\lambda_1} \left( 1- \frac{\lambda}{\lambda_1} \right)  \lambda_1 \sin^2 \theta_t,
\end{align*}
where $\theta_t = \arctan \frac{ \norm{ \projd{2} \Tilde{x}(t) } }{ \abs{ e_1^{\top} \Tilde{x}(t) } }$ and $\lambda = \min(\lambda_1 - \lambda_2, \lambda_D)$.
\end{lemma}

\begin{proof}

We first show that the left side inequality holds by the following update rule for $\langle e_1, \Tilde{x}(t) \rangle$:
\begin{align*}
    \langle e_1, \Tilde{x}(t+1) \rangle = ( \norm{\Tilde{x}(t)} - \lambda_1) \frac{\langle e_1, \Tilde{x}(t)\rangle}{ \norm{ \Tilde{x}(t) } }.
\end{align*}
Since $\norm{ \Tilde{x}(t+1) } \ge \abs{ \langle e_1, \Tilde{x} (t+1) \rangle }$ and $\theta_t$ denotes the angle between $e_1$ and $\Tilde{x}(t+1)$, we get the left side inequality.

Now, we focus on the right hand side inequality. First of all, the update in the coordinate $j \in [2, D]$ is given by
\begin{align*}
    \langle e_j, \Tilde{x}(t+1) \rangle = ( \norm{\Tilde{x}(t)} - \lambda_j ) \frac{\langle e_j, \Tilde{x}(t)\rangle}{ \norm{ \Tilde{x}(t) } }.
\end{align*}

Then, we have
\begingroup
\allowdisplaybreaks
\begin{align*}
    \norm{ \Tilde{ x } (t+1) }^2 &= \sum_{j = 1}^{D} \langle e_j, \Tilde{x} (t+1) \rangle^2 \\&
    =   \sum_{j = 1}^{D} ( \norm{\Tilde{x}(t)} - \lambda_j )^2 \left ( \frac{\langle e_j, \Tilde{x}(t)\rangle}{ \norm{ \Tilde{x}(t) } } \right)^2 \\&
    = ( \norm{\Tilde{x}(t)} - \lambda_1 )^2 \cos^2 \theta_t +   \sum_{j = 2}^{D} ( \norm{\Tilde{x}(t)} - \lambda_j )^2 \left ( \frac{\langle e_j, \Tilde{x}(t)\rangle}{ \norm{ \Tilde{x}(t) } } \right)^2 \\&
    \le ( \norm{\Tilde{x}(t)} - \lambda_1 )^2 \cos^2 \theta_t +  ( \norm{\Tilde{x}(t)} - \overline{\lambda} )^2 \sum_{j=2}^{D} \left ( \frac{\langle e_j, \Tilde{x}(t)\rangle}{ \norm{ \Tilde{x}(t) } } \right)^2 \\&
    = ( \norm{\Tilde{x}(t)} - \lambda_1 )^2 \cos^2 \theta_t +  ( \norm{\Tilde{x}(t)} - \overline{\lambda} )^2 \sin^2 \theta_t \\&
    = ( \norm{\Tilde{x}(t)} - \lambda_1 )^2 + (\lambda_1 - \overline{\lambda}) (2 \norm{\Tilde{x}(t)} - \overline{\lambda} - \lambda_1) \sin^2 \theta_t \\&
    \le ( \norm{\Tilde{x}(t)} - \lambda_1 )^2 - \overline{\lambda} (\lambda_1 - \overline{\lambda}) \sin^2 \theta_t,
\end{align*}
where in the fourth step, we have used $\overline{\lambda} = \argmax_{\lambda_i \mid 2 \le i \le D} \abs{ \norm{\Tilde{x}(t)} - \lambda_i }.$ The final step uses $\norm{\Tilde{x}(t)} < \frac{\lambda_1}{2}$. Hence, using the fact that $\sqrt{1 - y} \le 1 - y/2$ for any $y \le 1$, we have
\begin{align*}
     \norm{ \Tilde{ x } (t+1) } &\le  \lambda_1 - \norm{\Tilde{x}(t)}  - \frac{1}{2 (  \lambda_1 - \norm{\Tilde{x}(t)} ) } \overline{\lambda} (\lambda_1 - \lambda) \sin^2 \theta_t \\&
     \le \lambda_1 - \norm{\Tilde{x}(t)}  - \frac{\overline{\lambda}}{2\lambda_1} \left( 1 - \frac{\overline{\lambda}}{\lambda_1} \right) \lambda_1 \sin^2 \theta_t,
\end{align*}
where again in the final step, we have used $\norm{\Tilde{x}(t)} < \frac{\lambda_1}{2}$. 
The above bound can be further bounded by
\begin{align*}
    \norm{ \Tilde{ x } (t+1) } 
     &\le \lambda_1 - \norm{\Tilde{x}(t)}  - \frac{\overline{\lambda}}{2\lambda_1} \left(1 - \frac{\overline{\lambda}}{\lambda_1} \right) \lambda_1 \sin^2 \theta_t \\&
     \le \lambda_1 - \norm{\Tilde{x}(t)}  - \frac{1}{2}\left(  \min_{\lambda' \in \{ \lambda_2, \lambda_D \} }  \frac{\lambda'}{\lambda_1} \left(1 - \frac{\lambda'}{\lambda_1} \right) \right) \lambda_1 \sin^2 \theta_t \\&
     = \lambda_1 - \norm{\Tilde{x}(t)}  - \frac{1}{2}\left(    \frac{ \lambda }{\lambda_1} \left(1 - \frac{\lambda}{\lambda_1} \right) \right) \lambda_1 \sin^2 \theta_t, 
\end{align*}
where we have used $\lambda = \min (\lambda_1 - \lambda_2, \lambda_D).$
\endgroup

\end{proof}

%

\begin{lemma}\label{lem:firscoordincr_ondrop}
If at some step $t$, $\norm{\Tilde{x}(t+1) } + \norm{\Tilde{x}(t)} \le \lambda_1$, then $\abs{\Tilde{x}_1(t+2)} \ge \abs{\Tilde{x}_1(t)}$, where the equality holds only when $\norm{\Tilde{x}(t+1) } + \norm{\Tilde{x}(t)} = \lambda_1$. Therefore, by \Cref{lem:normbound_onincrease}, we have :
\begin{align*}
   \norm{\Tilde x(t)} \le \frac{\lambda_1}{2} \implies \abs{\Tilde{x}_1(t+2)} \ge \abs{\Tilde{x}_1(t)} (1 + 2\frac{\lambda}{\lambda_1} (1 - \frac{\lambda}{\lambda_1}) \sin^2 \theta_t)  ,
\end{align*}
where $\theta_t = \arctan \frac{ \norm{ \projd{2} \Tilde{x}(t) } }{ \abs{ e_1^{\top} \Tilde{x}(t) } },$ and $\lambda = \min(\lambda_1 - \lambda_2, \lambda_D)$.
\end{lemma}

\begin{proof}[Proof of \Cref{lem:firscoordincr_ondrop}]
Using the Normalized GD update rule, we have
\begin{align*}
    \Tilde{x}_1(t + 1) = \left(1 - \frac{\lambda_1}{ \norm{ \Tilde{x}(t) }} \right) \Tilde{x}_1(t), \quad 
    \Tilde{x}_1(t + 2) = \left(1 - \frac{\lambda_1}{ \norm{ \Tilde{x}(t+1) }} \right) \Tilde{x}_1(t+1).
\end{align*}
Combining the two updates, we have
\begin{align*}
    \abs{ \Tilde{x}_1(t + 2) } &= \abs{ \left(1 - \frac{\lambda_1}{ \norm{ \Tilde{x}(t) }} \right) \left(1 - \frac{\lambda_1}{ \norm{ \Tilde{x}(t+1) }} \right) } \abs{ \Tilde{x}_1(t) } \\&
    = \abs{ 1 + \frac{\lambda_1^2 - \lambda_1(\norm{ \Tilde{x}(t) } - \norm{ \Tilde{x}(t+1) } )}{ \norm{ \Tilde{x}(t) }  \norm{ \Tilde{x}(t+1) } } } \abs{ \Tilde{x}_1(t) } \\&
    \ge \abs{ \Tilde{x}_1(t) },
\end{align*}
where the equality holds only when $\norm{\Tilde{x}(t+1) } + \norm{\Tilde{x}(t)} = \lambda_1$.

Moreover, with the additional condition that $\norm{ \Tilde{ x } (t) } < \frac{\lambda_1}{2}$, we have from \Cref{lem:normbound_onincrease}, $\norm{\Tilde{x}(t+1)} \le \lambda_1 - \norm{\Tilde{x}(t) } - \lambda (\lambda_1 - \lambda) \sin^2 \theta_t$, where $\lambda = \min(\lambda_1 - \lambda_2, \lambda_D)$.

Hence, retracing the steps we followed before, we have
\begin{align*}
    \abs{ \Tilde{x}_1(t + 2) } &= \abs{ 1 + \frac{\lambda_1^2 -\lambda_1( \norm{ \Tilde{x}(t) } + \norm{ \Tilde{x}(t+1) } ) }{ \norm{ \Tilde{x}(t) }  \norm{ \Tilde{x}(t+1) } } } \abs{ \Tilde{x}_1(t) } \\&
    \ge \abs{ 1 + \frac{ \lambda (\lambda_1 - \lambda) \sin^2 \theta_t }{ \norm{\Tilde{x}(t) } \norm{ \Tilde{x} (t+1) } } } \abs{ \Tilde{x}_1(t) } \\&
    \ge  \abs{ 1 + 2 \frac{\lambda}{\lambda_1} (1 - \frac{\lambda}{\lambda_1}) \sin^2 \theta_t  } \abs{ \Tilde{x}_1(t) }, 
\end{align*}
where the final step follows from  $\norm{\Tilde{x} (t+1)} \le \lambda_1 - \norm{\Tilde{x} (t)}$ and therefore $\norm{\Tilde{x} (t+1)}\norm{\Tilde{x} (t)}\le \frac{\lambda_1^2}{4}$.
\end{proof}

\subsection{Proof of Main theorems for Quadratic Loss}\label{sec:main_lemma_quadratic}

\begin{proof}[Proof of \Cref{thm:main_lemma_ngd_quadratic}]
The analysis will follow in two phases:
\begin{enumerate}
    \item \textbf{Preparation phase: }  $\Tilde{x}(t)$ enters and stays in an invariant set around the origin, that is, $\cap_{j=1}^D \mathcal{I}_j$, where $\mathcal{I}_j:=  \{ \Tilde{x}\mid \sum_{i=j}^{D} \langle e_i, \Tilde{x}(t) \rangle^2  \le \lambda^2_j\}$. (See \Cref{lem:prepphase}, which is a direct consequence of \Cref{lem:quadratic_norm_invariant,lem:invariant1,lem:quadratic_norm_invariant}.)
    
    \item \textbf{Alignment phase: } The projection of $\Tilde{x}(t)$ on the top eigenvector, $|\inner{\Tilde{x}(t)}{e_1}|$, is shown to increase monotonically among the steps among the steps $\{t\mid \norm{\Tilde{x}(t)} \le 0.5\}$, up until convergence, since it's bounded. (\Cref{lem:alignmentphase})

		By \Cref{lem:firscoordincr_ondrop}, the convergence of $|\inner{\Tilde{x}(t)}{e_1}|$ would imply the convergence of $\Tilde x(t)$ to $e_1$ in direction.
\end{enumerate}

Below we elaborate the convergence argument in the alignment phase. For convenience, we will use $\theta_t$ to denote the angle between $e_1$ and $\Tilde x(t)$ and we assume $\Tilde (0)\in \cap_{j=1}^D \mathcal{I}_j$ without loss of generality. We first define  $S:=\{t\in \mathbb{N}\mid \norm{\Tilde x(t)}\le \frac{\lambda_1}{2}\}$ and $S':=\{t\in S\mid t+2 \in S\}$.  The result in alignment phase says that $\frac{1}{\lambda_1}\abs{\Tilde x_1(t)} $ monotone increases and converges to some constant $C\in(0,\frac{1}{2}]$ among all $t\in S$, thus $\lim\limits_{t\to\infty,t\in S'} \frac{\abs{\Tilde x_1(t+2)}}{\abs{\Tilde x_1(t)}}=1 $. 
By \Cref{lem:firscoordincr_ondrop}, we have $\lim\limits_{t\to\infty,t\in S'} \theta_t = 0$. Since the one-step update function $F(\Tilde x) = \Tilde x- A \frac{\Tilde x}{\norm{\Tilde x}}$ is uniformly lipschitz when $\norm{\Tilde x}$ is bounded away from zero, we know $\lim\limits_{t\to\infty,t\in S'} \theta_{t+k} = 0,\ \forall k\in \mathbb{N}$. 

Now we claim $\forall t\ge 3$, there is some $k\in \{0,1,3\}$ such that $t-k \in S'$. This is because  \Cref{lem:iteratenorm_drops} says that if $t\notin S$, then both $t-1,t+1 \in S$. Thus for any $t\notin S$, $t-1 \in S'$. Therefore, for any $t \in S/S'$,  if $t-2 \notin S$, then $t-3\in S'$. Thus we conclude that $\forall t\ge 3$, there is some $k\in \{0,1,3\}$ such that $t-k \in S'$, which implies $\lim\limits_{t\to\infty} \theta_{t} = 0$. Hence $\lim\limits_{t\to \infty} \norm{\Tilde x(t+1)- \Tilde x(t)} = \lambda_1$, meaning for sufficiently large $t$, $\Tilde x_1(t)$ flips its sign per step and thus $\lim\limits_{t\to\infty} \Tilde x(t+2)- \Tilde x(t) = 0$, $\lim\limits_{t\to \infty} \norm{\Tilde x(t+1)} +\norm{ \Tilde x(t)} = \lambda_1$. 

If $C=\frac{1}{2}$, then we must have $\lim\limits_{t\to \infty} \norm{ \Tilde x(t)} =\frac{\lambda_1}{2}$ and we are done in this case. If $C<\frac{1}{2}$, note that $\lim\limits_{t\to \infty, t\in S'} \abs{ \Tilde x_1(t)}=C\lambda_1$, it must hold that $\lim\limits_{t\to \infty, t\in S'} \norm{ \Tilde x(t+1)} = (1-C)\lambda_1$, thus there is some large $T\in S$ such that for all $t\in S, t\ge T$, $t+1\notin S$. By \Cref{lem:iteratenorm_drops}, $t+2 \in S$. Thus we conclude $\lim\limits_{t\to\infty} \Tilde x(T+2t) = C\lambda s e_1$ for some $s\in \{-1,1\}$ and thus $\lim\limits_{t\to\infty} \Tilde x(T+2t+1) = (C-1)\lambda s e_1$. This completes the proof.
\end{proof}

\subsection{Some Extra Lemmas (only used in the general loss case)}
For a general loss function $L$ satisfying \Cref{ass:manifold}, the loss landscape looks like a strongly convex quadratic function locally around its minimizer. When sufficient small learning rate, the dynamics will be sufficiently close to the manifold and behaves like that in quadratic case with small perturbations. Thus it will be very useful to have more refined analysis for the quadratic case, as they allow us to bound the error in the approximate quadratic case quantitatively.  \Cref{cor:iteratenorm_drops,lem:incr_in_other_coordinates,lem:quadratic_tan_two_step,lem:behaviornormhalf_theta_decrease} are such examples. Note that they are only used in the proof of the general loss case, but not in the quadratic loss case.

\Cref{cor:iteratenorm_drops} is a slightly generalized version of \Cref{lem:iteratenorm_drops}.
\begin{lemma}\label{cor:iteratenorm_drops}
Suppose at time  $t$, $ \norm{ \projd{j} \Tilde{x}(t)} \le \lambda_j(1+ \frac{\lambda^2_D}{\lambda^2_1}), \text{ for all } j \in [D]$, if $\norm{\Tilde{x}(t)} > \frac{\lambda_1}{2}$, then $\norm{\Tilde{x}(t+1)} \le \frac{\lambda_1}{2}$. 
\end{lemma}

\begin{proof}[Proof of \Cref{cor:iteratenorm_drops}]
	The proof is similar to the proof of \Cref{lem:iteratenorm_drops}. Let the index $k$ be the smallest integer such that $\lambda_{k+1} < 2 \norm{\Tilde{x}(t)} - \lambda_1$.  If no such index exists, then one can observe that $\norm{ \Tilde{x}(t+1) } \le \lambda_1 - \norm{ \Tilde{x} (t) }$. Assuming that such an index exists in $[D]$, we have $\lambda_{k} \ge 2 \norm{\Tilde{x}(t)} - \lambda_1$ and $\norm{\Tilde x(t)}-\lambda_j\le \lambda_1- \norm{\Tilde x(t)}$, $\forall j\le k$. With the same decomposition and estimation, since $\Tilde x(t)\in \cap_{j=1}^D (1+\frac{\lambda_D^2}{\lambda_1^2})\mathcal{I}_j$, we have 
\begin{align*}
    &\norm{v^{(1)}(t)} = ( \lambda_1 - \norm{\Tilde{x}(t)} )\norm{\Tilde{x}(t)} \\&
    \norm{v^{(2)}(t)} \le (1+\frac{\lambda_D^2}{\lambda_1^2}) ( 2\norm{\Tilde{x}(t)} - \lambda_1 - \lambda_{k} ) \lambda_k \\&
    \norm{v^{(2+j)}(t)} \le (1+\frac{\lambda_D^2}{\lambda_1^2})( \lambda_{k - 1 + j} - \lambda_{k  + j} ) \lambda_{k+ j}, \text{ for all } j \ge 1.
\end{align*}

Thus we conclude 
\begin{align*}
    \norm{ \Tilde{x}(t+1) } &\le \frac{1}{\norm{\Tilde{x}(t)}} \left( \norm{ v^{(1)}(t) } + \norm{ v^{(2)} (t) } + \ldots + \norm{ v^{(D-k+1)}(t) } \right) \\
    \le &\frac{\lambda_1}{2} (1- \frac{\lambda_D^2}{\lambda^2_1}) (1+\frac{\lambda^2_1}{\lambda_D^2})     
    \le  \frac{\lambda_1}{2},
\end{align*}
which completes the proof.
\end{proof}

\begin{lemma}\label{lem:incr_in_other_coordinates}
    Consider the function $g: \mathbb{R} \to \mathbb{R}$, with $g(\lambda) = \frac{\lambda_1}{2} \left(1 - \sqrt{1 - 2 \frac{\lambda}{\lambda_1} \left(1 - \frac{\lambda}{\lambda_1}\right)}\right)$.  
    For any small constant $c > 0$ and coordinate $1 \le k \le D$, consider any  $t$ with $\Tilde x(t)\in \cap_{j=1}^D \mathcal{I}_j$. If $\Tilde{x}(t)$ satisfies that
    \begin{itemize}
        \item $\abs{ \langle e_1, \Tilde{x}(t) \rangle} \le (1 - 2c) g(\lambda_k)$. 
        \item $\theta_t \le  \sqrt{c \abs{  \langle e_1, \Tilde{x}(t) \rangle } }$,
    \end{itemize}
    where $\theta_t = \arctan \frac{ \norm{ \projd{2} ( \Tilde{x}(t) ) } }{ \abs{\langle e_1, \Tilde{x}(t)\rangle} }$. 
    
    Then, we have
    \begin{align*}
        \abs{ \frac{ \langle e_k,  \Tilde{x}(t+2) \rangle } { \langle e_1, \Tilde{x}(t+2)\rangle } } \ge (1 + c)  \abs{ \frac{ \langle e_k,  \Tilde{x}(t) \rangle } { \langle e_1, \Tilde{x}(t)\rangle } }.
    \end{align*}
\end{lemma}

\begin{proof}[Proof of \Cref{lem:incr_in_other_coordinates}]

    From the quadratic update, we have
     the update rule as:
    \begin{align*}
        \Tilde{x}_k(t+1) = \Tilde{x}_k(t) \left(1 - \frac{\lambda_k}{\norm{\Tilde{x}(t)}} \right), \text{ for all } k \in \{1,\ldots,  D\}.
    \end{align*}

    Thus, we have for any $1 \le k \le d$, 
    \begin{align*}
        \abs{ \frac{ \langle e_k, \Tilde{x}(t+2)\rangle }{ \langle e_1, \Tilde{x}(t+2)\rangle } }
        &= \abs{ \left(1 - \frac{\lambda_1 - \lambda_k}{\lambda_1 - \norm{ \Tilde{x}(t) }} \right) \left(1 - \frac{\lambda_1 - \lambda_k}{\lambda_1 - \norm{ \Tilde{x}(t+1) }} \right) \frac{ \langle e_k, \Tilde{x}(t)\rangle }{ \langle e_1, \Tilde{x}(t)\rangle } } \\&
        = \abs{ \left(1 - \frac{(\lambda_1 - \lambda_k) (  \lambda_1 + \lambda_k - \norm{\Tilde{x}(t)} - \norm{\Tilde{x}(t+1)} ) }{ (\lambda_1 - \norm{ \Tilde{x}(t+1) })( \lambda_1 - \norm{ \Tilde{x}(t) } ) }\right) \frac{ \langle e_k, \Tilde{x}(t)\rangle }{ \langle e_1, \Tilde{x}(t)\rangle } } .
    \end{align*}
    
    Thus, as long as, the following holds true:
    \begin{align*}
        \frac{(\lambda_1 - \lambda_k) (  \lambda_1 + \lambda_k - \norm{\Tilde{x}(t)} - \norm{\Tilde{x}(t+1)} ) }{ (\lambda_1 - \norm{ \Tilde{x}(t+1) })( \lambda_1 - \norm{ \Tilde{x}(t) } ) } \ge 2+c ,
    \end{align*}
    we must have
    \begin{align*}
      \abs{ \frac{ \langle e_k,  \Tilde{x}(t+2) \rangle } { \langle e_1, \Tilde{x}(t+2)\rangle } } \ge \left(1 + c\right)  \abs{ \frac{ \langle e_k,  \Tilde{x}(t) \rangle } { \langle e_1, \Tilde{x}(t)\rangle } }.
    \end{align*}

    We can use $(\lambda_1 - \norm{\Tilde{x}(t)}) \cos \theta_t \le  \norm{\Tilde{x}(t+1)}  \le \lambda_1 - \norm{\Tilde{x}(t)} - \frac{\lambda}{2\lambda_1} \left( 1- \frac{\lambda}{\lambda_1} \right)  \lambda_1 \sin^2 \theta_t$, where $\lambda = \min(\lambda_1 - \lambda_2, \lambda_D)$ from \Cref{lem:normbound_onincrease} to show the following with additional algebraic manipulation:
    \begin{align*}
         \frac{(\lambda_1 - \lambda_k) (  \lambda_1 + \lambda_k - \norm{\Tilde{x}(t)} - \norm{\Tilde{x}(t+1)} ) }{ (\lambda_1 - \norm{ \Tilde{x}(t+1) })( \lambda_1 - \norm{ \Tilde{x}(t) } ) } \ge \frac{(\lambda_1 - \lambda_k)  \lambda_k }{ (  \lambda_1 - (\lambda_1 - \norm{\Tilde{x}(t)}) \cos \theta_t  ) (\lambda_1 - \norm{ \Tilde{x}(t) }) } .
    \end{align*}
    
    Hence, it suffices to show that
    \begin{align*}
       \frac{(\lambda_1 - \lambda_k)   \lambda_k  }{ (  \lambda_1 - (\lambda_1 - \norm{\Tilde{x}(t)}) \cos \theta_t  ) (\lambda_1 - \norm{ \Tilde{x}(t) }) } \ge 2+c .
    \end{align*}
    The left hand side can be simplified as
    \begin{align*}
        \frac{(\lambda_1 - \lambda_k)   \lambda_k  }{ (  \lambda_1 - (\lambda_1 - \norm{\Tilde{x}(t)}) \cos \theta_t  ) (\lambda_1 - \norm{ \Tilde{x}(t) }) } &=   \frac{(\lambda_1 - \lambda_k)   \lambda_k  }{ (  2\lambda_1 \sin^2 (\theta_t/2) + \abs{\langle e_1, \Tilde{x}(t) \rangle}  ) (\lambda_1 - \norm{ \Tilde{x}(t) }) } \\
        &\ge \frac{(\lambda_1 - \lambda_k)   \lambda_k  }{   \lambda_1 \theta_t^2/2 + \abs{\langle e_1, \Tilde{x}(t) \rangle}  ) (\lambda_1 - \abs{\langle e_1, \Tilde{x}(t) \rangle}) }\\
		&\ge \frac{(\lambda_1 - \lambda_k)   \lambda_k  }{   \abs{\langle e_1, \Tilde{x}(t) \rangle}   (\lambda_1 + \frac{c}{2} \lambda_1 - \abs{\langle e_1, \Tilde{x}(t) \rangle}) },
    \end{align*}
    where the last step we use that $\abs{ \theta_t } \le  \sqrt{ c \abs{ \langle e_1, \Tilde{x}(t) \rangle } }$, 
    we only need 
    \begin{align*}
        (2 + c) \abs{\langle e_1, \Tilde{x}(t) \rangle}^2 - 2 \lambda_1 ( 1 + c/2) (2 + c) \abs{\langle e_1, \Tilde{x}(t) \rangle} +  (\lambda_1 - \lambda_k)   \lambda_k \ge 0.
    \end{align*}
    The above inequality is true when $\abs{ \langle e_1, \Tilde{x}(t) \rangle} \le \left(1 - 2c\right) g(\lambda_k)$.
\end{proof}

\begin{lemma}\label{lem:quadratic_tan_two_step} 
Consider the function $g: \mathbb{R} \to \mathbb{R}$, with $g(\lambda) = \frac{\lambda_1}{2} \left(1 - \sqrt{1 - 2 \frac{\lambda}{\lambda_1} \left(1 - \frac{\lambda}{\lambda_1}\right)}\right)$.  Consider any coordinate $2 \le k \le D$.
    For any constant $0 < c < 4 \frac{\lambda_k}{\lambda_1} (1 - \frac{\lambda_k}{\lambda_1})$, consider any  $t$ with $\Tilde x(t)\in \cap_{j=1}^D \mathcal{I}_j$, with $\Tilde{x}(t)$ satisfying
    \begin{align*}
       0.5 \lambda_1 \ge \norm{ \Tilde{x}(t) } \ge (1 + c) g(\lambda_k). 
    \end{align*}

    Then, the following must hold true at time $t$.
    \begin{align*}
        \abs{ \frac{ \langle e_k,  \Tilde{x}(t+2) \rangle } { \langle e_1, \Tilde{x}(t+2)\rangle } } \le \left (1 - 0.5 c \right)
        \abs{ \frac{ \langle e_k,  \Tilde{x}(t) \rangle } { \langle e_1, \Tilde{x}(t)\rangle } },
    \end{align*}
\end{lemma}

\begin{proof}
    By the Normalized GD update, we have:
    \begin{align}
        \abs{ \frac{ \langle e_k,  \Tilde{x}(t+2) \rangle } { \langle e_1, \Tilde{x}(t+2)\rangle } } 
        &= \abs{ \left( \frac{1 - \frac{\lambda_k}{\norm{\Tilde{x}(t+1)}} }{ 1 - \frac{\lambda_1}{\norm{\Tilde{x}(t+1)}}  } \right)
         \left( \frac{1 - \frac{\lambda_k}{\norm{\Tilde{x}(t)}} }{ 1 - \frac{\lambda_1}{\norm{\Tilde{x}(t)}}  } \right)  } \abs{ \frac{ \langle e_k,  \Tilde{x}(t) \rangle } { \langle e_1, \Tilde{x}(t)\rangle } } \nonumber\\
        &= \abs{ \left(1 - \frac{(\lambda_1 - \lambda_k) (  \lambda_1 + \lambda_k - \norm{\Tilde{x}(t)} - \norm{\Tilde{x}(t+1)} ) }{ (\lambda_1 - \norm{ \Tilde{x}(t+1) })( \lambda_1 - \norm{ \Tilde{x}(t) } ) }\right) \frac{ \langle e_k, \Tilde{x}(t)\rangle }{ \langle e_1, \Tilde{x}(t)\rangle } } \label{eq:angle_drop_good_region}.
    \end{align}

    Now, we focus on the term $\frac{(\lambda_1 - \lambda_k) (  \lambda_1 + \lambda_k - \norm{\Tilde{x}(t)} - \norm{\Tilde{x}(t+1)} ) }{ (\lambda_1 - \norm{ \Tilde{x}(t+1) })( \lambda_1 - \norm{ \Tilde{x}(t) } ) }$. For simplicity, we will denote the term as $\mathrm{ratio}( \lambda_1, \lambda_k, \norm{\Tilde{x}(t)}, \norm{\Tilde{x}(t+1)} )$. The term behaves differently, depending on whether $\norm{\Tilde{x}(t)} \ge \lambda_k$ or $\norm{\Tilde{x}(t)} \le \lambda_k$:
    \begin{enumerate}
        \item If $\norm{\Tilde{x}(t)} \ge \lambda_k$, which is only possible when $\lambda_k \le \frac{\lambda_1}{2}$, we find that $\mathrm{ratio}( \lambda_1, \lambda_k, \norm{\Tilde{x}(t)}, \norm{\Tilde{x}(t+1)} )$ is a monotonically decreasing function w.r.t. $\norm{\Tilde{x}(t+1)}$, keeping other terms fixed. Using the fact that  $\norm{\Tilde{x}(t+1)}  \le \lambda_1 - \norm{\Tilde{x}(t)}$ from \Cref{lem:normbound_onincrease}, we can bound the term as:
        \begin{align*}
            \min_{ \lambda_k  \le a \le 0.5 \lambda_1} \mathrm{ratio}( \lambda_1, \lambda_k, a, \lambda_1 - a ) &\le  \mathrm{ratio}( \lambda_1, \lambda_k, \norm{\Tilde{x}(t)}, \norm{\Tilde{x}(t+1)} ) \\&\le  \max_{ \lambda_k  \le a \le 0.5 \lambda_1} \mathrm{ratio}( \lambda_1, \lambda_k, a, 0 ).
        \end{align*}
        
        We can simplify $\mathrm{ratio}( \lambda_1, \lambda_k, a, 0 )$ as $\frac{(\lambda_1 + \lambda_k - a)(\lambda_1 - \lambda_k)}{\lambda_1 (\lambda_1 - a)}$ for any $a$, and can be shown to be at most $1 + \frac{\lambda_k}{\lambda_1}$ ($\le 3/2$) for any $a$ in the range $(\lambda_k, 0.5 \lambda_1)$. Furthermore, $\mathrm{ratio}( \lambda_1, \lambda_k, a, \lambda_1 - a )$ simplifies as $\frac{\lambda_k(\lambda_1 - \lambda_k)}{a (\lambda_1 - a)}$ for any $a$, and can be shown to be at least $4 \frac{\lambda_k}{\lambda_1} \left( 1 - \lambda_k / \lambda_1 \right)$ in the range $(\lambda_k, 0.5 \lambda_1)$, which it attains at $a = \lambda_k$.

        \item If $\norm{\Tilde{x}(t)} \le \lambda_k$, we find that $\mathrm{ratio}( \lambda_1, \lambda_k, \norm{\Tilde{x}(t)}, \norm{\Tilde{x}(t+1)} )$ is a monotonically increasing function w.r.t. $\norm{\Tilde{x}(t+1)}$, keeping other terms fixed. Using the fact that  $\norm{\Tilde{x}(t+1)}  \le \lambda_1 - \norm{\Tilde{x}(t)}$ from \Cref{lem:normbound_onincrease}, we can bound the term as:
        \begin{align*}
            \min_{ (1 + c) g(\lambda_k) \le a \le \min( 0.5 \lambda_1, \lambda_k ) } \mathrm{ratio}( \lambda_1, \lambda_k, a, 0 ) &\le  \mathrm{ratio}( \lambda_1, \lambda_k, \norm{\Tilde{x}(t)}, \norm{\Tilde{x}(t+1)} ) \\&\le  \max_{ (1 + c) g(\lambda_k) \le a \le \min( 0.5 \lambda_1, \lambda_k ) } \mathrm{ratio}( \lambda_1, \lambda_k, a, \lambda_1 - a ).
        \end{align*}
        
        Continuing in the similar way as the previous case, we show that $\mathrm{ratio}( \lambda_1, \lambda_k, a, 0 )$ is at least $1 - (\lambda_k / \lambda_1)^2$ in the range $((1 + c) g(\lambda_k), \min( 0.5 \lambda_1, \lambda_k ))$. $\mathrm{ratio}( \lambda_1, \lambda_k, a, \lambda_1 - a )$ is maximized in the range $((1 + c) g(\lambda_k),  \min( 0.5 \lambda_1, \lambda_k ))$ at $a = (1 + c) g(\lambda_k)$ and is given by $\frac{\lambda_k ( \lambda_1 - \lambda_k ) }{ (1 + c) g(\lambda_k) ( \lambda_1 - (1 + c) g(\lambda_k) ) }.$ From the definition of $\lambda_k$, we observe that $ \lambda_1 - (1 + c) g(\lambda_k) $ is atleast $(1 - \frac{c}{4}) ( \lambda_1 - (1 + c) g(\lambda_k) )$ for any $c \in (0, 1)$. Thus, we have 
        \begin{align*}
            \frac{\lambda_k ( \lambda_1 - \lambda_k ) }{ (1 + c) g(\lambda_k) ( \lambda_1 - (1 + c) g(\lambda_k) ) } &\le \frac{1}{(1+c)(1-\frac{c}{4})} \frac{\lambda_k ( \lambda_1 - \lambda_k ) }{  g(\lambda_k) ( \lambda_1 -  g(\lambda_k) ) } \\&
            = \frac{2}{(1+c)(1-\frac{c}{4})} \leq 2 - 0.5 c,
        \end{align*}
        where the final step holds true for any $c \in (0, 1).$
    \end{enumerate}
    Thus, we have shown that
    \begin{align*}
         2 \frac{\lambda_k}{\lambda_1} (1 - \frac{\lambda_k}{\lambda_1}) \le & \min \left( 4 \frac{\lambda_k}{\lambda_1} (1 - \frac{\lambda_k}{\lambda_1}), 1 - (\frac{\lambda_k}{\lambda_1})^2 \right)\\
          \le & \frac{(\lambda_1 - \lambda_k) (  \lambda_1 + \lambda_k - \norm{\Tilde{x}(t)} - \norm{\Tilde{x}(t+1)} ) }{ (\lambda_1 - \norm{ \Tilde{x}(t+1) })( \lambda_1 - \norm{ \Tilde{x}(t) } )  } \le 2 - 0.5 c.
    \end{align*}
    
    The result follows after substituting this bound in \Cref{eq:angle_drop_good_region}.
    
\end{proof}

\begin{lemma}\label{lem:behaviornormhalf_theta_decrease}
At any step $t$, if $\norm{\Tilde{x}(t)} \le \frac{\lambda_1}{2}$,  
\begin{enumerate}
\item $\abs{ \tan (\angle (\Tilde{x}(t+1), e_1)) } \le  \max( \frac{\lambda_2}{\lambda_1}, 1 - 2\frac{\lambda_D}{\lambda_1} ) \abs{ \tan  (\angle (\Tilde{x}(t), e_1)) }$.
\item $\abs{ \tan (\angle (\Tilde{x}(t+2), e_1)) } \le \frac{\lambda_1}{\norm{\tilde x(t)}} \abs{ \tan  (\angle (\Tilde{x}(t), e_1)) }$.
\end{enumerate}
\end{lemma}

\begin{proof}[Proof of \Cref{lem:behaviornormhalf_theta_decrease}]
From the Normalized GD update rule, we have
\begin{align*}
    \Tilde{x}_i(t+1) = \Tilde{x}_i(t) \left(1 - \frac{\lambda_i}{\norm{\Tilde{x}(t)}} \right), \text{ for all } i \in [ D],
\end{align*}
implying $\abs{ \Tilde{x}_i(t+1) } < \abs{ \left(1 - \frac{1}{\norm{\Tilde{x}(t)}} \right) } \abs{ \Tilde{x}_i(t) }$ for all $i \in [2, D]$, since $\lambda_i < 1$.
 
Since  $\lambda_i < \lambda_1$ and $ \norm{\Tilde{x}(t)} \le \frac{\lambda_1}{2}$, it holds that
\begin{align*}
    \frac{ \abs{ \Tilde{x}_i(t+1) } }{  \abs{ \Tilde{x}_1(t+1) }  } 
    = \abs{ \frac{1 - \frac{\lambda_i}{\norm{\Tilde{x}(t)}} }{ 1 - \frac{\lambda_1}{\norm{\Tilde{x}(t)}}  } } \frac{ \abs{ \Tilde{x}_i(t) } }{  \abs{ \Tilde{x}_1(t) }  } 
    = \abs{ 1 - \frac{ \lambda_1 - \lambda_i }{  \lambda_1 - \norm{\Tilde{x}(t)}  } } \frac{ \abs{ \Tilde{x}_i(t) } }{  \abs{ \Tilde{x}_1(t) }  } 
    \le \max( \frac{\lambda_i}{\lambda_1}, 1- 2 \frac{\lambda_i}{\lambda_1} ) \frac{ \abs{ \Tilde{x}_i(t) } }{  \abs{ \Tilde{x}_1(t) }  }.
\end{align*}

Finally we conclude 
\begin{align*}
    \frac{ \norm{ \projd{2} \Tilde{x}(t+1) } }{  \abs{ \Tilde{x}_1(t+1) }  } \le 
\max( \frac{\lambda_2}{\lambda_1}, 1- 2\frac{\lambda_D}{\lambda_1} ) \frac{ \norm{\projd{2} \Tilde{x}(t)} } {  \abs{ \Tilde{x}_1(t) } }.
\end{align*}
Recall $\abs{ \tan ( \angle (v, e_1) ) } = \frac{ \norm{ \projd{2} v } }{  \abs{ \inner{e_1}{v} }  } $ for any vector $v$, the first claim follows from re-arranging the terms.

For the second claim, it suffices to apply the above inequality to $t+1$, which yields that 
\begin{align*}
	\abs{ \tan (\angle (\Tilde{x}(t+2), e_1)) } \le \max(\frac{\lambda_2}{\lambda_1}, \frac{\lambda_1-\lambda_i}{\lambda_1 - \norm{x_{t+1}}} - 1) \abs{ \tan  (\angle (\Tilde{x}(t+1), e_1)) } \le  \frac{\lambda_1}{\lambda_1 - \norm{x_{t+1}}} \abs{ \tan  (\angle (\Tilde{x}(t+1), e_1)) }. 
\end{align*}
The proof is completed by noting $\norm{\tilde x(t+1)}\le \lambda_1 - \norm{\tilde x{t}}$ (\Cref{lem:normbound_onincrease}) and $\tan  (\angle (\Tilde{x}(t+1), e_1)) \le \tan  (\angle (\Tilde{x}(t), e_1))$.
\end{proof}

\section{Setups for General Loss Functions}\label{appsec:setup}

Before we start the analysis for Normalized GD for general loss functions in \Cref{appsec:general_loss_analysis}, we need to introduce some new notations and terminologies to  complete the formal setup. We  start by first recapping some core assumptions and definitions in the main paper and provide the missing proof in the main paper.

\assumptiononL*
\assumptiononEigengap*

\paragraph{Notations:}  We define $\Phi$ as the limit map of gradient flow below. We summarize various properties of $\Phi$ from \Cref{ch:diffusion_on_manifold} in \Cref{appsec:property_phi}. \begin{equation}\label{eq:Phi}
    \Phi(x) = \lim_{\conttime \to \infty} \phi(x, \conttime), \quad \textrm{where} \quad \phi(x, \conttime) = x - \int_{0}^{\conttime} \nabla L(\phi(x, s)) ds.
\end{equation}

Let $U$ be the sets of points starting from which, gradient flow w.r.t. loss $L$ converges to some point in $\Gamma$, that is, $U:=\{ x \in \RR^D \mid \Phi(x) \textrm{ exists and } \Phi(x) \in \Gamma \}$. We have that $U$ is open and $\Phi$ is $\mathcal{C}^3$ on $U$. (By \Cref{lem:U_open_Phi_smooth})

%

For a matrix $A \in \mathbb{R}^{D \times D}$, we denote its eigenvalue-eigenvector pairs by $\{\lambda_i(A), v_i(A) )\}_{i \in [D]}$. For simplicity, 
whenever $\Phi$ is defined and $\mathcal{C}^2$ at point $x$, we use $\{(\lambda_i(x), v_i(x))\}_{i=1}^{D}$ to denote the eigenvector-eigenvalue pairs of $\nabla^2 L ( \Phi(x) )$, with $\lambda_1(x) > \lambda_2(x) \ge \lambda_3(x) \ldots \ge \lambda_D(x)$. Given a differentiable submanifold $\Gamma$ of $\mathbb{R}^D$ and point $x\in\Gamma$, we  use $\normal_{x}{\Gamma}$ and $\tangent_{x}{\Gamma}$  to denote the normal space and the tangent space of the manifold $\Gamma$  for any point $x\in\Gamma$. 
We use $P_{x,\Gamma}: \Gamma \to \mathbb{R}^{D}$ to denote the projection operator onto the normal space of $\Gamma$ at $x$, and $P_{x,\Gamma}^\perp :=I_D - P_{ x,\Gamma }$. Similar to quadratic case, for any $x\in U$, we use $\Tilde x$ to denote $\nabla^2 L(\Phi (x  )) (x - \Phi(x))$ for Normalized GD on $L$ and to denote $\sqrt{2\nabla^2 L(\Phi (x  ))} (x - \Phi(x))$ for GD on $\sqrt{L}$. 
We also use $P^{(j:D)}_{x,\Gamma}$ to denote the projection matrix $\sum_{j=i}^M v_i(x)v_i(x)^\top $ for $x\in \Gamma$ and $j\in [M]$. Therefore, $P^{1:D}_{x,\Gamma} = P_{x,\Gamma}$ by \Cref{lem:topeig_in_normal}.
Additionally, for any $x\in U$, we use $\theta(x)$ to denote the angle between $\Tilde{x}$
and the top eigenspace of the hessian at $\Phi(x)$, i.e. $ \theta(x) = \arctan \frac{ \norm{ P_{\Phi(x), \Gamma}^{(2:M)}  \Tilde{x} } }{ \abs{ \langle v_1(x), \Tilde{x} \rangle } }$. Furthermore, when the iterates $x(t)$ is clear in the context, we  use shorthand $\lambda_i(t): =\lambda_i(x(t))$,  $v_i(t):= v_i(x(t))$, $P_{t,\Gamma} := P_{\Phi(x(t)),\Gamma}$, $P_{t,\Gamma}^\perp := P_{\Phi(x(t)),\Gamma}^{\perp} $ and $\theta_t$ to denote $\theta(x(t))$.  
We define the function $g_t: \mathbb{R} \to \mathbb{R}$ for every $t\in\NN$ as
$$g_t(\lambda) = \frac{1}{2} \left( 1 - \sqrt{ 1 - 2 \frac{\lambda}{\lambda_1(t)} \left(1 - \frac{\lambda}{\lambda_1(t)}\right) } \right).$$ 

Given any two points $x,y$, we use $\overline{xy}$ to denote the line segment between $x$ and $y$, \emph{i.e.}, $\{z\mid \exists \lambda\in[0,1], z=(1-\lambda)x+\lambda y\}$.  




The main result of this chapter focuses on the trajectory of  Normalized GD from fixed initialization $\xinit$ with LR $\eta$ converges to $0$,  which can be roughly split into two phases. 
In the first phase, \Cref{thm:ngd_phase_1} shows that the normalized GD trajectory converges to the gradient flow trajectory, $\phi(\xinit,\cdot)$.
 In second phase, \Cref{thm:ngd_phase_2} shows that the normalized GD trajectory converges to the limiting flow which decreases sharpness on $\Gamma$, \eqref{eq:log_limiting_flow}. 
Therefore, for sufficiently small $\eta$, the entire trajectory of normalized GD will be contained in a small neighbourhood of gradient flow trajectory $\GFtraj$ and limiting flow trajectory $\flowtraj$. The convergence rate given by our proof depends on the various local constants like smoothness of $L$ and $\Phi$ in this small neighbourhood, which intuitively can be  viewed as the actual "working zone" of the algorithm. The constants are upper bounded or lower bounded from zero because this "working zone" is compact after fixing the stopping time of \eqref{eq:log_limiting_flow}, which is denoted by $T_2$.  
\begin{align*}
    \loglimitingflow
\end{align*}

Below we give formal definitions of the "working zones" and the corresponding properties.
For any point $y\in\mathbb{R}^D$ and positive number $r$, we define $B_r(y):=\{x \in \mathbb{R}^D\mid \norm{y-x}< r \}$  as the open $\ell_2$ norm ball centered at $y$ and $\overline B_r(y)$ as its closure. 
For any set $S$ and positive number $r$, we define $S^r:=\cup_{y\in S} B_r(y)$ and $\overline B_r(S):=\cup_{y\in S} \overline B_r(y)$.
Given the stopping time $T_2>0$, we denote the trajectory of limiting flow~\Cref{eq:log_limiting_flow} $\{X(\tau)\}_{\tau=0}^{T_2}$ by $\flowtraj$ and we use the notation $\flowtraj^r:= \cup_{y\in\flowtraj} \overline B_y(r)$ for any $r>0$. By definition,  $\flowtraj^r$ are compact for any $r>0$.

We construct the "working zone" of the second phase, $\flowtraj^\presaferadius$ and $\flowtraj^\saferadius$ in \Cref{lem:exist_presaferadius,lem:exist_saferadius} respectively, where $0<\saferadius < \presaferadius$, implying $\flowtraj^\saferadius\subset \flowtraj^\presaferadius$.  The reason that we need the two-level nested "working zones" is that even though we can ensure all the points in $\flowtraj^\presaferadius$ have nice properties as listed in \Cref{lem:exist_presaferadius}, we cannot ensure the trajectory of gradient flow from $x\in \flowtraj^\presaferadius$ to $\Phi(x)$ or the line segment $\overline{x\Phi(x)}$ is in $\flowtraj^\presaferadius$, which will be crucial for the geometric lemmas (in \Cref{appsec:geometric_lemmas}) that we will heavily use in the trajectory analysis around the manifold. For this reason we further define $\flowtraj^\saferadius$ and \Cref{lem:exist_saferadius} guarantees the trajectory of gradient flow from $x$ to $\Phi(x)$ or the line segment $\overline{x\Phi(x)}$ whenever $x\in \flowtraj^\presaferadius$.

\begin{definition}[PL condition]\label{defi:localPL}
A  function $L$ is said to be $\PLconstant$-PL in a set $U$ iff for all $x\in U$, $$\norm{ \nabla L (x) }^2 \ge 2\PLconstant ( L (x) - \inf\limits_{x\in U} L(x)).$$
\end{definition}

For convenience, we define $\eigengap:= \frac{1}{2}\inf_{x\in\flowtraj}\big(\lambda_1(\nabla^2 L(x)) - \lambda_2(\nabla^2 L(x)))\big) $ and $\mineigen :=  \frac{1}{4}\inf_{x \in \flowtraj } \lambda_M( \nabla^2 L (x) )$. By \Cref{ass:manifold}, we have $\mineigen>0$.  By  \Cref{ass:eigen_gap}, $\eigengap>0$.

\begin{lemma}\label{lem:exist_presaferadius}
Given $\flowtraj$, there are sufficiently small  $\presaferadius>0$ such that
\begin{enumerate}[nosep]
    \item $\flowtraj^\presaferadius\cap \Gamma$ is compact;
    \item  $\flowtraj^\presaferadius\subset U$;
    \item $L$ is $\PLconstant$-PL on $\flowtraj^\presaferadius$; (see \Cref{defi:localPL})
    \item $\inf_{x\in\flowtraj^\presaferadius}\big(\lambda_1(\nabla^2 L(x)) - \lambda_2(\nabla^2 L(x)))\big) \ge\eigengap>0$;
    \item $\inf_{x \in \flowtraj^\presaferadius } \lambda_M( \nabla^2 L (x) )\ge\mineigen>0$.
\end{enumerate}
\end{lemma}
\begin{proof}[Proof of \Cref{lem:exist_presaferadius}]
We first claim for every $y\in \flowtraj$, for all sufficiently small $\presaferadius_y>0$ (\emph{i.e.} for all $\presaferadius_y$ smaller than some threshold depending on $y$), the following three properties hold (1) $\overline B_y({\presaferadius_y})\cap \Gamma$ is compact; (2) $\overline B_y({\presaferadius_y})\cap \Gamma\subset U$ and (3) $L$ is $\PLconstant$-PL on $\overline B_y({\presaferadius_y}\cap \Gamma)$. 

Among the above three claims, (2) is immediate. (1) holds because $\overline B_y(\presaferadius_y)\cap \Gamma$ is bounded and we can make $\presaferadius_y$ small enough to ensure $\overline B_y(\presaferadius_y)\cap \Gamma$ is closed. For (3), by Proposition 7 of \citep{fehrman2020convergence}, we define $p(y):=\argmin_{x\in \Gamma} \norm{x-y}$ which is uniquely defined and $\mathcal{C}^1$ in $\overline B_y(\presaferadius_y)$ for sufficiently small $\presaferadius_y$. Moreover, Lemma 14 in \citep{fehrman2020convergence} shows that $\norm{\nabla L(x) - \nabla^2 L(p(x)) (x-p(x)) } \le c\norm{x-p(x)}_2^2$ for all $x$  in $B_y(\presaferadius_y)$ uniformly and  some constant $c$. Thus for small enough $\presaferadius_y$,
\begin{align}\label{eq:PL_square_grad_to_square_taylor}
    \norm{\nabla L(x)}^2 
    \ge &{(x-p(x))^\top(\nabla^2 L(p(x)))^2 (x-p(x))} - O(\norm{x-p(x)}^3)
\end{align}
Furthermore, by Lemma 10 in \citep{fehrman2020convergence}, it holds that $ x-p(x) \in \normal_{p(x)}\Gamma = \text{span}(\{v_i(p(x))\}_{i=1}^M)$, which implies 
\begin{align}\label{eq:PL_square_taylor_to_eigen_taylor}
	(x-p(x))^\top (\nabla^2L(p(x)))^2(x-p(x))  \ge  \lambda_M(\nabla^2 L(p(x)))(x-p(x))^\top \nabla^2L(p(x))(x-p(x)),
\end{align}
and that
\begin{align*}
(x-p(x))^\top \nabla^2L(p(x))(x-p(x)) \ge \lambda_M(\nabla^2 L(p(x))) \norm{x-p(x)}_2^2.	
\end{align*}

Thus for any $c'>0$, for sufficiently small $\presaferadius_y$, $(x-p(x))^\top \nabla^2L(p(x))(x-p(x)) \ge c' \norm{x-p(x)}^3$. Combining \Cref{eq:PL_square_grad_to_square_taylor,eq:PL_square_taylor_to_eigen_taylor}, we conclude that  for sufficiently small $\presaferadius_y$,
\[\norm{\nabla L(x)}^2 \ge \lambda_M(\nabla^2 L(p(x)))(x-p(x))^\top \nabla^2L(p(x))(x-p(x)) - O(\norm{x-p(x)}^3_2) \]

 Again for sufficiently small $\presaferadius_y$, by Taylor expansion of $L$ at $p(x)$, we have 
\begin{align*}
\frac{1}{2}(x-p(x))^\top\nabla^2 L(p(x)) (x-p(x))  \ge L(x)- 	O( \norm{x-p(x)}^3).
\end{align*}
Thus we conclude
 $$ \norm{\nabla L(x)}^2 \ge 2\lambda_M(\nabla^2 L(p(x))) L(x) - O( \norm{x-p(x)}^3)
 \ge \lambda_M(\nabla^2 L(p(x))) L(x)
 \ge 2\PLconstant L(x).$$
 
  Meanwhile, since $\lambda_M(\nabla^2 L(p(x)))$ and $\lambda_1(\nabla^2 L(p(x))) - \lambda_2(\nabla^2 L(p(x)))$ are  continuous functions in $x$, we can also choose a sufficiently small $\presaferadius_y$ such that for all $x\in \overline B_y(\presaferadius_y)$, $\lambda_M(\nabla^2 L(p(x)))\ge \frac{1}{2}\lambda_M(\nabla^2 L(p(y))) = \frac{1}{2}\lambda_M(\nabla^2 L(y))>\eigengap$  and $\lambda_1(\nabla^2 L(p(x))) - \lambda_2(\nabla^2 L(p(x))) \ge \frac{1}{2} \big(\lambda_1(\nabla^2 L(p(y))) - \lambda_2(\nabla^2 L(p(y))) \big) = \frac{1}{2} \big( \lambda_1(\nabla^2 L(y)) - \lambda_2(\nabla^2 L(y)) \big)\ge \mineigen$. 
Further note $\flowtraj\subset \cup_{y\in\flowtraj} B_y(\presaferadius_y)$ and $\flowtraj$ is a compact set, we can take a finite subset of $\flowtraj$, $\flowtraj'$, such that $\flowtraj\subset \cup_{y\in\flowtraj'} B_y(\presaferadius_y)$. Taking $\presaferadius:=\min_{y\in\flowtraj'} \frac{\presaferadius_y}{2}$ completes the proof.
\end{proof}


%
%
%

\begin{definition}\label{defi:tensor}
    The spectral $2$-norm of a $k$-order tensor $\mathcal{T} = (t_{i_1 i_2 \cdots i_k}) \in \mathbb{R}^{d_1 \times d_2 \times \cdots d_k}$ is defined as the  maximum of the following constrained multilinear optimization problem:
    $$
    \|\mathcal{T}\|=\max \left\{\mathcal{T}\left(x^{(1)}, \cdots, x^{(k)}\right):\left\|x^{(i)}\right\|_{2}=1, x^{(i)} \in \mathbb{R}^{d_{i}}, i=1,2, \ldots, k\right\}.
    $$
    Here, $\mathcal{T}\left(x^1, \cdots, x^{(k)}\right)=\sum_{i_{1}=1}^{d_{1}} \sum_{i_{2}=1}^{d_{2}} \ldots \sum_{i_{k}=1}^{d_{k}} t_{i_{1} i_{2} \ldots i_{d}} x_{i_{1}}^{(1)} x_{i_{2}}^{(2)} \ldots x_{i_{k}}^{(k)}$.
\end{definition}

\begin{definition}\label{defi:constants}
    We define the following constants regarding smoothness of $L$ and $\Phi$ of various orders over  $\flowtraj^\presaferadius$. 
\begin{align*}
    \secondL &= \sup_{x \in \flowtraj^\presaferadius} \norm{ \nabla^2 L(x) }, \quad
    \thirdL = \sup_{x \in \flowtraj^\presaferadius }  \norm{ \nabla^3 L(x)},\quad
    \fourthL = \sup_{x \in \flowtraj^\presaferadius } \norm{\nabla^4 L(x)}, \\
    \secondPhi &= \sup_{x \in \flowtraj^\presaferadius } \norm{\nabla^2 \Phi(x)}, \quad
    \thirdPhi = \sup_{x,y \in \flowtraj^\presaferadius } \frac{\norm{\nabla^2 \Phi(x)-\nabla^2 \Phi(y)}}{\norm{x-y}}, \quad
\end{align*}
\end{definition}

\begin{lemma}\label{lem:exist_saferadius}
    Given $\presaferadius$ as defined in \Cref{lem:exist_presaferadius}, there is an $\saferadius \in (0, \presaferadius)$ such that 
    \begin{enumerate}[nosep]
        \item $\sup_{x\in \flowtraj^\saferadius}  L (x) - \inf\limits_{x\in \flowtraj^\saferadius} L(x)< \min(\frac{\PLconstant\presaferadius^2}{8}, \frac{2\PLconstant^5}{\thirdL^2\secondL^2})$;
        \item $\forall x \in \flowtraj^\saferadius$, $\Phi(x)\in \flowtraj^{\frac{\presaferadius}{2}}$.
    \end{enumerate}
\end{lemma}
\begin{proof}[Proof of \Cref{lem:exist_saferadius}]
For every $y\in \flowtraj$, there is an $\saferadius_y$, such that $\forall x\in B_y(\saferadius_y)$, it holds that $ L(x)< \min(\frac{\PLconstant\presaferadius^2}{8}, \frac{2\PLconstant^4}{\thirdL^2\secondL^2})$ and $\Phi(x)\in \flowtraj^{\frac{\presaferadius}{2}}$, as both $L(x)$ and $\Phi(x)$  are continuous. Further note $\flowtraj\subset \cup_{y\in\flowtraj} B_y(\saferadius_y)$ and $\flowtraj$ is a compact set, we can take a finite subset of $\flowtraj$, $\flowtraj'$, such that $\flowtraj\subset \cup_{y\in\flowtraj'} B_y(\saferadius_y)$. Taking $\saferadius:=\min_{y\in\flowtraj'} \frac{\saferadius_y}{2}$ completes the proof. 
\end{proof}



\paragraph{Summary for Setups:} The initial point $\xinit$ is chosen from an open neighborhood of manifold $\Gamma$, $U$, where the infinite-time limit of gradient flow $\Phi$ is well-defined and for any $x\in U$, $\Phi(x)\in \Gamma$.  We consider normalized GD with  sufficiently small LR $\eta$ such that the  trajectory enters a small neighborhood of limiting flow trajectory, $\flowtraj^\presaferadius$. Moreover, $L$ is $\PLconstant$-PL on $\flowtraj^\presaferadius$ and the eigengaps and smallest eigenvalues are uniformly lower bounded by positive $\eigengap,\mineigen$ respectively on $\flowtraj^\presaferadius$. Finally, we  consider a proper subset of $\flowtraj^\presaferadius$, $\flowtraj^\saferadius$, as the final "working zone" in the second phase (defined in \Cref{lem:exist_saferadius}), which enjoys more properties than $\flowtraj^\presaferadius$, including \Cref{lem:GF_no_escape,lem:orthogonal_is_small,lem:appr_gradient_norm,lem:approx_travel1}.

\subsection{Geometric Lemmas}\label{appsec:geometric_lemmas}

In this subsection we present several geometric lemmas which are frequently used in the trajectory analysis of normalized GD. In this section, $O(\cdot)$ only hides absolute constants. Below is a brief summary:
\begin{itemize}
	\item \Cref{lem:bound_GF_length_by_loss}: Inequalities connecting various terms: the distance between $x$ and $\Phi(x)$, the length of GF trajectory  from $x$ to $\Phi(x)$, square root of loss and gradient norm;
	\item \Cref{lem:GF_no_escape}: For any $x\in\flowtraj^\saferadius$, the gradient flow trajectory from $x$ to $\Phi(x)$ and the line segment between $x$ and $\Phi(x)$ are all contained in $\flowtraj^\presaferadius$, so it's "safe" to use Taylor expansions along GF trajectory or $\overline{x\Phi(x)}$ to derive properties;
	\item \Cref{lem:orthogonal_is_small,lem:appr_gradient_norm,lem:approx_travel1}:  for any $x\in \flowtraj^\saferadius$, the normalized GD dynamics at $x$ can be roughly viewed as approximately quadratic around $\Phi(x)$ with positive definite matrix $\nabla^2 L(\Phi(x))$.
	\item \Cref{lem:boundmovementinPhi}: In the  "working zone", $\flowtraj^\presaferadius $, one-step normalized GD update with LR $\eta$ only changes $\Phi(x_t)$ by $O(\eta^2)$.
	\item \Cref{lem:ngd_phase_1_firsthalf}: In the  "working zone", $\flowtraj^\presaferadius $, one-step normalized GD update with LR $\eta $ decreases $\sqrt{L(x) - \min_{y\in \flowtraj} L(y)}$ by $\eta\frac{\sqrt{2\PLconstant}}{4}$ if $\norm{\nabla L(x)}\ge \frac\secondL\eta$.
\end{itemize}

\begin{lemma}\label{lem:bound_GF_length_by_loss}
    If the trajectory of gradient flow starting from $x$, $\phi(x,t)$, stays in $\flowtraj^\presaferadius$ for all $t\ge 0$, then we have 
    \begin{align*}
    	\norm{x-\Phi(x)} 
        \le \int_{t=0}^\infty \norm{\frac{\diff \phi(x,t)}{\diff t}}\diff t 
        \le \sqrt{\frac{2 (L(x)-L(\Phi(x)))}{\PLconstant}} \le \frac{\norm{\nabla L(x)}}{\PLconstant }.
    \end{align*}
\end{lemma}

\begin{proof}[Proof of \Cref{lem:bound_GF_length_by_loss}]
Since $\Phi(x)$ is defined as $\lim_{t\to\infty }\phi(x,t)$ and $\phi(x,0)=x$, the left-side inequality follows immediately from triangle inequality. The right-side inequality is by the definition of PL condition. Below we prove the middle inequality.

Since $\forall t\ge 0$, $\phi(x,t)\in \flowtraj^\presaferadius$, it holds that $\norm{\nabla L(\phi(x,t))}^2\ge 2\PLconstant (L(\phi(x,t)) - L(\Phi(x)))$ by the choice of $\presaferadius$ in \Cref{lem:exist_presaferadius}. 
 Without loss of generality, we assume $L(y)=0,\forall y\in \Gamma$. Thus we have
\begin{align*}
\int_{t=0}^\infty \norm{\nabla L(\phi(x,t))}\diff t 
        \le \int_{t=0}^\infty \frac{\norm{\nabla L(\phi(x,t))}^2}{\sqrt{2\PLconstant L(\phi(x,t))}} \diff t. 
    \end{align*}
    Since $\diff \phi(x,t) = -\nabla L(\phi(x,t))\diff t$, if holds that 
    \begin{align*}
        \int_{t=0}^\infty \frac{\norm{\nabla L(\phi(x,t))}^2}{\sqrt{2\PLconstant L(\phi(x,t))}} \diff t
        \le \int_{t=0}^\infty \frac{-\diff L(\phi(x,t))}{\sqrt{2\PLconstant L(\phi(x,t))}} 
        = \int_{t=0}^\infty \sqrt{\frac{2}{\PLconstant } }\diff \sqrt{L(\phi(x,t))}
        = \sqrt{\frac{2L(\phi(x,0))}{\PLconstant}}.
    \end{align*}
The proof is complete since $\phi(x,0)=x$ and we assume $L(\Phi(x))$ is $0$.
\end{proof}

\begin{lemma}\label{lem:GF_no_escape}
    Let $\presaferadius, \saferadius$ be defined in \Cref{lem:exist_presaferadius,lem:exist_saferadius}. For any $x\in \flowtraj^\saferadius$, we have 
    \begin{enumerate}
		\item The entire trajectory of gradient flow starting from $x$ is contained in $\flowtraj^\presaferadius$, \emph{i.e.}, $\phi(x,t)\in \flowtraj^\presaferadius$, $\forall t\ge 0$;  
		\item Moreover, $\norm{\Phi(x)-\phi(x,t)}\le \min( \presaferadius, \frac{2\PLconstant^2}{\thirdL\secondL})$, $\forall t\ge0$.      	
    \end{enumerate}
\end{lemma}

\begin{proof}[Proof of \Cref{lem:GF_no_escape}]
Let time $\tau^*\ge0$ be the smallest time after which the trajectory of GF is completely contained in $\flowtraj^\presaferadius$, that is, $\tau^* := \inf\{t\ge 0\mid \forall t'\ge t, \phi(x,t')\in \flowtraj^\presaferadius\}$. Since $\flowtraj^\presaferadius$ is closed and $\phi(x,\cdot)$ is continuous, we have $\phi(x,\tau^*)\in \flowtraj^\presaferadius$.

Since $\forall \tau\ge \tau^*$, $\phi(x,\tau)\in \flowtraj^\presaferadius$, by \Cref{lem:bound_GF_length_by_loss}, it holds that $\norm{\phi(x,\tau^*)-\Phi(x)} \le \sqrt{\frac{2 (L(\phi(x,\tau^*))-L(\Phi(x)))}{\PLconstant}}$.

    Note that loss doesn't increase along GF, we have $L(\phi(x,\tau^*)) - L(\Phi(x)) \le L(x) - L(\Phi(x))\le \frac{\PLconstant\presaferadius^2}{8}$, which implies that $\norm{\phi(x,\tau^*)-\Phi(x)} \le \frac{\presaferadius}{2}$. Therefore $\tau^*$ must be $0$, otherwise there exists a $0<\tau'<\tau^*$ such that $\norm{\phi(x,\tau)-\Phi(x)} \le \presaferadius$ for all $\tau'<\tau<\tau^*$ by the continuity of $\phi(x,\cdot)$. This proves the first claim.
    
    Given the first claim is proved, the second claim follows directly from \Cref{lem:bound_GF_length_by_loss}.
    
\end{proof}

The following theorem shows that the projection of $x$ in the tangent space of $\Phi(x)$ is small when $x$ is close to the manifold. In particular if we can show that in a discrete trajectory with a vanishing learning rate $\eta$, the iterates $\{ x_{\eta} (t) \}$ stay in $\flowtraj^{\saferadius}$, we can interchangeably use $\norm{x_{\eta} (t) - \Phi(x_{\eta} (t))}$ with $\norm{P_{t, \Gamma} (x_{\eta} (t) - \Phi(x_{\eta} (t)))}$, with an additional error of $\mathcal{O}(\eta^3)$, when $\norm{P_{t, \Gamma} (x_{\eta} (t) - \Phi(x_{\eta} (t)))} \le \mathcal{O}(\eta).$

\begin{lemma}\label{lem:orthogonal_is_small}

    For all  $x\in\flowtraj^{\saferadius}$, we have that 
    \begin{align*}
        \norm{P^{\perp}_{\Phi(x), \Gamma} ( x - \Phi(x ))} \le  \frac{\thirdL \secondL}{ 4\PLconstant^2 }  \norm{ x - \Phi(x)  }^2 ,
    \end{align*}
    and that
    \begin{align*}
        \norm{P_{\Phi(x), \Gamma} ( x - \Phi(x ))}^2 \ge  \norm{ x - \Phi( x ) }^2 \left( 1 -   \frac{\thirdL \secondL}{ 4\PLconstant^2 }  \norm{ x - \Phi(x)  }  \right) \ge \frac{1}{2}\norm{ x - \Phi( x ) }^2 . 
    \end{align*}
\end{lemma}

\begin{proof}[Proof of \Cref{lem:orthogonal_is_small}]

First of all, we can track the decrease in loss along the Gradient flow trajectory starting from $x$. At any time $\tau$, we have
\begin{align*}
    \frac{d}{d \tau} L( \phi( x, \tau ) ) = \langle \nabla L( \phi( x, \tau ) ),  \frac{d}{d \tau} \phi( x, \tau ) \rangle = - \norm{ \nabla L( \phi( x, \tau ) ) }^2,
\end{align*}
where $\phi(x, 0) = x$. Without loss of generality, we assume $L(y)=0,\forall y\in \Gamma$. Using the fact that $L$ is $\PLconstant$-PL on $\flowtraj^{\presaferadius}$ and the GF trajectory starting from any point in $\flowtraj^{\saferadius}$ stays inside $\flowtraj^{\presaferadius}$ (from \Cref{lem:GF_no_escape}), we have
\begin{align*}
    \frac{d}{d \tau} L( \phi( x, \tau ) ) \le - 2 \PLconstant L( \phi( x, \tau ) ),
\end{align*}
which implies
\begin{align*}
    L( \phi( x, \tau ) ) \le L( \phi ( x, 0 )  ) e^{-2 \PLconstant \tau} 
\end{align*}

By \Cref{lem:bound_GF_length_by_loss}, we have 
\begin{align} \label{eq:distanceubound}
    \norm{\phi(x, \tau) - \Phi(x)} \le \sqrt{\frac{2}{\PLconstant}} \sqrt{L}(\phi(x, \tau)) \le \sqrt{\frac{2 L( \phi ( x, 0 )  ) e^{-2 \PLconstant \tau} }{\PLconstant}}. 
\end{align}

Moreover, we can relate $L( \phi ( x, 0 )$ with $\norm{\Phi(x) - x}$ with a second order taylor expansion:
\begin{align*}
    L(x) =& L(\Phi(x)) + \langle \nabla L(\Phi(x)), x - \Phi(x)\rangle \\
    + &  \int_{s=0}^{1} (1-s) ( x - \Phi(x))^\top \nabla^2 L( sx + (1-s) \Phi(x) )(x - \Phi(x))  ds 
\end{align*}
where in the final step, we have used the fact that $L(\Phi(x)) = 0$ and $\nabla L(\Phi(x)) = 0$.  By \Cref{lem:GF_no_escape}, we have $\overline{x\Phi(x)}\subset \flowtraj^\presaferadius$. Thus  $\max_{s \in [0, 1]} \norm{ \nabla^2 L( sx + (1-s) \Phi(x) ) } \le \secondL$ from \Cref{defi:constants} and it follows that

\begin{align}\label{eq:lossubound}
    &L(x) \le \int_{s=0}^{1} (1-s)\secondL \norm{ x - \Phi(x)}^2  ds
    = \frac{\secondL}{2} \norm{\Phi(x) - x}^2, 
\end{align}

Finally we focus on the movement in the tangent space. It holds that 
\begin{align}\label{eq:tangent_gf_decompose}
    \norm{P^{\perp}_{\Phi(x), \Gamma} (\phi(x, \infty) - \phi(x, 0)) } \le  \norm{ P^{\perp}_{\Phi(x), \Gamma} \int_{0}^{\infty}\nabla L(\phi(x, \tau))  } d\tau 
    \le \int_{0}^{\infty} \norm{ P^{\perp}_{\Phi(x), \Gamma} \nabla L(\phi(x, \tau))  } d\tau .
\end{align}

By \Cref{lem:GF_no_escape}, we have $\overline{\phi(x,\tau)\Phi(x)}\subset \flowtraj^\presaferadius$ for all $\tau \ge 0$ and thus $$\norm{\nabla L(\phi(x, \tau)) - \nabla^2 L(\Phi(x)) \big( \Phi(x)\big)\big(\phi(x, \tau) - \Phi(x)\big)} \le  \frac{\thirdL}{2}\norm{\phi(x, \tau) - \Phi(x)}^2.$$ Since $P^{\perp}_{\Phi(x), \Gamma} $ is the projection matrix for the tangent space, $P^{\perp}_{\Phi(x), \Gamma}  \nabla^2 L(\Phi(x)) =0 $ and thus by \Cref{eq:distanceubound}
\begin{align}\label{eq:tangent_gf_estimate}
	\norm{ P^{\perp}_{\Phi(x), \Gamma} \nabla L(\phi(x, \tau)) }
	 \le  \frac{\thirdL}{2}\norm{\phi(x, \tau) - \Phi(x)}^2 
	 \le \frac{\thirdL L( \phi ( x, 0 )  ) e^{-2 \PLconstant \tau} }{\PLconstant}\end{align}

Plug \Cref{eq:tangent_gf_estimate} into \Cref{eq:tangent_gf_decompose}, we conclude that 
\begin{equation}\label{eq:tangent_proj_small}
 \norm{P^{\perp}_{\Phi(x), \Gamma} (\phi(x, \infty) - x) }\le \int_{\tau=0}^\infty  	\frac{\thirdL L( \phi ( x, 0 )  ) e^{-2 \PLconstant \tau} }{\PLconstant} = \frac{ \thirdL L(x)}{2\PLconstant^2} \le \frac{\thirdL\secondL \norm{x-\Phi(x)}^2}{4\PLconstant^2}
\end{equation}

For the second claim, simply note that 
\begin{align*}
	&\norm{P_{\Phi(x), \Gamma}^{\perp} ( x - \Phi(x ))} \\
	= & \sqrt{ \norm{x-\Phi(x)}^2 - \norm{P_{\Phi(x), \Gamma} ( x - \Phi(x ))}^2 }\\
	\ge &\norm{x-\Phi(x)} - \frac{\norm{P_{\Phi(x), \Gamma} ( x - \Phi(x ))}^2 }{\norm{x-\Phi(x)}}. 
\end{align*}

The left-side inequality of the second inequality is proved by plugging the first claim into the above inequality~\Cref{eq:tangent_proj_small} and rearranging the terms. Note by the second claim in \Cref{lem:GF_no_escape}, $ \frac{\thirdL \secondL}{ 4\PLconstant^2 }  \norm{ x - \Phi(x)  }  \le \frac{1}{2}$, the right-side inequality is also proved.
\end{proof}

\begin{lemma}\label{lem:appr_gradient_norm}
At any point $x \in \flowtraj^{\saferadius}$, we have
\begin{align*}
 \norm{ \nabla L(x) - \nabla^2 L(\Phi(x)) (x - \Phi(x)) }\le \frac{1}{2}\thirdL \norm{ x - \Phi(x) }^2.
\end{align*}
and 
\begin{align*}
       \abs{ \frac{\norm{ \nabla L(x) }}{\norm{ \nabla^2 L(\Phi(x)) (x - \Phi(x)) }}-1 } \le \frac{ \thirdL}{\mineigen} \norm{x - \Phi(x)},
\end{align*}

Moreover, the normalized gradient of $L$ can be written as 
\begin{align}
      \frac{\nabla L(x)}{ \norm{ \nabla L(x) }} = \frac{\nabla^2 L (\Phi(x)) ( x- \Phi ( x ) ) }{ \norm{ \nabla^2 L (\Phi(x  )) ( x - \Phi ( x ) )   } } + \mathcal{O}( \frac{\thirdL}{\mineigen} \norm{x - \Phi(x  )}). \label{eq:noisynormaliedgdupdate_primal}
\end{align}


\end{lemma}

\begin{proof}[Proof of \Cref{lem:appr_gradient_norm}]
    Using taylor expansion at $x$, we have using  $\nabla L(\Phi(x)) = 0$:
    \begin{align*}
        & \norm{ \nabla L(x) - \nabla^2 L(\Phi(x)) (x - \Phi(x)) }\\
         =& \norm{ \int_{0}^{1} (1 - s) \partial^2 (\nabla L) (s x + (1-s) \Phi(x) ) [x - \Phi(x), x - \Phi(x) ] ds } \\
        \le &  \norm{ \int_{0}^{1} (1 - s) ds } \max_{0 \le s \le 1} \norm{ \partial^2 (\nabla L) (s x + (1-s) \Phi(x) ) } \norm{ x - \Phi(x) }^2 \\
        \le & \frac{1}{2}\thirdL \norm{ x - \Phi(x) }^2.
    \end{align*}
 Further note that  
 \begin{align*}
&\norm{\nabla^2 L(\Phi(x)) (x - \Phi(x))} 
\ge  \norm{ P_{\Phi(x), \Gamma} \nabla^2 L(\Phi(x)) (x - \Phi(x))}
  =  \norm{  \nabla^2 L(\Phi(x)) P_{\Phi(x), \Gamma} (x - \Phi(x))} \\
  \ge & \mineigen \norm{ P_{\Phi(x), \Gamma} (x - \Phi(x)) }, 
 \end{align*}
 we have 
    \begin{align*}
       \abs{ \frac{\norm{ \nabla L(x) }}{\norm{ \nabla^2 L(\Phi(x)) (x - \Phi(x)) }}-1 } \le    \frac{ \thirdL \norm{ x - \Phi(x) }^2 } { 2\mineigen \norm{ P_{\Phi(x), \Gamma} (x - \Phi(x)) } } \le \frac{\thirdL}{\mineigen} \norm{x - \Phi(x)},
    \end{align*}
    where we use \Cref{lem:orthogonal_is_small} since  $x \in \flowtraj^\saferadius$.
   Thus, the normalized gradient at any step $t$ can be written as
     \begingroup
     \allowdisplaybreaks
        \begin{align*}
            \frac{\nabla L(x) }{ \norm{ \nabla L(x) } }  
            &=   \frac{\nabla^2 L (\Phi(x)) [ x - \Phi ( x ) ] + \mathcal{O}(\thirdL \norm{x - \Phi(x)}^2) }{ \norm{ \nabla^2 L(\Phi(x)) (x - \Phi(x)) } \left(1 + \mathcal{O} \left( \frac{ \thirdL }{ \mineigen } \norm{ x - \Phi(x) }  \right) \right) }. \nonumber\\&
            =    \frac{\nabla^2 L (\Phi(x)) [ x - \Phi ( x ) ] }{ \norm{ \nabla^2 L (\Phi(x)) [ x - \Phi ( x ) ]   } } + \mathcal{O}( \frac{\thirdL}{\mineigen}  \norm{x - \Phi(x)}), 
        \end{align*}
    \endgroup    
    which completes the proof.
\end{proof}

\begin{lemma}\label{lem:approx_travel1}
    Consider any point $x \in \flowtraj^{\saferadius}$.  Then, 
    \begin{align*}
        \abs{\left\langle v_1( x ), \frac{ \nabla L(x) }{ \norm{ \nabla L( x ) }  } \right\rangle } \ge \cos \theta - \mathcal{O}( \frac{ \thirdL }{ \mineigen } \norm{x - \Phi(x)} ),
    \end{align*}
    where  $\theta = \arctan \frac{ \norm{ P_{\Phi(x), \Gamma}^{(2:M)} \Tilde{x} } }{ \abs{ \langle v_1( x ), \Tilde{x} \rangle } }$, with $\Tilde{x} = \nabla^2 L(\Phi(x)) (x - \Phi(x))$. 
\end{lemma}

\begin{proof}[Proof of \Cref{lem:approx_travel1}]
    
    From \Cref{lem:appr_gradient_norm}, we have that
    \begin{align*}
        \frac{\nabla L(x)}{ \norm{ \nabla L(x) }} = \frac{\nabla^2 L (\Phi(x)) ( x- \Phi ( x ) ) }{ \norm{ \nabla^2 L (\Phi(x  )) ( x - \Phi ( x ) )   } } + \mathcal{O}( \frac{\thirdL}{\mineigen} \norm{x - \Phi(x  )}).
    \end{align*}
    
    Hence, we have that 
    \begin{align*}
        \frac{ \abs{ \langle v_1(x), \nabla L(x) \rangle } }{ \norm{ \nabla L(x) }} &= \frac{ \abs{ \langle v_1(x), \nabla^2 L (\Phi(x)) ( x- \Phi ( x ) )  \rangle } }{ \norm{ \nabla^2 L (\Phi(x  )) ( x - \Phi ( x ) )   } } + \mathcal{O}( \frac{\thirdL}{\mineigen} \norm{x - \Phi(x  )}) \\&
        \ge \cos \theta - \mathcal{O}( \frac{\thirdL}{\mineigen} \norm{x - \Phi(x  )}),
    \end{align*}
    which completes the proof.
\end{proof}

\begin{lemma}\label{lem:boundmovementinPhi}
For any $\overline{xy} \in \flowtraj^\saferadius$ where $y= x - \eta \frac{\nabla L(x)}{\norm{\nabla L(x)} }$ is the one step Normalized GD update from $x$, we have
$$\norm{\Phi(y) - \Phi(x)} \le  \frac{1}{2} \secondPhi\eta^2.$$
Moreover, we must have for every $1 \le k \le M$,
$$ \abs{ \lambda_k( \nabla^2 L( \Phi (x) ) ) - \lambda_k( \nabla^2 L( \Phi (y) ) ) } \le \frac{1}{4} \thirdL \secondPhi \eta^2 ,$$
and 
$$ \norm{ v_1(\nabla^2 L (\Phi(x)) ) - v_1(\nabla^2 L (\Phi(y) )) } \le \frac{1}{2} \frac{\thirdL \secondPhi \eta^2}{ \eigengap - \frac{1}{4} \thirdL \secondPhi \eta^2 } = \frac{\thirdL \secondPhi \eta^2}{2\eigengap} + \mathcal{O}(\frac{\thirdL^2 \secondPhi^2 \eta^4}{\eigengap}).$$
\end{lemma}

\begin{proof}
By \Cref{lem:Phi_grad_dot}, we have $\partial\Phi(x)\nabla L(x)=0$ for all $x\in U$. Thus we have
	\begin{align*}
		\norm{\Phi(y)- \Phi(x)}
		 = &\eta \norm{\int_{s=0}^1 \partial \Phi\left(x-s\eta\frac{\nabla L(x)}{\norm{\nabla L(x)}}\right) \frac{\nabla L(x)}{\norm{\nabla L(x)}} \diff s}	\\
		 = & \eta \norm{\int_{s=0}^1 \left(\partial \Phi\left(x-s\eta\frac{\nabla L(x)}{\norm{\nabla L(x)}}\right) -\partial \Phi(x)\right)\frac{\nabla L(x)}{\norm{\nabla L(x)}} \diff s}	\\
		 \le & \eta \int_{s=0}^1 \norm{\partial \Phi\left(x-s\eta\frac{\nabla L(x)}{\norm{\nabla L(x)}}\right) -\partial \Phi(x)} \diff s	\\
		 \le & \eta^2 \int_{s=0}^1 s \sup_{s'\in [0,s]} \norm{\nabla^2 \Phi((1-s')x+s'y) } \diff s	\\
		 =  & \frac{\eta^2}{2}\sup_{s'\in [0,1]} \norm{\nabla^2 \Phi((1-s')x+s'y) } \\
		 \le & \frac{1}{2}\secondPhi \eta^2,
	\end{align*}
	where the final step follows from using \Cref{defi:constants}.
	
	For the second claim, we have for every $1 \le k \le M$,
	\begin{align*}
	    &\abs{ \lambda_k( \nabla^2 L( \Phi (x) ) ) - \lambda_k( \nabla^2 L( \Phi (y) ) ) }\\ 
	    \le &\norm{  \nabla^2 L( \Phi (x) )  -  \nabla^2 L( \Phi (y) )  } \\
	    = &\norm{ \int_{s=0}^{1} (1 - s) \partial^2 (\nabla L)(\Phi (sx + (1-s)y) ) (\Phi(x) - \Phi(y)) ds } \\
	    \le &\abs{ \int_{s=0}^{1} (1-s) ds } \max_{s\in [0,1]}  \norm{ \partial^2 (\nabla L)(\Phi (sx + (1-s)y) ) } \norm{ \Phi(x) - \Phi(y)) } \\
	    \le &\frac{1}{4} \thirdL \secondPhi \eta^2,
	\end{align*}
    where the first step involves \Cref{lem:perturb_eigenvalue}.
    
    The third claim follows from using \Cref{lem:DavisKahn}. Again,
    \begin{align*}
        \norm{ v_1(\nabla^2 L (\Phi(x)) ) - v_1(\nabla^2 L (\Phi(y)) ) } &\le \frac{ \norm{  \nabla^2 L( \Phi (x) )  -  \nabla^2 L( \Phi (y) )  } }{ \lambda_1( \nabla^2 L \Phi(x) ) - \lambda_2( \nabla^2 L (\Phi(y)) ) } \\&
        \le \frac{1}{2} \frac{\thirdL \secondPhi \eta^2}{ \lambda_1( \nabla^2 L (\Phi(x)) ) - \lambda_2( \nabla^2 L (\Phi(y) ) } \\&
        \le \frac{1}{2} \frac{\thirdL \secondPhi \eta^2}{ \lambda_1( \nabla^2 L (\Phi(x)) ) - \lambda_2( \nabla^2 L (\Phi(x) ) - \frac{1}{4} \thirdL \secondPhi \eta^2 } \\&
        \le \frac{1}{2} \frac{\thirdL \secondPhi \eta^2}{ \eigengap - \frac{1}{4} \thirdL \secondPhi \eta^2 },
    \end{align*}
    where we borrow the bound on $\norm{ \nabla^2 L (\Phi(x))  -  \nabla^2 L (\Phi(y)) }$ from our previous calculations. The final step follows from the constants defined in \Cref{defi:constants}.
\end{proof}

\begin{lemma}\label{lem:PhiGt}
For any $\overline{xy} \in \flowtraj^\saferadius$ where $y= x - \eta \frac{\nabla L(x)}{\norm{\nabla L(x)} }$ is the one step Normalized GD update from $x$, we have that
    \begin{align*}
        \Phi( y ) - \Phi( x ) &= -\frac{\eta^2}{4} P_{\Phi(x), \Gamma}^{\perp} \nabla (\log \lambda_{1}( \nabla^2 L (\Phi(x)) ))  \\&
         +  O(\eta^2 \secondPhi \theta) + O( \frac{ \thirdL \secondPhi \norm{x - \Phi(x)} \eta^2 }{ \mineigen  } ) + O( \thirdPhi \norm{x - \Phi(x)} \eta^2 ) + O(\thirdPhi \eta^3).
    \end{align*}
     Here $\theta = \arctan \frac{ \norm{ P_{\Phi(x), \Gamma}^{(2:M)} \Tilde{x} } }{ \abs{ \langle v_1( x ),\Tilde{x} \rangle } },$ with $\Tilde{x} = \nabla^2 L ( \Phi(x) ) (x - \Phi(x))$. Additionally, we have that
    \begin{align*}
        \norm{ P_{\Phi(x),\Gamma} ( \Phi(y) - \Phi(x) ) } \le   O(\thirdPhi \norm{ x - \Phi ( x ) } \eta^2)  + O( \thirdPhi \eta^3)+ O( \frac{\thirdL\secondPhi}{\mineigen} \norm{x - \Phi(x  )} \eta^2).
    \end{align*}
\end{lemma}

\begin{proof}[Proof of \Cref{lem:PhiGt}]
By Taylor expansion for  $\Phi$ at $x$, we have
\begin{align*}
    \Phi(y ) - \Phi( x ) \nonumber= &\partial \Phi( x ) \left(  y - x  \right) + \frac{1}{2} \partial^2 \Phi( x ) [ y - x, y - x ] + O(\thirdPhi \norm{y-x}^3) \nonumber\\
    = &\partial \Phi( x ) \left(  -\eta \frac{ \nabla L ( x ) }{ \norm{ \nabla L ( x ) } }  \right) + \frac{\eta^2}{2} \partial^2 \Phi( x ) \left[ \frac{ \nabla L ( x ) }{ \norm{ \nabla L ( x ) } }, \frac{ \nabla L ( x ) }{ \norm{ \nabla L ( x ) } } \right] + O(\thirdPhi \eta^3) \nonumber\\
    = &\frac{\eta^2}{2} \partial^2 \Phi( x ) \left[ \frac{ \nabla L ( x ) }{ \norm{ \nabla L ( x ) } }, \frac{ \nabla L ( x ) }{ \norm{ \nabla L ( x ) } } \right] + O( \thirdPhi \eta^3), 
\end{align*}
where in the pre-final step, we used the property of $\Phi$ from \cref{lem:Phi_grad_dot}. In the final step, we have used a second order taylor expansion to bound the difference between $\partial^2 \Phi( x )$ and $\partial^2 \Phi( \Phi( x ) ).$
Additionally, we have used $y-x = \eta \frac{\nabla L(x)}{\norm{\nabla L(x)} }$ from the Normalized GD update rule.

Applying Taylor expansion on $\Phi$ again but at $\Phi(x)$, we have that
\begin{align}\label{eq:taylorexpansion_Phi_again} 
    \Phi(y ) - \Phi( x )    =  \frac{\eta^2}{2} \partial^2 \Phi( \Phi( x ) ) \left[ \frac{ \nabla L ( x ) }{ \norm{ \nabla L ( x ) } }, \frac{ \nabla L ( x ) }{ \norm{ \nabla L ( x ) } } \right] + O(\thirdPhi \norm{ x - \Phi ( x ) } \eta^2)  + O( \thirdPhi \eta^3)
\end{align}



Also, at $\Phi(x)$, since $v_1(x)$ is the top eigenvector of the hessian $\nabla^2 L$, we have that from \cref{cor:nabla2_tangentperp},
\begin{align}
     \partial^{2} \Phi( \Phi( x ) )\left[ v_1(x) v_1(x)^{\top} \right]  =- \frac{1}{ 2 \lambda_1(x) }\partial \Phi( \Phi( x ) ) \partial^{2}(\nabla L) (\Phi(x)) [v_1(x), v_1(x)]. \label{eq:simplifyingPhiPhi_alongtop_again}
\end{align}

By \Cref{lem:approx_travel1}, it holds that 
\begin{align}
    \norm{ \sign \left( \left\langle \frac{ \nabla L(x) }{ \norm{ \nabla L( x ) } }, v_1(x) \right\rangle\right) \frac{ \nabla L(x) }{ \norm{ \nabla L( x ) } } - v_1(x) } \le 2 \sin \frac{\theta}{2} + O( \frac{ \thirdL \norm{x - \Phi(x)} }{ \mineigen  } ) \le \theta + O( \frac{ \thirdL  \norm{x - \Phi(x)} }{ \mineigen  } ). \label{eq:approx_travel1_theta}
\end{align}


Plug \cref{eq:approx_travel1_theta,eq:simplifyingPhiPhi_alongtop_again} into \cref{eq:taylorexpansion_Phi_again}, we have that
\begin{align*}
    \Phi( y ) - \Phi( x ) &= -\frac{\eta^2}{2} \frac{1}{ 2 \lambda_1(x) }\partial \Phi( \Phi( x ) ) \partial^{2}(\nabla L) (\Phi(x)) [v_1(x), v_1(x)] \\&+ O(\eta^2 \secondPhi \theta) + O( \frac{ \thirdL \secondPhi \norm{x-\Phi(x)} \eta^2 }{ \mineigen  } ) + O( \thirdPhi \norm{x-\Phi(x)}  \eta^2 ) + O(\thirdPhi \eta^3).
\end{align*}

By \Cref{lem:Phi_projection}, for any $x \in \Gamma$, $\partial \Phi(x)$ is the projection matrix onto the tangent space $\tangent_{\Phi(x)}{\Gamma}$. Thus, $\partial \Phi (\Phi(x)) = P_{\Phi(x), \Gamma}^{\perp}$. Thus the proof of the first claim is completed by noting that $\partial \Phi( \Phi( x ) ) \partial^{2}(\nabla L) (\Phi(x)) [v_1(x), v_1(x)] = P_{\Phi(x), \Gamma}^{\perp} \nabla  \lambda_{1}( \nabla^2 L (\Phi(x)) )$ by \cref{lem:riemanniansinglestep}.

For the second claim, continuing from \Cref{eq:taylorexpansion_Phi_again}, we have that
\begin{align*} 
    \Phi(y ) - \Phi( x )    
    = & \frac{\eta^2}{2} \partial^2 \Phi( \Phi( x ) ) \left[ \frac{ \nabla L ( x ) }{ \norm{ \nabla L ( x ) } }, \frac{ \nabla L ( x ) }{ \norm{ \nabla L ( x ) } } \right] + O(\thirdPhi \norm{ x - \Phi ( x ) } \eta^2)  + O( \thirdPhi \eta^3)\\
    = & \frac{\eta^2}{2} \partial^2 \Phi( \Phi( x ) ) \left[\Sigma \right] + O(\thirdPhi \norm{ x - \Phi ( x ) } \eta^2)  + O( \thirdPhi \eta^3) + O( \frac{\thirdL\secondPhi}{\mineigen} \norm{x - \Phi(x  )} \eta^2),
\end{align*}
where $\Sigma = P_{\Phi(x), \Gamma}  \frac{ \nabla L ( x ) }{ \norm{ \nabla L ( x ) } }  \left(P_{\Phi(x), \Gamma} \frac{ \nabla L ( x ) }{ \norm{ \nabla L ( \Phi( x )  )} } \right)^{\top} $ and the last step is by \cref{lem:appr_gradient_norm}.  Here $P_{\Phi(x), \Gamma}$ denotes the projection matrix of the subspace spanned by $v_1(x), \ldots, v_{M}(x)$.

By \Cref{lem:nabla2_Sigmatangentperp,lem:Phi_projection,lem:topeig_in_normal}, we have that 
$$P_{\Phi(x),\Gamma} \partial^2 \Phi( \Phi( x ) ) \left[\Sigma \right] = - P_{\Phi(x),\Gamma} \partial \Phi(x) \partial^2 (\nabla L)(x)[\lyapunov^{-1}_{\nabla^2 L(x)} \Sigma] =0,$$ we conclude that 
    \begin{align*}
        \norm{ P_{\Phi(x),\Gamma} ( \Phi(y) - \Phi(x) ) } \le   O(\thirdPhi \norm{ x - \Phi ( x ) } \eta^2)  + O( \thirdPhi \eta^3)+ O( \frac{\thirdL\secondPhi}{\mineigen} \norm{x - \Phi(x  )} \eta^2),
    \end{align*}
which completes the proof.
\end{proof}

\begin{lemma}\label{lem:ngd_phase_1_firsthalf}
 Let $L_{min}= \min_{y\in U} L(y)$. For any $\overline{xy} \in \flowtraj^\saferadius$ where $y= x - \eta \frac{\nabla L(x)}{\norm{\nabla L(x)} }$ is the one step Normalized GD update from $x$, if $\norm{\nabla L(x_{\eta}(t))}\ge  \secondL\eta$, we have that 
 \begin{align*}
\sqrt{L(y)-L_{min}} \le  \sqrt{L(x) - L_{min}}   -\eta\frac{\sqrt{2\PLconstant}}{4}.
\end{align*}
\end{lemma}

\begin{proof}[Proof of \Cref{lem:ngd_phase_1_firsthalf}]
By Taylor expansion, we have that
\begin{align*}
	L(y)\le L(x) - \eta \norm{\nabla L(x)} + \frac{\secondL\eta^2}{2}.
\end{align*}
Thus for $\norm{\nabla L(x_{\eta}(t))}\ge  \secondL\eta$, we have that 
$$L(y)-L(x)\le -\frac{\eta}{2}\norm{\nabla L(x)}\le -\eta\frac{\sqrt{2\PLconstant}}{2}\sqrt{L(x) - L_{min}}\le 0,$$ 
where the last step is because $L$ is $\mu$-PL on $\flowtraj^\saferadius$. In other words, we have that
$$\sqrt{L(y)-L_{min}} -  \sqrt{L(x) - L_{min}}\le 
-\eta\frac{\sqrt{L(x) - L_{min}}}{\sqrt{L(y)-L_{min}} + \sqrt{L(x)-L_{min}}}\frac{\sqrt{2\PLconstant}}{2}\le -\eta\frac{\sqrt{2\PLconstant}}{4},$$ 
where in the last step we use $L(y)-L(x)\le 0$. This completes the proof.	
\end{proof}

\subsection{Properties of limiting map of gradient flow, $\Phi$}\label{appsec:property_phi}

\begin{lemma}\label{lem:discrete_continuous_same_limit}
If $\lim_{n\to \infty, n\in\NN}\phi(x,n)$ exists, then $\Phi(x)$ also exists and $\Phi(x) = \lim_{n\to \infty, n\in\NN}\phi(x,n)$. 	
\end{lemma}

\begin{proof}[Proof of \Cref{lem:discrete_continuous_same_limit}]
	Suppose $K\subseteq \RR^D$ is a compact set and $\phi(x,t)\in K$ for all $t\in[0,T]$, we have that 
	\begin{align*}
		\frac{\diff \norm{\nabla L(\phi(x,t))}^2}{\diff t} = - 2\nabla L(\phi(x,t))\nabla^2 L(\phi(x,t)) \nabla L(\phi(x,t))\le 2\secondL_K \norm{\nabla L(\phi(x,t))}^2,
	\end{align*}
	where $\secondL_K$ denotes $\sup_{x\in K}\norm{\nabla^2 L(\phi(x,t))}$. This implies that $\norm{\nabla L(\phi(x,t))}\le e^{\secondL_K T }\norm{\nabla L(x)}$  and $\norm{\phi(x,t)-x}\le e^{\secondL_K T }\norm{\nabla L(x)}$ for all $t\in [0,T]$.
	
	 Now suppose $\lim_{n\to \infty, n\in\NN}\phi(x,n) = x^*$, we know $\nabla L(x^*)$ and $\norm{\nabla L (\phi(x,n))}$ must converges to $0$ due to the continuity of $\nabla L$.  Take any compact neighborhood of $x^*$ as the above defined $K$, we know there exists $N>0$, such that for all $n>N$,  $\norm{\nabla L(\Phi(x,n) )}$ and $\norm{\phi(x,n) - x^*}$ are small enough such that $\phi(x,n+\delta)\in K$ for all $\delta\in [0,1]$. Therefore, we know that when $t\to\infty$ as a real number, $\norm{\phi(x,t)-\phi(x,\lfloor t\rfloor)} \le e^{\secondL_K}\norm{\nabla L(\phi(x,\lfloor t\rfloor))}\to 0$. This completes the proof. 
\end{proof}

\begin{lemma}\label{lem:U_open_Phi_smooth}
	Let $U=\{x\in \RR^D \mid \Phi(x) \textrm{ exists and } \Phi(x)\in \Gamma\}$. We have that  $U$ is open and $\Phi$ is $L\mathcal{C}^2$ on $U$. 
\end{lemma} 

\begin{proof}[Proof of \Cref{lem:U_open_Phi_smooth}] 
	  We first define $f:\RR^D\to \RR^D$, $f(x) = \phi(x,1)$, and we have $\Gamma$ is the set of the fixed points of mapping $f$. Since $L$ is $\mathcal{C}^4$, $\nabla L$ is $\mathcal{C}^{3}$ and thus $L\mathcal{C}^{2}$. Thus $f$ is $L\mathcal{C}^{2}$. By Theorem 5.1 of \citet{Falconer_1983}\footnote{The original version Theorem 5.1 requires $f$ to be $k$ times globally lipschitz differentiable and the conclusion is only that $\Phi$ is $C^k$ for any natural number $k$. However, the proof actually only uses $k$ times locally lipschitz differentiability and proves that $\Phi$ is $k$ times locally lipschitz differentiable.}, we know that there is an open set $V$ containing $\Gamma$ and $f^\infty(x)$ is well-defined and $\mathcal{C}^{2}$ on $V$ with $f^\infty(x)\in\Gamma$ for any $x\in\Gamma$, where $f^\infty(x):= \lim_{n\to \infty} f^n(x)$, $f^n(x) = f(f^{n-1}(x))$ and $f^1(x) =f(x)$. By \Cref{lem:discrete_continuous_same_limit}, we know that $\Phi(x) = f^\infty(x)$. Therefore $V\subseteq U$.
	
	Since  $\Phi(x)\in \Gamma$ for all $x\in U$, we know that there is a $t>0$, such that $\phi(x,t)\in V$. Thus $U = \cup_{t\in \NN } \phi(V,-t)$ is a union of open sets, which is still open. Moreover, $\Phi(x) = \Phi(\phi(x,t))$ for each $x\in U$ and any $t>0$ with $\phi(x,t)\in V$. Since $\Phi$ is $L\mathcal{C}^{2}$ in $V$ and $\phi(\cdot,t)$ is $\mathcal{C}^{k-1}$ in $\RR^D$ for any $t$, we conclude that $\Phi$ is $L\mathcal{C}^2$ in $U$.
\end{proof}

The following results  \Cref{lem:Phi_grad_dot,lem:nabla2_tangent,lem:Phi_projection,lem:topeig_in_normal,lem:nabla2_tangentperp,defi:lyapunov,lem:nabla2_Sigmatangentperp} are from \citep{li2022what}.

\begin{lemma}\label{lem:Phi_grad_dot}
    For any $x \in U$, it holds that (1). $\partial \Phi(x) \nabla L(x)=0$ and
(2). $\partial^{2} \Phi(x)[\nabla L(x), \nabla L(x)]=-\partial \Phi(x) \nabla^{2} L(x) \nabla L(x)$.
\end{lemma}

\begin{lemma} \label{lem:nabla2_tangent}
    $
    \text { For any } x \in \Gamma \text { and any } v \in \tangent_{x}{\Gamma} \text {, it holds that } \nabla^{2} L(x) v=0 \text {. }
    $
\end{lemma}

\begin{lemma} \label{lem:Phi_projection}
For any $x \in \Gamma$, $\partial \Phi (x) \in \mathbb{R}^{D \times D}$ is the projection matrix onto the tangent space $\tangent_{x}{\Gamma}$, i.e. $\partial \Phi(x) = P_{x, \Gamma}^\perp$.
\end{lemma}

\begin{lemma}\label{lem:topeig_in_normal}
    For any $x \in \Gamma$, if $v_1, \ldots, v_M$ denote the non-zero eigenvectors of the hessian $\nabla^2 L( \Phi( x ) )$, then $v_1, \ldots, v_M \in \normal_{ x } {\Gamma} $. 
\end{lemma}

\begin{lemma}\label{lem:nabla2_tangentperp}
    For any $x \in \Gamma$ and $u \in \normal_{x}{\Gamma}$, it holds that
$$
\partial^{2} \Phi(x)\left[u u^{\top}+\nabla^{2} L(x)^{\dagger} u u^{\top} \nabla^{2} L(x)\right]=-\partial \Phi(x) \partial^{2}(\nabla L)(x)\left[\nabla^{2} L(x)^{\dagger} u u^{\top}\right].
$$
\end{lemma}

\begin{definition}[Lyapunov Operator] \label{defi:lyapunov}
For a symmetric matrix $H$, we define $W_{H}=\left\{\Sigma \in \mathbb{R}^{D \times D} \mid\right.$ $\left.\Sigma=\Sigma^{\top}, H H^{\dagger} \Sigma=\Sigma=\Sigma H H^{\dagger}\right\}$ and Lyapunov Operator $\mathcal{L}_{H}: W_{H} \rightarrow W_{H}$ as $\mathcal{L}_{H}(\Sigma)=$ $H^{\top} \Sigma+\Sigma H$. It's easy to verify $\mathcal{L}_{H}^{-1}$ is well-defined on $W_{H}$.
\end{definition}

\begin{lemma}\label{lem:nabla2_Sigmatangentperp}
For any $x \in \Gamma$ and $\Sigma = \mathrm{span}\{ u u^{\top} \mid u \in \normal_x{\Gamma} \}$,
\begin{align*}
    \langle \partial^2 \Phi(x), \Sigma \rangle = -\partial \Phi(x) \partial^2 (\nabla L)(x)[\lyapunov^{-1}_{\nabla^2 L(x)} (\Sigma)].
\end{align*}
\end{lemma}

We will also use the following two corollaries of \Cref{lem:nabla2_Sigmatangentperp}.
\begin{corollary}\label{cor:nabla2_tangentperp}
For any $x \in \Gamma$, let $v_1$ be a  top eigenvector of $\nabla^2 L(x)$, then
\begin{align*}
    \partial^{2} \Phi(x) [v_1v_1^{\top}] = -\frac{1}{2 \lambda_{1} ( \nabla^2 L(x) ) } \partial \Phi(x) \partial^{2}(\nabla L)(x) [v_1, v_1]
\end{align*}
\end{corollary}

\begin{proof}[Proof of \Cref{cor:nabla2_tangentperp}]
	Simply note that $\lyapunov^{-1}_{\nabla^2 L(x)}(v_1v_1^\top) = \frac{1}{2\lambda_1(\nabla^2 L(x))}v_1v_1^\top$ and apply \Cref{lem:nabla2_Sigmatangentperp}. 
\end{proof}

\begin{corollary}\label{lem:riemanniansinglestep}
For any $x \in \Gamma$, let $v_1$ be the unit top eigenvector  of $\nabla^2 L(x)$, then
\begin{align*}
    \partial^{2} \Phi(x) [v_1v_1^{\top}] = -\frac{1}{2} P^\perp_{x, \Gamma} \nabla \log (  \lambda_{1} ( \nabla^2 L(x) ) ).
\end{align*}
\end{corollary}

\begin{proof}[Proof of \Cref{lem:riemanniansinglestep}]
The proof follows from using \cref{cor:nabla2_tangentperp} and the derivative of $\lambda_{1}$ from \cref{lem:derivative_eig}.
\end{proof}

As a variant of \Cref{lem:riemanniansinglestep}, we have the following lemma.
\begin{corollary}\label{lem:riemanniansinglestep_no_log}
For any $x \in \Gamma$, let $v_1$ be the unit top eigenvector  of $\nabla^2 L(x)$, then
\begin{align*}
    \partial^{2} \Phi(x) [\lambda_1(\nabla^2 L(x))v_1v_1^{\top}] = -\frac{1}{2} P^\perp_{x, \Gamma} \nabla  (  \lambda_{1} ( \nabla^2 L(x) ) ).
\end{align*}
\end{corollary}

\section{Analysis of Normalized GD on General Loss Functions}\label{appsec:general_loss_analysis}

\subsection{Phase I, Convergence}\label{appsec:phaseIconv}

We restate the theorem concerning Phase I for the Normalized GD algorithm. Recall the following notation for each  $1 \le j \le M$:
$$R_j(x):=\sqrt{ \sum_{i=j}^{M} \lambda^2_i(x) \langle v_i(x),  x - \Phi(x)\rangle^2 } - \lambda_j(x) \eta, \ \text{ for all } x\in U.$$
\theoremngdfirstphase*
The intuition behind the above theorem is that for sufficiently small LR $\eta$, $x_\eta(t)$ will track the normalized gradient flow starting from $\xinit$, which is a time-rescaled version of the standard gradient flow. Thus the normalized GF will enter $\flowtraj^\saferadius$ and so does normalized GD. Since $L$ satisfies PL condition in $\flowtraj^\saferadius$, the loss converges quickly and the iterate $x_\eta(t)$ gets $\mathcal{\eta}$ to manifold. To finish, we need the following theorem, which is the approximately-quadratic version of \Cref{lem:prepphase} when the iterate is ${O}(\eta)$ close to the manifold.

\begin{lemma}\label{lem:phase0}
    Suppose $\{x_{\eta}(t)\}_{t\ge 0}$ are iterates of Normalized GD~\eqref{eq:normalized_gd} with a learning rate $\eta$ and $x_\eta(0) = \xinit$.
    There is a constant $C>0$, such that for any constant $\secphaseinit>1$, if at some time $t'$, $x_{\eta}(t') \in \flowtraj^\saferadius$ and satisfies $  \frac{\norm{x_{\eta}(t') - \Phi(x_{\eta}(t'))}}{\eta} \le \secphaseinit$,
    then for all $\bar{t} \ge t' + C \frac{\secondL \secphaseinit}{\mineigen} \log \frac{\secphaseinit \secondL }{ \mineigen }$, the following must hold true for all $1 \le j \le M$:
    \begin{align}  \label{eq:condition_perturbed_proof}
         \sqrt{ \sum_{i=j}^{M} \langle v_i(\bar{t}), \tilde{x}(\bar{t}) \rangle^2 }   
                \le \eta \lambda_j(\bar{t}) +   {O}( \eta^2),
    \end{align}
    provided that for all steps $t \in \{t',\ldots, \bar{t}-1\}$, $\overline{x_{\eta}(t)x_{\eta}(t+1)} \subset \flowtraj^\saferadius$. 
\end{lemma}
The proof of the above theorem is in \Cref{sec:phase0}.

\begin{proof}[Proof of \Cref{thm:ngd_phase_1}]

We define the Normalized gradient flow as $\overline\phi(x,\conttime) = x-\int_{0}^\conttime \frac{\nabla L(\overline \phi(x,s))}{\norm{\nabla L(\overline \phi(x,s))}}\diff s$. Since $\overline \phi(x,\cdot)$ is only a time rescaling of $\phi(x,\cdot)$, they have the same limiting mapping, i.e.,  $\Phi(x)=\lim _{\conttime \rightarrow T_x} \overline\phi(x, \conttime)$, where $T_x$ is the length of the trajectory of the gradient flow starting from $x$. 

\newcommand{\criticalloss}{{L_{\text{critical}}}}

Let $T_x$ be the length of the GF trajectory starting from $x$, and we know $\lim_{\conttime\to T_x} \overline \phi(x,\conttime) =\Phi(x)$, where $\overline \phi(x,\conttime)$ is defined as the Normalized gradient flow starting from $x$.
In \Cref{lem:exist_saferadius,lem:exist_presaferadius} we show there is a small neighbourhood around $\Phi(\xinit)$, $\flowtraj^\saferadius$ such that $L$ is $\PLconstant$-PL in $\flowtraj^\saferadius$. 
Thus we can take some time $T_0<T_{\xinit}$ such that $\overline \phi(\xinit,T_0)\in \flowtraj^{\saferadius/2}$ and $L(\overline \phi(\xinit), T_0)\le \frac{1}{2}\criticalloss$, where $\criticalloss:= \frac{\saferadius^2\PLconstant}{8}$. (Without loss of generality, we assume $\min_{y\in \flowtraj} L(y)=0$) By standard ODE approximation theory, we know there is some  small $\eta_0$, such that for all $\eta\le \eta_0$, $\norm{x_{\eta}(\lceil T_0 /\eta\rceil) - \overline\phi(\xinit,T_0)} = {O}(\eta)$, where  ${O}(\cdot)$ hides constants depending on the initialization $\xinit$ and the loss function $L$.

Without loss of generality, we can assume $\eta_0$ is small enough such that $x_\eta(\lceil T_0 /\eta\rceil)\in \flowtraj^\saferadius$ and $L(x_\eta(\lceil T_0 /\eta\rceil)) \le \criticalloss$. 
Now let $t_\eta$ be the smallest integer (yet still larger than $\lceil T_0 /\eta\rceil$)  such that $\overline {x_\eta(t_\eta)x_\eta(t_\eta-1)} \not \subset \flowtraj^\saferadius$
 and we claim that there is  $t \in \{ \lceil T_0 /\eta\rceil , \ldots, t_\eta\}$, $\norm{\nabla L(x_\eta(t))}< \secondL \eta$. By the definition of $t_\eta$,  we know for any $t \in \{ \lceil T_0 /\eta\rceil +1, \ldots, t_\eta-1\}$, by \Cref{lem:boundmovementinPhi} we have 
$\norm{\Phi(x_\eta(t)) - \Phi(x_\eta(t-1))} \le \secondPhi \eta^2,$ 
and by  \Cref{lem:ngd_phase_1_firsthalf}, 
$\sqrt{L(x_\eta(t))}- \sqrt{x_\eta(t-1)}\le -\eta \frac{\sqrt{2\PLconstant}}{4}$ if $\norm{\nabla L(x_\eta(t))}\ge \secondL \eta$. 
If the claim is not true, since $\sqrt{L(x_\eta(t))}$ decreases $\eta \frac{\sqrt{2\PLconstant}}{4}$ per step, we have 
$$0\le \sqrt{L(x_\eta(t_\eta-1))}\le \sqrt{L(x_\eta( \lceil T_0 /\eta\rceil))} - (t_\eta-\lceil T_0 /\eta\rceil-1)\eta \frac{\sqrt{2\PLconstant}}{4},$$
which implies that $t_\eta-\lceil T_0 /\eta\rceil-1 \le \frac{\saferadius}{\eta} $, and therefore by \Cref{lem:boundmovementinPhi}, 
\begin{align*}
\norm{\Phi(x_\eta(t_\eta-1)) - \Phi(x_\eta( \lceil T_0 /\eta\rceil) )}\le   (t_\eta-\lceil T_0 /\eta\rceil-1 )\frac{\secondPhi\eta^2}{2} = \frac{\secondPhi\eta\saferadius}{2}
\end{align*}
Thus we have 
\begin{align*}
 &\norm{ \Phi(x_\eta(t_\eta-1))- \Phi(\xinit) }\\
  \le & \norm{\Phi(x_\eta(t_\eta-1)) - \Phi(x_\eta( \lceil T_0 /\eta\rceil) )} + \norm{\Phi(x_\eta( \lceil T_0 /\eta\rceil) -\Phi(\overline \phi(\xinit, T_0)))} ={O}(\eta).
\end{align*}

Meanwhile, by \Cref{lem:bound_GF_length_by_loss}, we have $\norm{\Phi(x_\eta(t_\eta-1)) - x_\eta(t_\eta-1)} \le \sqrt{\frac{2 L(x_\eta(t_\eta-1))}{\PLconstant}}\le \sqrt{\frac{2 L(x_\eta(\lceil T_0 /\eta\rceil))}{\PLconstant}} = \frac{\saferadius}{2}$.
Thus for any $\kappa\in [0,1]$, we have $\norm{\kappa x_\eta(t_\eta) + (1-\kappa) x_\eta(t_\eta-1)- \Phi(\xinit) }$ is upper bounded by 
\begin{align*}  \kappa\norm{x_\eta(t_\eta) - x_\eta(t_\eta-1)} + &\norm{\Phi(x_\eta(t_\eta-1)) -  x_\eta(t_\eta-1) } + \norm{\Phi(x_\eta(t_\eta-1)) -  \Phi(\xinit)} \\
= &\kappa \eta + \frac{\saferadius}{2}+ {O}(\eta),	
\end{align*}

which is smaller than $\saferadius$  since we can set $\eta_0$ sufficiently small. In other words, $\overline{\Phi(x_\eta(t_\eta))\Phi(x_\eta(t_\eta-1))}\subset \flowtraj^\saferadius$, which contradicts with the definition of $t_\eta$. 
So far we have proved our claim that there is some $t'_\eta \in \{ \lceil T_0 /\eta\rceil , \ldots, t_\eta\}$, $\norm{\nabla L(x_\eta(t'_\eta))}< \secondL \eta$.
Moreover, since $\sqrt{L(x_\eta(t))}$ decreases $\eta \frac{\sqrt{2\PLconstant}}{4}$ per step before $t'_\eta$, we know $t'_\eta- \lceil T_0 /\eta\rceil \le \frac{\saferadius}{\eta}$. 
By \Cref{lem:bound_GF_length_by_loss}, we know $\norm{x_\eta(t'_\eta) - \Phi(x_\eta(t'_\eta))}\le \frac{\secondL \eta}{\PLconstant}$.

Now we claim that for any $T_1'$, there is some sufficiently small threshold $\eta_0$, $t_\eta \ge \frac{T_1'}{\eta}+1$ if $\eta\le \eta_0$. Below we prove this claim by contradiction. If the claim is not true, that is, $t_\eta < \frac{T_1'}{\eta}+1$. if $t_\eta \le C\frac{\secondL \secphaseinit}{\mineigen} \log \frac{\secphaseinit \secondL }{ \mineigen } + t'_\eta $ with $\secphaseinit = \frac{\secondL}{\PLconstant}$, 
we know $\norm{x_\eta(t_\eta) - \Phi(\xinit)}\le \norm{x_\eta(t_\eta) - x_\eta(t'_\eta)} +  \norm{x_\eta(t'_\eta) - \Phi(x_\eta(t'_\eta))} + \norm{\Phi(x_\eta(t'_\eta)) - \Phi(\xinit)} = {O}(\eta)$,  which implies that $\overline{x_\eta(t_\eta) x_\eta(t_\eta-1) }\in Y$. If  $t_\eta \ge  C\frac{\secondL \secphaseinit}{\mineigen} \log \frac{\secphaseinit \secondL }{ \mineigen } + t'_\eta$, by \Cref{lem:phase0}, we have $\norm{x_\eta(t_\eta)- \Phi(x_\eta(t_\eta))} = {O}(\eta)$. 
 By \Cref{lem:boundmovementinPhi}, we have $\norm{\Phi(x_\eta(t_\eta)) - \Phi(x_\eta(\lceil T_0 /\eta\rceil))}\le {O}(\eta)$. Thus again we have that 
$\norm{x_\eta(t_\eta) - \Phi(\xinit)}\le \norm{x_\eta(t_\eta) - \Phi(x_\eta(t_\eta) )} +  \norm{\Phi(x_\eta(t_\eta) ) - \Phi(x_\eta(\lceil T_0 /\eta\rceil))} \allowbreak + \norm{\Phi(x_\eta(\lceil T_0 /\eta\rceil )) - \Phi(\xinit)} = {O}(\eta)$,  which implies that $\overline{x_\eta(t_\eta) x_\eta(t_\eta-1) }\in Y$.  In both cases, the implication is in contradiction to the definition of $t_\eta$.

Thus for any $T_1'$, $t_\eta \ge \frac{T_1'}{\eta}+1$ for sufficiently small threshold $\eta_0$ and $\eta\le \eta_0$.  To complete the proof of \Cref{thm:ngd_phase_1}, we pick $T_1$ to be any real number strictly larger than $\saferadius+ T_0$, as $\frac{T_1}{\eta} > C\frac{\secondL \secphaseinit}{\mineigen} \log \frac{\secphaseinit \secondL }{ \mineigen } + \frac{\saferadius}{\eta} + \lceil T_0 /\eta\rceil \ge C\frac{\secondL \secphaseinit}{\mineigen} \log \frac{\secphaseinit \secondL }{ \mineigen } + t'_\eta $ 
when $\eta$ is sufficiently small with $\secphaseinit = \frac{\secondL}{\PLconstant}$. By \Cref{lem:phase0} the second claim of \Cref{thm:ngd_phase_1} is proved. Using the same argument again, we know $\forall \frac{T_1}{\eta}\le t\le \frac{T_1'}{\eta}$, it holds that 
 	$ \norm{ \Phi(x_\eta(t))- \Phi(\xinit) } \le {O}(\eta)$.
\end{proof}

\subsection{Phase II, Limiting Flow}
We first restate the main theorem that demonstrates that the trajectory implicitly minimizes sharpness.
\theoremngdsecondphase*

To show the closeness between the continuous and the discrete dynamic, we will need to use the following classic differential inequality from \citep{hairer1993solving}. The original statement is for differential equations defined on $\RR^D$. Without loss of generality, we can restrict it to an open subset of $\RR^D$ with the same proof.

\begin{theorem}\label{thm:fundamental_lemma}[Adaption of ``Variant form of Theorem 10.2'', p.59, \citep{hairer1993solving}]
Let $U$ be an open subset of $\RR^D$. Suppose that $\{y(\conttime)\in U\}_{\conttime =0}^{T}$ is a solution of the differential equation $\frac{\diff y}{\diff \conttime} = f( y(\conttime))$, $y(\conttime) = y_0$, and that $v(\conttime)\in U$ is a piecewise linear curve. If $f(y)$ is $\lipschitz$-Lipschitz in $y$, that is, $\forall y,y'\in U$, $\conttime \in [0,T]$, $\norm{f(y)-f(y')} \le \lipschitz$, then for any $0\le \conttime \le T$, it holds that for any $\conttime\in[0,T]$, 
\begin{align*}
	\norm{y(\conttime) - v(\conttime)}
	 \le & e^{T\lipschitz}\left( \norm{v(0) - y(0)} + \int_{\conttime'=0}^\conttime e^{-\lipschitz\conttime} \norm{v'(\conttime'+0) - f(v(\conttime))} \right) \diff \conttime'\\
	 \le & e^{T\lipschitz}\left( \norm{v(0) - y(0)} + \int_{\conttime'=0}^\conttime \norm{v'(\conttime'+0) - f(v(\conttime))} \right) \diff \conttime',
\end{align*} 
\end{theorem}
where  $v'(\conttime'+0):= \lim_{\delta\to 0}\frac{v(\conttime'+\delta)-v(\conttime')}{\delta}$ is the right time derivative of $v$ at $\conttime'$.

\begin{proof}[Proof of \Cref{thm:ngd_phase_2}]
	Without loss of generality, we can change assumption (3) in the theorem statement into $\norm{\Tilde x_{\eta}(0)} \le \eta \lambda_1(0) / 2 + \normerrterm \eta^2$ and $ \abs{ \langle v_1(x_\eta(0)), \tilde x_\eta(0) \rangle } \ge \Omega(\eta)$. (Constant $\normerrterm$ is defined in \Cref{lem:dropbyhalf_general})
	This is because  we know from \Cref{lem:dropbyhalf_general}, that the norm can't stay above $\frac{\lambda_1(\cdot)}{2} \eta  + \normerrterm \eta^2$ for two consecutive steps. 
	Moreover, if $\abs{v_1(0), \tilde x_{\eta}(0) } \ge \Omega(\eta)$ but $\eta \lambda_1(0) / 2 + \Omega(\normerrterm \eta^2) \le \norm{\Tilde x_{\eta}(0)} \le \eta \lambda_1(0) - \Omega(\eta)$, we can further show that $\abs{v_1(1), \tilde x_{\eta}(1) } \ge \Omega(\eta)$ from the update rule of Normalized GD (\Cref{lem:appr_gradient_norm}). 
	Thus, we can shift our analysis by one time-step if our assumption isn't true at step $0$. This simplification of assumption helps us to prove the second claim using \Cref{lem:avgGt}.
	
	To prove the first claim, we first show the movement in the manifold for the discrete trajectory for \Cref{alg:PNGD} by  \Cref{lem:PhiGt}: for each step $t$, provided $\overline{\Phi(x_{\eta}(t))  \Phi(x_{\eta}(t+1))} \in \flowtraj^\saferadius$, it holds that
\begin{align}
       \Phi(x_{\eta} (t+1)) - \Phi(x_{\eta}(t)) = - \frac{\eta^2 }{ 4 }  P_{ \Phi(x_{\eta} (t) ), \Gamma }^{\perp} \nabla \log \lambda_1 (  x_{\eta} (t)   ) + O(( \theta_t + \norm{x_{\eta}(t) - \Phi (x_{\eta} (t) ) } ) \eta^2) .
       \label{eq:movement_in_Phi_informal}
\end{align}

To recall, the limiting flow is given by
\begin{align*}
    \loglimitingflow
\end{align*}

The high-level idea for the proof of the first claim is to bound the gap between \Cref{eq:movement_in_Phi_informal} and \Cref{eq:log_limiting_flow} using \Cref{thm:fundamental_lemma}. And the first claim eventually boils down to upper bound the average angle by $O(\eta)$, which is exactly the second claim.

Formally, let $t_2$ be the largest integer no larger than $\lfloor T_2/\eta^2 \rfloor$ such that for any $0\le t \le t_2$, it holds that $\overline{\Phi(x_{\eta}(t))  \Phi(x_{\eta}(t+1))} \in \flowtraj^\saferadius$. 

To apply \Cref{thm:fundamental_lemma}, we let $y(\conttime) = X(\conttime)$, $f:U\to \RR^D, f(y) = P_{ \Phi( y ), \Gamma}^{\perp} \nabla \log \lambda_1(y)$, $\lipschitz$ be an upper bound for lipschitzness of $f$ on compact set $\flowtraj^\saferadius$ and $v(\conttime) = \Phi(x_\eta(\lfloor\conttime/\eta^2\rfloor)) +  (\conttime/\eta^2 - \lfloor\conttime/\eta^2\rfloor))\left( \Phi(x_\eta(\lfloor\conttime/\eta^2\rfloor+1)) -\Phi(x_\eta(\lfloor\conttime/\eta^2\rfloor))\right)$. In other words, $v$ is a piecewise linear curve interpolating all $x_\eta(t)$ at time $t\eta^2$.  Therefore, by \Cref{eq:movement_in_Phi_informal}, note $\Phi(x)=\Phi(\Phi(x))$ for all $x\in U$, it holds that
\begin{align*}
\norm{v'(t\eta^2 + 0) - f(v(t\eta^2))} 
=& \norm{1/\eta^2(\Phi(x_\eta(t+1)) - \Phi(x_\eta(t))) - P_{ \Phi( x_\eta(t))) , \Gamma}^{\perp} \nabla \log \lambda_1(\Phi(x_\eta(t))))} \\
=& O(\theta_t + \norm{x_{\eta}(t) - \Phi (x_{\eta} (t) ) }). 	
\end{align*}

Since  we  started from a point that has $\max_{1 \le j \le M} R_j( x_{\eta} (0) ) \le O(\eta^2)$, we have from \Cref{lem:phase0}, that the iterate  satisfies the condition  $\max_{1 \le j \le M} R_j( x_{\eta} (t) ) \le O(\eta^2)$ at step $t$ as well, meaning that $\norm{x_{\eta}(t) - \Phi( x_{\eta} (t) )} \le O( \eta)$. 

Therefore, for any $\conttime\le t_2\eta^2$, note that $v'(\conttime+0) =v'(\lfloor\conttime/\eta^2\rfloor+0) $ and that $\norm{f(v(\lfloor\conttime/\eta^2\rfloor+0)) - f(v(\conttime))} = O(\norm{\Phi(x_\eta(\lfloor\conttime/\eta^2\rfloor+1)) -\Phi(x_\eta(\lfloor\conttime/\eta^2\rfloor))}) = O(\eta^2)$, we have that 
\begin{align*}
\norm{v'(\conttime+0) - f(v(\conttime))}
 =& \norm{v'(\lfloor\conttime/\eta^2\rfloor+0) - f(v(\lfloor\conttime/\eta^2\rfloor))} + \norm{f(v(\lfloor\conttime/\eta^2\rfloor+0)) - f(v(\conttime))} \\
 \le & O(\theta_{\lfloor\conttime/\eta^2\rfloor} + \eta)  + O(\eta^2) = O(\theta_{\lfloor\conttime/\eta^2\rfloor} + \eta) .
\end{align*}

Thus by \Cref{thm:fundamental_lemma}, we conclude that 
$$ X(t_2\eta^2) - x_\eta(t_2) = y(t_2\eta^2)-v(t_2\eta^2) = O\big(\eta + \int_{\tau=0}^{t_2\eta^2} (\theta_{\lfloor\conttime/\eta^2\rfloor} + \eta)\big) = O(\eta+ \big(\eta^2\sum_{t=0}^{t_2} (\theta_{t} + \eta)\big)) = O(\eta),$$ 
where in the last step we use the second claim. This implies that $t_2$ must be equal to $\lfloor T_2/\eta^2\rfloor$ for sufficiently small $\eta$ otherwise $\overline{x_\eta(t_2)x_\eta(t_2+1)}\subseteq Y^\saferadius$. This is  because $\norm{x_\eta(t_2+1)-x_\eta(t_2)} = O(\eta)$ and $X(t_2\eta^2)\in Y$. The proof is completed by noting that $ \norm{X(T_2) -  X(\lfloor T_2/\eta^2\rfloor)} = O(\eta^2)$.
\end{proof}

\section{Phase I, Omitted Proofs of the Main Lemmas}\label{appappsec:phaseIconv_proofs}
\subsection{Proof of \Cref{lem:phase0}}\label{sec:phase0}


\begin{proof}[Proof of \Cref{lem:phase0}]
      The Normalized GD update at any step $t$ can be written as (from \Cref{lem:appr_gradient_norm})
     \begingroup
     \allowdisplaybreaks
        \begin{align}
            x_{\eta}(t+1) - x_{\eta}(t) 
            =  - \eta  \frac{\nabla^2 L (\Phi(x_{\eta}(t))) [ x_{\eta}(t) - \Phi ( x_{\eta}(t) ) ] }{ \norm{ \nabla^2 L (\Phi(x_{\eta}(t))) [ x_{\eta}(t) - \Phi ( x_{\eta}(t) ) ]   } } + {O}( \frac{\thirdL}{\mineigen} \eta \norm{x_{\eta}(t) - \Phi(x_{\eta}(t))}) . \label{eq:noisynormaliedgdupdate_local}
        \end{align}
    \endgroup    
        
        From \Cref{lem:boundmovementinPhi}, we have $\norm{\Phi(x_{\eta}(t)) - \Phi(x_{\eta}(t+1))} \le {O}(\secondPhi \eta^2)$, which further implies, \\
            $\norm{ \nabla^2 L(\Phi(x_{\eta}(t+1))) - \nabla^2 L(\Phi(x_{\eta}(t))) } \le {O}( \thirdL \secondPhi \eta^2 ) $.  Thus, using the notation $\Tilde{x} = \nabla^2 L(\Phi(x)) (x - \Phi(x))$, we have
        \begingroup
        \allowdisplaybreaks
        \begin{align*}
            \Tilde{x}_{\eta}(t+1) - \Tilde{x}_{\eta}(t) 
            &= \Tilde{x}_{\eta}(t+1) - \nabla^2 L(\Phi ({x}_{\eta} (t)) ) (x_{\eta}(t+1) - \Phi(x_{\eta}(t)) )  \\ 
            &\qquad + \nabla^2 L(\Phi ({x}_{\eta} (t)) ) (x_{\eta}(t+1) - \Phi(x_{\eta}(t)) ) - \Tilde{x}_{\eta}(t) \\
            &=\nabla^2 L(\Phi(x_{\eta}(t+1))) (\Phi(x_{\eta}(t)) - \Phi(x_{\eta}(t+1))) \\
            &\qquad + ( \nabla^2 L(\Phi(x_{\eta}(t+1))) - \nabla^2 L(\Phi(x_{\eta}(t))) ) (x_{\eta}(t+1) - \Phi(x_{\eta}(t))) \\
            &\qquad + \nabla^2 L(\Phi ({x}_{\eta} (t)) ) (x_{\eta}(t+1) - x_{\eta}(t)) \\
			&=  -\eta \nabla^2 L(\Phi ({x}_{\eta} (t)) )  \frac{\Tilde{x}_{\eta}(t)}{\norm{ \Tilde{x}_{\eta}(t)} }  + \mathrm{err} + {O}(\eta^2  + \eta\norm{x_{\eta}(t) - \Phi(x_{\eta}(t))}),
        \end{align*}
        \endgroup
        That is,
        \begin{align}\label{eq:noisynormalizedGD}
            \Tilde{x}_{\eta}(t+1) = \left(I - \eta \frac{ \nabla^2 L(\Phi(x_{\eta}(t)))}{ \norm{ \Tilde{x}_{\eta} (t) }} \right)\Tilde{x}_{\eta}(t) +  {O}(   \eta^2 ) + {O}(\eta\norm{x_{\eta}(t) - \Phi(x_{\eta}(t))}).
        \end{align}
        
        Below we will show that $\norm{x_{\eta}(t) - \Phi(x_{\eta}(t))} \le {O}(\eta)$, and thus the trajectory of $\Tilde{x}_\eta$ is similar to the trajectory in the qudratic model with an ${O}(\eta^2)$ error, with the hessian fixed at $\nabla^2 L(\Phi ( x_{\eta}(t) ) )$, and hence we can apply the same techniques from \Cref{lem:invariantj} and \Cref{lem:quadratic_norm_invariant}.

        First, we consider the norm of the vector $\Tilde{x}_{\eta}(t)$ for $t' + 1 \le t \le \overline{t}$. We will show the following induction hypothesis:
        \begin{align*}
            \norm{\Tilde{x}_{\eta}(t)} \le 1.01 \eta \secondL \secphaseinit.
        \end{align*}
        \begin{enumerate} 
            \item \textbf{Base case:} $(t = t')$. We have $\norm{\Tilde{x}_{\eta}(t') } =   \norm{ \nabla^2 L(\Phi(x_{\eta}(t'))) [x_{\eta}(t') - \Phi(x_{\eta}(t'))] } \le \eta\lambda_1(t) \secphaseinit \le \eta \secondL \secphaseinit$.
            \item \textbf{Induction case:}$(t > t')$. Suppose the hypothesis holds true for  $t-1$. Then, $$\norm{x_{\eta}(t-1) - \Phi(x_{\eta}(t-1))} \le \frac{1}{\lambda_M(t)} \norm{ \Tilde{x}_{\eta}(t-1) } \le \frac{1.01 \eta \secphaseinit \secondL}{\mineigen}.$$
			
             We consider the following two cases: 
            \begin{enumerate}
                \item If $\norm{\Tilde{x}_{\eta}(t-1)} \ge \eta \lambda_1(t)$. We can directly apply \Cref{lem:invariant1} on \eqref{eq:noisynormalizedGD} to show that
            \begin{align*}
                 \norm{\Tilde{x}_{\eta}(t)}   
                 &\le  \left( 1 - \frac{ \eta \lambda_M(t-1) }{ 2  \norm{\Tilde{x}_{\eta}(t-1)} }  \right)   \norm{\Tilde{x}_{\eta}(t-1)}   + {O}( \thirdL \secondPhi  \eta^2  ) \\
                 &+ {O}(\frac{\thirdL \secondL}{\mineigen} \eta  \norm{x_{\eta}(t-1) - \Phi(x_{\eta}(t-1))}) \\
                 &\le \norm{\Tilde{x}_{\eta}(t-1)}  - \frac{ \eta \lambda_M(t-1) }{ 2   }  + {O}( \thirdL \secondPhi  \eta^2  ) + {O}(\frac{\thirdL \secondL}{\mineigen^2} \secphaseinit \eta^2 ) \\&
                \le \norm{\Tilde{x}_{\eta}(t-1)}  - \frac{ \eta \lambda_M(t-1) }{ 4 }  ,
            \end{align*}
            where the final step follows if $\eta$ is sufficiently small.  Hence, $\norm{\Tilde{x}_{\eta}(t)} < \norm{\Tilde{x}_{\eta}(t-1)} \le \eta \secondL \secphaseinit$.
            
            \item If $\norm{\Tilde{x}_{\eta}(t-1)} \le \eta \lambda_1(t)$. Then, we can directly apply \Cref{lem:quadratic_norm_invariant} on \eqref{eq:noisynormalizedGD} to show that
            \begin{align*}
                \norm{\Tilde{x}_{\eta}(t)} &\le \eta \lambda_1(t) + {O}( \thirdL \secondPhi  \eta^2  ) + {O}(\frac{\thirdL \secondL}{\mineigen} \eta  \norm{x_{\eta}(t-1) - \Phi(x_{\eta}(t-1))}) \\&
                \le \eta \lambda_1(t) + {O}( \thirdL \secondPhi  \eta^2  ) + {O}(\frac{\thirdL \secondL \secphaseinit}{\mineigen^2} \eta^2) \\&
                \le 1.01 \eta \lambda_1(t).
            \end{align*}

            \end{enumerate}
        \end{enumerate}
        
        Hence, we have shown that, $\norm{x_{\eta}(t) - \Phi(x_{\eta}(t))} \le \frac{1}{\lambda_M(t)} \norm{ \Tilde{x}_{\eta}(t) } \le \frac{1.01 \eta \secphaseinit \secondL}{\mineigen}$ for all time $t' \le t \le \overline{t}$.

        We complete the proof of \Cref{lem:phase0} with a similar argument as that for the quadratic model (see \Cref{lem:invariantj} and \Cref{lem:quadratic_norm_invariant}). The major difference from the quadratic model is that here the hessian changes over time, along with its eigenvectors and eigenvalues. Hence, we need to take care of the errors introduced in each step by the change of hessian.

        The high-level idea is to  divide the eigenvalues at each step $t$  into groups such that  eigenvalues in the same group are $O(\eta)$ close and eigenvalues from different groups are at least $2\eta$ far away from each other. Formally, we divide $[M]$ into disjoint subsets $S^{(t)}_1, \cdots, S^{(t)}_{p(t)}$ (with $1 \le p(t) \le  M$) such that  
        $$\forall k,\ell\in[p(t)], k\neq \ell \quad \min_{i \in S_k, j \in S_\ell} \abs{\lambda_i(t) - \lambda_{j}(t)} > \eta$$ and
                $$\forall k\in[p(t)], i,i+1\in S^{(t)}_k \quad \lambda_i(t) - \lambda_{i+1}(t) \le  \eta.$$
		For $S\subset [M]$, we denote by $P^{S}_{t}$ the projection matrix at time $t$ onto the subspace spanned by $\{v_i(t)\}_{i \in S}$.
        From \Cref{lem:boundmovementinPhi}, we have $\norm{\Phi(x_{\eta}(t+1)) - \Phi(x_{\eta}(t+1))} \le \secondPhi \eta^2$, which further implies, $\norm{ \nabla^2 L(\Phi(x_{\eta}(t+1))) - \nabla^2 L(\Phi(x_{\eta}(t))) } \le {O}( \thirdL \secondPhi \eta^2 ) $. 
        That implies, using \Cref{lem:perturb_eigenvalue}, 
        $\abs{ \lambda_j(t) - \lambda_j (t) } \le {O}( \thirdL \secondPhi \eta^2 )$ for any $j \in [M]$. Therefore, we can use \Cref{lem:DavisKahn} to have for any $\ell \in [p]$ $\norm{P^{S^{(t)}_\ell}_{t} - P^{S^{(t)}_\ell}_{t+1}} \le {O}( \thirdL \secondPhi \eta  )$, since we have created the eigen subspaces such that the eigenvalue gap between any two distinct eigen subspaces is at least $0.5\eta$ in the desired interval.
        
        Thus for any $t'\le t \le \overline t-1$ and $k\in[p(t)]$, suppose $i\in S^{(t)}_k$ and $j = \min S^{(t)}_k$, we have that 
        \begin{align*}
        	 &\sqrt{ \sum_{h=i}^{M} \langle v_h(t+1), \Tilde{x}_{\eta}(t+1) \rangle^2 }  \\
        	\le & \sqrt{ \sum_{h=j}^{M} \langle v_h(t+1), \Tilde{x}_{\eta}(t+1) \rangle^2 }   
        	=  \sqrt{ \sum_{\ell=k}^{p(t)}\norm{P_{t+1}^{S^{(t)}_\ell}\Tilde{x}_{\eta}(t+1) }^2 }  \\
        	\le &\sqrt{ \sum_{\ell=k}^{p(t)}\norm{P_{t}^{S^{(t)}_\ell}\Tilde{x}_{\eta}(t+1) }^2 } + O(\eta^2) 
        	=  \sqrt{ \sum_{h=j}^{M} \langle v_h(t), \Tilde{x}_{\eta}(t+1) \rangle^2 } + O(\eta^2)  
        \end{align*}
        and  $ \eta\lambda_i(t+1)	\ge \eta \lambda_i(t) -O(\eta^2) \ge \eta\lambda_j(t) - O(\eta^2). $

        Therefore, we have that 
        \begin{align*}
        	\frac{1}{\eta\lambda_i(t+1)}\sqrt{ \sum_{h=i}^{M} \langle v_h(t+1), \Tilde{x}_{\eta}(t+1) \rangle^2 }
        	\le \frac{1}{\eta\lambda_j(t)}\sqrt{ \sum_{h=j}^{M} \langle v_h(t), \Tilde{x}_{\eta}(t+1) \rangle^2 }  + O(\eta) 
        \end{align*}

        Next we will use the results from the quadratic case to upper bound $\sqrt{ \sum_{h=j}^{M} \langle v_h(t), \Tilde{x}_{\eta}(t+1) \rangle^2 }$ using $\sqrt{ \sum_{h=j}^{M} \langle v_h(t), \Tilde{x}_{\eta}(t) \rangle^2 }$.
        For all $1 \le j \le M$, we consider the following two cases for any time $t' + 1 \le t \le \overline t$:
        \begin{enumerate}
            \item If $\sqrt{ \sum_{h=j}^{M} \langle v_h(t), \Tilde{x}_{\eta}(t) \rangle^2 } > \eta \lambda_j(t)$, then we can apply \Cref{lem:quadratic_prep_decrease} on \eqref{eq:noisynormalizedGD} to show that
            \begingroup
            \allowdisplaybreaks
            \begin{align*}
                &\sqrt{ \sum_{h=j}^{M} \langle v_h(t), \Tilde{x}_{\eta}(t+1) \rangle^2 }   \\
                &\le  \left( 1 - \frac{ \eta \lambda_M(t) }{   \norm{\Tilde{x}_{\eta}(t)} }  \right)   \sqrt{ \sum_{h=j}^{M} \langle v_h(t), \Tilde{x}_{\eta}(t) \rangle^2 }  + {O}( \thirdL \secondPhi  \eta^2  ) + {O}(\frac{\thirdL \secondL}{\mineigen} \eta  \norm{x_{\eta}(t) - \Phi(x_{\eta}(t))}) \\&
                \le \left( 1 - \frac{  \mineigen }{ 2  \secondL \secphaseinit  }  \right)   \sqrt{ \sum_{h=j}^{M} \langle v_h(t), \Tilde{x}_{\eta}(t) \rangle^2 }  + {O}( \thirdL \secondPhi  \eta^2  ) + {O}(\frac{\thirdL \secondL^2 \secphaseinit}{\mineigen^2} \eta^2  ). 
            \end{align*}
            \endgroup

            \item If $\sqrt{ \sum_{h=j}^{M} \langle v_h(t), \Tilde{x}_{\eta}(t) \rangle^2 } \le \eta \lambda_j(t)$, then we can apply \Cref{lem:quadratic_norm_invariant} on \eqref{eq:noisynormalizedGD} to show that
            \begin{align*}
                \sqrt{ \sum_{h=j}^{M} \langle v_h(t), \Tilde{x}_{\eta}(t+1) \rangle^2 }   &\le  \eta \lambda_j(t)  + {O}( \thirdL \secondPhi  \eta^2  ) + {O}(\frac{\thirdL \secondL}{\mineigen} \eta  \norm{x_{\eta}(t) - \Phi(x_{\eta}(t))}) \\&
                \le \eta \lambda_j(t)  + {O}( \thirdL \secondPhi  \eta^2  ) + {O}(\frac{\thirdL \secondL^2 \secphaseinit }{\mineigen^2} \eta^2 ).
            \end{align*}
        \end{enumerate}         
        Thus we conclude that 
        \begin{align*}
        	&\max_{i\in [M]} \frac{1}{\eta\lambda_i(t+1)}\sqrt{ \sum_{h=i}^{M} \langle v_h(t+1), \Tilde{x}_{\eta}(t+1) \rangle^2 }\\
        	\le & \max\left\{1, (1-\frac{\mineigen}{2\secondL\secphaseinit}) \cdot \frac{1}{ \eta\lambda_j(t)}\max_{j\in [M]}  \sqrt{ \sum_{h=j}^{M} \langle v_h(t), \Tilde{x}_{\eta}(t) \rangle^2 }  \right\}+ O(\eta),
        \end{align*}
       and therefore following the same proof of quadratic case~\Cref{lem:invariantj}, for $\overline t\ge t' + {\Omega}( \frac{\secphaseinit \secondL}{\mineigen} \log \frac{\secondL \secphaseinit}{\mineigen} )$,  it holds that $\forall j\in [M]$, 
        $ \sqrt{ \sum_{i=j}^{M} \langle v_i(\bar{t}), \tilde{x}(\bar{t}) \rangle^2 }   
                \le \eta \lambda_j(\bar{t}) +   {O}( \eta^2).$
\end{proof}

\subsection{Properties of the condition in \Cref{eq:condition_perturbed_proof}}

By \Cref{lem:phase0}, the following condition will continue to hold true for all $1 \le j \le M$ before $\tilde{x}_\eta(t)$ leaves $\flowtraj^\saferadius$:
\begin{align}  \label{eq:condition_perturbed}
         \sqrt{ \sum_{i=j}^{M} \langle v_i(t), \Tilde{x}_{\eta}(t) \rangle^2 } \le \lambda_j(t) \eta +  {O}(\eta^2),
\end{align}
where $\Tilde{x}_{\eta}(t) = \nabla^2 L( \Phi(x_{\eta}(t)) ) (x_{\eta}(t) - \Phi(x_{\eta}(t))).$
We will call the above condition as the alignment condition from now onwards. 

From the alignment condition~ \eqref{eq:condition_perturbed}, we can derive the following property that continues to hold true throughout the trajectory, once the condition is satisfied:
\begin{lemma}\label{lem:dropbyhalf_general}
    There is some constant $\normerrterm>0$, such that if the
    condition~\eqref{eq:condition_perturbed} holds true and  $\norm{ \Tilde{x}_{\eta}(t) } > \frac{\eta \lambda_1(t)}{2}$, we have:
    \begin{align*}
        \norm{\Tilde{x}_{\eta}(t+1)} \le \frac{\eta \lambda_1(t)}{2} + \normerrterm \eta^2.
    \end{align*}
\end{lemma}

The proof follows from applying \Cref{cor:iteratenorm_drops} using the alignment condition~\eqref{eq:condition_perturbed}. Hence, the iterate $\Tilde{x}_{\eta}(t)$ can't stay at norm larger than $0.5 \eta \lambda_1(t) + \normerrterm \eta^2$ for time larger than $1$.

Another useful lemma is to about the change of the angle  between $\tilde x_\eta(t)$ and  the top eigenvector  when $\norm{ \Tilde{x}_{\eta}(t) } \le \frac{\eta \lambda_1(t)}{2} +  \normerrterm \eta^2$, which is a noisy version of \Cref{lem:behaviornormhalf_theta_decrease} for a quadratic model.
\begin{lemma}\label{lem:behaviornormhalf_general}
    Consider any time $t$ such that  $\norm{ \Tilde{x}_{\eta}(t) } \le \frac{\eta \lambda_1(t)}{2} + \normerrterm \eta^2$, and the  condition~\eqref{eq:condition_perturbed} holds true, then we have that 
    \begin{align*}
        \tan \theta_{t+1} \le \left( 1 - \frac{ \min(\eigengap, 2 \mineigen) } { \secondL }  \right)\tan \theta_t +O(\frac{\eta}{\norm{\tilde x_\eta(t) }}).
    \end{align*}
and that 
    \begin{align*}
        \tan \theta_{t+2} \le \frac{\eta \lambda_1}{\norm{\tilde x_\eta(t)}}\tan \theta_t + O(\frac{\eta^2}{\norm{\tilde x_\eta(t) }}).
    \end{align*}
\end{lemma}

\begin{corollary}\label{cor:behaviornormhalf_general}
    If for some $1 \le k \le M$, $\norm{ \Tilde{x}_{\eta}(t) } \le \frac{\eta \lambda_k(t)}{2} + \normerrterm \eta^2$ and  condition ~\eqref{eq:condition_perturbed} holds true, the following must hold true:
    \begin{align*}
        &\abs{v_k(t+1)^{\top} \Tilde{x}_{\eta}(t+1) } \ge \abs{v_k(t)^{\top} \Tilde{x}_{\eta}(t) } - {O}(  \eta^2). 
    \end{align*}
\end{corollary}
The proof follows from using the noisy quadratic update for Normalized GD in \Cref{lem:appr_gradient_norm}  (\Cref{eq:noisynormaliedgdupdate_primal}) and the behavior in a quadratic model along the non-top eigenvectors  in \Cref{lem:behaviornormhalf}.

\section{Phase II, Omitted Proofs of the Main Lemmas}\label{appsec:phaseIIconv_proofs}
The main lemma in this section is \Cref{lem:avgGt} in \Cref{appsec:avgangle}, which says the sum of the angles across the entire trajectory in any interval $[0, t_2]$ with $t_2 = \Omega(1/\eta^2)$, is at most $O(\eta t_2)$. Before proving the main lemma, we will first recap and introduce some notations that will be used. 

In Phase II, we start from a point $x_{\eta} (0)$, such that (1) $\norm{ x_\eta(0)-\Phi(\xinit)} \le O(\eta)$, 
   (2) $\max_{ j\in[D]} R_j(x_\eta(t))\le O(\eta^2)$,
    and additionally (3) $\abs{ \langle v_1(x_\eta(0)), x_\eta(0) - \Phi(x_\eta(0)) \rangle } =\Omega(\eta)$.

\newcommand{\gmax}{g_{\mathrm{max}}}
\newcommand{\gmin}{g_{\mathrm{min}}}

Formally, recall our notation on $\theta_t$ as $\theta_t = \arctan \frac{\norm{P_{t, \Gamma}^{(2:M)} \Tilde{x}_{\eta}(t) }}{ \abs{ \langle v_1(t),  \Tilde{x}_{\eta}(t) \rangle }  },$ with our notation of $\Tilde{x}_{\eta} (t)$ as $\nabla^2 L(\Phi(x_{\eta} (t))) (x_{\eta} (t) - \Phi(x_{\eta} (t)))$. Moreover, recall the definition of the function $g_t: \mathbb{R} \to \mathbb{R}$ as
$$g_t(\lambda) = \frac{\lambda_1(t)}{2} \left( 1 - \sqrt{ 1 - 2 \frac{\lambda}{\lambda_1(t)} \left(1 - \frac{\lambda}{\lambda_1(t)}\right) } \right).$$ For convenience, we also define $G_t := |\inner{v_1(t)}{\tilde x(t)}|$, $\gmax(t) := \max_{k\in[M]}g_t(\lambda_k(t))$ and $\gmin(t) := \min_{k\in[M]}g_t(\lambda_k(t))$.

The condition~\eqref{eq:condition_perturbed} that was shown to hold true in Phase II is: 
\begin{align*} 
         \sqrt{ \sum_{i=j}^{M} \langle v_i(t), \Tilde{x}_{\eta} (t) \rangle^2 } \le \lambda_j(t) \eta +   O(\eta^2).
\end{align*}
Further, we had proved in \Cref{lem:dropbyhalf_general} that:
    \begin{align*}
       \norm{ \Tilde{x}_{\eta} (t) } > \frac{\eta \lambda_1(t)}{2} \implies  \norm{\Tilde{x}_{\eta} (t+1)} \le \frac{\eta \lambda_1(t+1)}{2} + \normerrterm \eta^2.
    \end{align*}
Thus, the iterate can't stay greater than $\frac{\eta \lambda_1(t)}{2} + \normerrterm\eta^2$  for more than $1$ step. This lemma allows us to divice all time steps into groups of length $1$ and $2$.

\subsection{Dividing Time Steps Into Cycles of Length 1 or 2}\label{sec:property_div}

\begin{algorithm}[tb]
   \caption{Grouping into 1-cycles and 2-cycles}
   \label{alg:seq_create}
\begin{algorithmic}
   \STATE {\bfseries Input: } Interval $[ \Tilde{t}, \overline{t} ]$, Iterates of \Cref{alg:PNGD}: $ \{ \Tilde{x}_{\eta}(t) \}_{t \in [ \Tilde{t}, \overline{t} ]}$ where $\overline{x_\eta(t)x_\eta(t+1)}\subseteq Y_{\saferadius}$, Top eigenvalue  $\lambda_1(t) = \nabla^2 L(\Phi(x_{\eta}(t)))$.
   \STATE {\bfseries Requires: } $\norm{ \Tilde{x}_{\eta}(\Tilde{t}) } \le 0.5 \lambda_1(\Tilde{t}) \eta + \normerrterm \eta^2.$ 
   \STATE {\bfseries Initialize:} $\cycleonestart, \cycletwostart, \cycletwoother \gets \emptyset$, $t \gets \Tilde{t}$.
   \WHILE{$t \le \overline{t}$}
   \IF{$\norm{ \Tilde{x}_{\eta}(t + 2) } > 0.5 \lambda_1(t+2) \eta + \normerrterm \eta^2$}
   \STATE $\cycleonestart \gets \cycleonestart \cup \{t\}$
   \STATE $t \gets t + 1$
   \ELSIF{$\norm{ \Tilde{x}_{\eta}(t + 2) } \le 0.5 \lambda_1(t+2) \eta + \normerrterm \eta^2$ }
   \STATE $\cycletwostart \gets \cycletwostart \cup \{t\}$
   \STATE $\cycletwoother \gets \cycletwoother \cup \{t + 1\}$
   \STATE $t \gets t + 2$
   \ENDIF
   \ENDWHILE
   \STATE {\bfseries Return: } $\cycleonestart, \cycletwostart, \cycletwoother$
\end{algorithmic}
\end{algorithm}

    The properties of the three sets $\cycleonestart, \cycletwostart, \cycletwoother$ in an interval $(\Tilde{t}, \overline{t})$ given directly by the algorithm in \Cref{alg:seq_create} include:
    \begin{enumerate} 
        \item $\forall t \in \cycleonestart$, $\norm{\Tilde{x}_{\eta}(t+2)}\ge 0.5 \lambda_1(t+2) \eta + \normerrterm \eta^2$.
        
        \item $\forall t \in \cycletwostart$, $\norm{\Tilde{x}_{\eta}(t+2)}\le 0.5 \lambda_1(t+2) \eta + \normerrterm \eta^2$.
        
        \item $\forall t, t \in \cycletwostart \Longleftrightarrow t+1\in \cycletwoother$. 
        \item $\cycleonestart\cup\cycletwostart\cup\cycletwoother = [\tilde t,\overline t]$, and the intersection between each pair of them is empty. 
    \end{enumerate}
    
    We also have the following lemmas, which is less direct:
    \begin{lemma}\label{lem:nonconsecutiveS0}
        For any step $t$ in $\cycleonestart$,  $t-1,t+2\in\cycletwoother$ and $t-2,t+1\in\cycletwostart$
    \end{lemma}
    
    \begin{proof}[Proof of \Cref{lem:nonconsecutiveS0}]
Suppose $t$ is in $\cycleonestart$, and therefore $\norm{ \Tilde{x}_{\eta}(t+2) } \ge 0.5 \lambda_1(t+2) \eta + \normerrterm \eta^2$. By \Cref{lem:dropbyhalf_general} we know that  $\norm{ \Tilde{x}_{\eta}(t+3) } < 0.5 \lambda_1(t+3) \eta + \normerrterm \eta^2$,  thus $t+1\in \cycletwostart$ and  $t+2\in \cycletwoother$. Applying similar argument on $t-1$, we know if  $t-1$  is in $\cycleonestart$, $t$ cannot be in $\cycleonestart$. Meanwhile, $t-1$ cannot be in $\cycletwostart$, which would imply $t\in\cycletwoother$. Thus $t-1$ must be in $\cycletwoother$ and therefore $t-2$ is in $\cycletwostart$.
    \end{proof}
    
    \begin{lemma}\label{lem:cycle_start_has_small_norm}
		$\forall t\in \cycleonestart \cup \cycletwostart$,  $\norm{\Tilde{x}_{\eta}(t)}\le 0.5 \lambda_1(t) \eta + \normerrterm \eta^2$.
    \end{lemma}
    \begin{proof}[Proof of \Cref{lem:cycle_start_has_small_norm}]
    	By \Cref{lem:nonconsecutiveS0}, we know if $t\in \cycleonestart$, then $t-2\in\cycletwostart$ and thus by the definition of $\cycletwostart$, we have $\norm{\Tilde{x}_{\eta}(t)}\le 0.5 \lambda_1(t) \eta + \normerrterm \eta^2$.  
    	
    	If $t\in \cycletwostart$, then we consider the three possibilities for $t-2$. If $t-2\in \cycletwostart$, then the proof is done by definition of $\cycletwostart$. If $t-2\in \cycleonestart$, then by \Cref{lem:nonconsecutiveS0}, we have $t\in\cycletwoother$, contradiction! Thus this case is not possible. If $t-2\in \cycletwoother$, then $t-1$ cannot be in $\cycletwostart$ as $t$ is not in $\cycletwoother$. Thus $t-2$ must be in $\cycleonestart$ and \Cref{lem:dropbyhalf_general} implies that $\norm{\Tilde{x}_{\eta}(t)}\le 0.5 \lambda_1(t) \eta + \normerrterm \eta^2$. 
    \end{proof}

    \begin{lemma}\label{lem:angledropS0S1}
        For any step $t$ in $\cycleonestart$ and $\cycletwostart$, 
        \begin{enumerate}
        	\item $\tan \theta_{t+1} \le \left( 1 - \frac{ \min (\eigengap, 2\mineigen)} { \secondL }\right) \tan \theta_t + O( \frac{ \eta^2}{G_t})$
        	\item $\tan \theta_{t+2} \le \frac{\eta \lambda_1}{G_t}\tan \theta_t + O(\frac{\eta^2}{G_t})$
        	\item If $G_t \ge 1.02 \gmax (t) \eta$, $ \tan \theta_{t+2} \le \left( 1 -  \min \left(0.01, \min\limits_{i \le M} \frac{2\lambda_i(t)}{\lambda_1(t)} ( 1-  \frac{\lambda_i(t)}{\lambda_1(t)}  ) \right) \right)\tan \theta_{t} + O(\eta).$
        \end{enumerate}
    \end{lemma}
    
    \begin{proof}[Proof of \Cref{lem:angledropS0S1}]
    The proof follows from using the noisy update rule for Normalized GD, as derived in \Cref{eq:noisynormalizedGD}, which says that the Normalized GD update is very close to the update in a quadratic model with an additional $O(\eta^2)$ error. Using the property of $\cycleonestart$ and $\cycletwostart$ outlined above, we have the norm of $\Tilde{x}_{\eta}(t)$ at most $0.5 \lambda_1(t) \eta + \normerrterm \eta^2$. The result then follows from using \Cref{lem:behaviornormhalf_theta_decrease} and \Cref{lem:quadratic_tan_two_step}, that computes the convergence rate towards the top eigenvector for a quadratic model. (The first two properties are stated in \Cref{lem:behaviornormhalf_general}).
\end{proof}
    
As a direct consequence of \Cref{lem:angledropS0S1}, we have the following lemma:
\begin{lemma}\label{lem:angle_drop_slow}
	Given any $t$ with $\theta_t = \Omega(1)$, let $\tilde t=\max \cycletwostart\cap \{\tilde t\mid \tilde t\le t\}$. If $G_{\tilde t}\ge \Omega(\eta)$, then $\theta_{\tilde t} = \Omega(1)$. 
\end{lemma}
\begin{proof}[Proof of \Cref{lem:angle_drop_slow}]
	The claim is clearly true if $t\in \cycletwostart$. If $t\in \cycleonestart$, then \Cref{lem:nonconsecutiveS0} shows that $t-1\in \cycletwoother, t-2 \in \cycletwostart$ and thus $\tilde t =t-2$. The claim is true because of the second property of \Cref{lem:angledropS0S1}. If $t\in \cycletwoother$, then $\tilde  t = t-1\in \cycletwostart$ and the proof is completed by applying the first property of \Cref{lem:angledropS0S1}.
\end{proof}

\subsection{Time Average of Angles Against Top Eigenspace}\label{appsec:avgangle}
\begin{restatable}[Average of the Angles]{lemma}{avgGt}\label{lem:avgGt}
    For any $T_2 > 0$ for which solution of \Cref{eq:log_limiting_flow} exists, consider an interval $[0, t_2]$, with $\Omega(1/\eta^2) \le t_2 \le  \lfloor T_2 / \eta^2 \rfloor$.
    Suppose \Cref{alg:PNGD} is run with learning rate $\eta$ for $t_2$ steps, starting from a point $x_{\eta}(0)$ that satisfies (1) $\max_{ j\in[D]} R_j(x_\eta(0))\le O(\eta^2)$, and (2) $ G_0 := \abs{\inner{v_1(0)}{\tilde x(0)}} \ge \init\eta, \norm{ \Tilde{x}_{\eta}(0)} \le \frac{\eta \lambda_1(0)}{2} + \normerrterm \eta^2$ for some constant $\init$ independent of $\eta$. 
   	The following holds true with probability at least $1 - \eta^{10}$:
	\begin{align*}
		  \frac{1}{t_2}\sum_{\ell = 0}^{t_2} \theta_{\ell} \le O\left( \eta  \right),
	\end{align*}
	provided $\eta $ is set sufficiently small, and for all time $0 \le t \le t_2-1$, $\overline{x_{\eta}( t ) x_{\eta}(t+1)} \subset \flowtraj^\saferadius$.
\end{restatable}

\begin{proof}
 We split the entire interval $[0,t_2)$ into small trunks in the following way, $0 = \Tilde{t}_0 < \Tilde{t}_1 < \Tilde{t}_2 \ldots \Tilde{t}_{\ell} = t_2$ with $\tilde t_\ell$ denoting the starting step of each trunk. Each $\Tilde{t}_i$ is defined from $\Tilde{t}_{i-1}$ for $i > 0$. The behavior of each trunk depends on the magnitude of the iterate along the top eigenvector of hessian.  We classify the trunks on the basis of $3$ possibilities: Consider a general $\Tilde{t}_i$,
\begin{enumerate}
    \item[$A$. ] If $G_{\Tilde{t}_i} \ge 1.02\gmax(\tilde t_i)$, then we define $\Tilde{t}_{i+1}$ as 
    $$
    \Tilde{t}_{i+1} = \min_{t > \Tilde{t}_i} \{ t \mid G_t \le 1.01 \gmax( t), \norm{\Tilde{x}(t)} \le 0.5 \lambda_1(t) \eta + \normerrterm \eta^2 \}.
    $$
    
    \item[$B$. ] If  $0.98\gmin(\tilde t_i) \le G_{\Tilde{t}_i} \le 1.02\gmax(\tilde t_i)$ then we define $\Tilde{t}_{i+1}$ as 
    $$
    \Tilde{t}_{i+1} = \min_{t > \Tilde{t}_i} \{ t \mid \left(0.97\gmin(t) \ge G_{t} \vee G_{t}  > 1.03\gmax(t) \right)\wedge \norm{\Tilde{x}(t)} \le 0.5 \lambda_1(t) \eta + \normerrterm \eta^2 \}.
    $$
    
    \item[$C$. ] If $G_{\Tilde{t}_i} \le 0.98\gmax(\tilde t_i)$, then we define $\Tilde{t}_{i+1}$ as 
    $$
    \Tilde{t}_{i+1} = \min_{t > \Tilde{t}_i} \{ t \mid G_t  \ge 0.99\gmax( t), \norm{\Tilde{x}(t)} \le 0.5 \lambda_1(t) \eta + \normerrterm \eta^2 \}
    $$
\end{enumerate}
    We analyze the behavior of a general $\Tilde{t}_i$ when it falls in any of the above cases:
  
    \paragraph{Case (A).} First of all, since $G_{t} \ge 1.02\gmax(t)$ for all $\Tilde{t}_i \le t < \Tilde{t}_{i+1}$ we can show from \Cref{lem:angledropS0S1} that the angle with the top eigenvector quickly drops to $O(\eta)$ in at most $O(\ln 1/\eta)$ time-steps. Moreover, the iterate's magnitude can only drop along the top eigenvector when the angle with the top eigenvector is smaller than $O(  \eta )$, and the drop is at most  $O(  \eta^3 )$ (\Cref{lem:BehaviorGt}). Thus, during alignment of the iterate to the top eigenvector, $G_t$ never drops. Moreover, after the alignment, it takes $\Omega(\frac{1}{\eta^2})$ steps for the iterate's magnitude along the top eignvector to drop below $ 1.01 \max_{k \in [M]} g_{t}(\lambda_k(t))$. Hence,
        \begin{align*}
            \abs{\Tilde{t}_{i+1} - \Tilde{t}_i} \ge \Omega\left( \frac{1}{ \eta^2} \right), \quad \sum_{t=\Tilde{t}_i}^{\Tilde{t}_{i+1}} \theta_t \le  O \left(  ( \Tilde{t}_{i+1} -\Tilde{t}_i )   \eta + \log1/\eta \right) = O \left(  ( \Tilde{t}_{i+1} -\Tilde{t}_i )   \eta \right) .
        \end{align*}
        
        After $G_t$ drops below $1.01\gmax( t)$, it moves to case $B(1).$
        
  \paragraph{Case (B).}  From \Cref{lem:stuck_region} we have that the sum of angle over this time is 
       $$
        \sum_{t=\Tilde{t}_i}^{\Tilde{t}_{i+1}} \theta_t =  O \left( \sqrt{\tilde t_{i+1} - \tilde t_{i}} + \eta ( \Tilde{t}_{i+1} - \Tilde{t}_i ) \right).
       $$       
       
   \paragraph{Case (C).} We claim $G_t$ will become larger than $0.99\gmax(t)$ in $O(\eta^{-0.1})$ steps with probability at least $1-O(\eta^{12})$, because of the $\eta^{100}$ perturbation added per $\Theta(\eta^{-\freq})$ steps.
   If for some $t>\tilde t_i$, $\theta_{t}\le \Omega(\eta)$, then by \Cref{lem:Behavior_explode2nd}, we know that in $O(\log 1/\eta)$ steps after the perurbation, with probability at least $1-O(\eta^{12})$, we have $\theta_{ t'}\ge \Omega(\eta)$ for some $t'\le t + O(\log 1/\eta)$. And thus we can apply \Cref{lem:Behavior_explode2nd_easycase} and $\theta_{t'} = \Omega(1)$. By \Cref{lem:angle_drop_slow}, we know the $\theta_{\tilde t} = \Omega(1)$ as well, where $\tilde t$ is the largest step in $\cycletwostart$ yet smaller than $\tilde t$. Then by \Cref{lem:BehaviorGt},  $G_{\tilde t+2}\ge G_{\tilde t} +\Omega(\eta)$. If $\tilde t+2\notin \cycletwostart$, then $\tilde t+2 $ must be in $\cycleonestart$ and $\tilde t+3 \in \cycletwostart$. Again by \Cref{lem:BehaviorGt_one_setp}, we have $G_{\tilde t+3}\ge G_{\tilde t+2} - O(\eta^2)\ge G_{\tilde t} +\Omega(\eta)$. Thus $G_{t}$ will increase $\Omega(\eta)$ every $O(\log 1/\eta)$ steps among those steps in $\cycletwostart$ (among the steps in $\cycletwostart$ and $\cycleonestart$, $G_t$ decreases at most $O(\eta^3)$ by \Cref{lem:BehaviorGt}. Thus $\tilde t_{i+1}\le \tilde t_{i} +O(\log 1/\eta) + O(\eta^{-0.1}) = \tilde t_{i} + O(\eta^{-0.1})$. Thus, $\sum_{t=\Tilde{t}_i}^{\Tilde{t}_{i+1}} \theta_t =  O( \eta^{-\freq})$.
        
    Now it remains to upper bound the number of occurrence of (A),(B) and (C). Since our goal is to show average angle is $O(\eta)$, which is equal to the average angle in case (A), so the number of occurrence of case (A) doesn't matter. For case (B), if it is followed by case (A), then there is an $\Omega(1/\eta^2)$ gap before next occurrence of (B). If (B) is followed by case (C), then by \Cref{lem:BehaviorGt}, it takes at least $\Omega(1/\eta^2)$ steps to escape from (B). Thus we can have $O(1)$ occurrence of case (B). For the same reason, there could be at most $O(1)$ occurrence of case (C). 
        
    All in all, with probability at least $1- O(\eta^{12}\cdot \eta^{12}) = 1-O(\eta^{10})$, we must have
    \begin{align*}
        \sum_{t=0}^{t_2} \theta_t
         &\le  O \left( (\tilde t_\ell - \tilde t_0) \eta\right) + O(1)\cdot O( \eta^{-\freq} ) +  O(\left( \sum_{i: \textrm{case (B)}} \sqrt{\tilde t_{i+1} - \tilde t_i}\right)\\
         & \le O\left((\tilde t_\ell - \tilde t_0) \eta    \right) +  O( \eta^{-\freq}  ) + O\left( \sqrt{\sum_{i: \textrm{case (B)}} (\tilde t_{i+1} - \tilde t_i) \sum_{i: \textrm{case (B)}}1}\right) \\
         & = O\left(t_2 \eta    \right)  + O\left( t_2\right) \\
         & = O(t_2 \eta)
    \end{align*} 
    where we use $t_2 \ge \Omega(\frac{1}{\eta^2})$ in the last step and and the number of occurrence of case (B) is $O(1)$ in the second to the last step.
\end{proof}

\begin{lemma}\label{lem:stuck_region}
    Consider the setting of \Cref{lem:avgGt}. Consider any time interval $[\overline{t},t']$, where $\overline{t} \le \ell < t'$, $\overline{x_{\eta}( \ell ) x_{\eta}(\ell + 1)} \subset \flowtraj^\saferadius$ and  $  \Omega(\eta) \le G_l:= \abs{ \langle v_1(t), \tilde x_\eta(t) \rangle }\le \frac{\lambda_1(\ell)\eta}{2}-\Omega(\eta)$, we have that
    \begin{align*}
\sum_{t\in[\overline t,t']} \theta_t = \sum_{t \in \cycleonestart\cup \cycletwostart \cup \cycletwoother} \theta_t   \le  O(\sqrt{ t'-\overline t}  + (t'-\overline t)\eta).
    \end{align*}
\end{lemma}

\begin{proof}[Proof of \Cref{lem:stuck_region}]
The noisy update rule for Normalized GD, as derived in \Cref{lem:appr_gradient_norm}, which says that the Normalized GD update is very close to the update in a quadratic model with an additional $O(\eta^2)$ error. Keeping this in mind, we then divide our trajectory in the interval $(\overline{t}, t')$ as per \Cref{alg:seq_create} into three subsets $\cycleonestart, \cycletwostart, \cycletwoother$. (Please see \Cref{sec:property_div} for a summary on the properties of these 3 sets.)
    
Consider any $t \in \cycletwostart$.
    Using the behavior of $G_t$ from \Cref{lem:BehaviorGt}, we can show that in each of the time-frames, $G_{t+2} \ge (1+\Omega(\sin ^2 \theta_t))G_t - O(\eta^2(\eta+ \eta_t))\ge G_t + \Omega(\theta_t^2\eta) - O(\eta^2(\eta+\theta_t))$. 
    
    Next we want to telescope over $G_{t+2}-G_t$ to get an upper bound for $\sum_{t\in \cycletwostart}\theta_t$. If $t+2$ is also in $\cycletwostart$ then it's fine. If $t+2\in\cycleonestart$, then $t+3\in\cycletwostart$ by \Cref{lem:nonconsecutiveS0} and we proceed in the following two cases. 
    \begin{itemize}
		\item If $\theta_{t+2} \le C $ for some sufficiently small constant $C$, since $G_{t+2}\le \lambda_1(t+2)\eta/2 - \Omega(\eta)$, we have $\norm{\tilde x_\eta(t+2)}\le \frac{G_{t+2}}{\cos\theta_{t+2}} = \lambda_1(t+2)\eta/2 - \Omega(\eta)$, and thus by \Cref{lem:BehaviorGt_one_setp}, we have $G_{t+3}\le G_{t+2}$ and therefore, $G_{t+3} \ge G_t + \Omega(\theta_t^2\eta) - O(\eta^2(\eta+\theta_t))$. 
		\item If $\theta_{t+2} \ge C$, then by \Cref{lem:behaviornormhalf_general}, we have $\theta_{t} =\Omega(1)$, thus $G_{t+2}\ge G_t +\Omega(\eta)$ by \Cref{lem:BehaviorGt}. Again by \Cref{lem:BehaviorGt_one_setp}, we have $G_{t+3}\ge G_{t+2} - O(\eta^2)$. Thus again we conclude $G_{t+3} \ge G_t +\Omega(\eta) \ge G_t + \Omega(\theta_t^2\eta) - O(\eta^2(\eta+\theta_t))$, since $\theta_t$ is always $O(1)$.
    \end{itemize}

    Since total increase in $G_t$ during this interval can is most $O(\eta)$,  we conclude that $\sum_{t\in\cycletwostart} \theta_t^2 = O(1) + \eta\sum_{t\in\cycletwostart}(\eta+\theta_t)$ and thus it holds that
    \begin{align*}
        \sum_{t \in \cycletwostart} \theta_t   \le  \sqrt{(t'-\overline t) \sum_{t \in \cycletwostart}\theta_t^2}
        \le &\sqrt{ t'-\overline t} \cdot  O\left(\sqrt{1 + \eta \sum_{t \in \cycletwostart} \theta_t +\eta^2 (t'-\overline t)}\right)\\
        \le & O(\sqrt{ t'-\overline t}  + (t'-\overline t)\eta)
    \end{align*}

	Moreover, by \Cref{lem:angledropS0S1}, we must have $\theta_{t} < \theta_{t-1} + O(\eta)$ for any time $t \in \cycletwoother$, and  $t-1$ must be in $ \cycletwostart$. By \Cref{lem:BehaviorGt_one_setp}, we have $\theta_{t}\le \Omega(\theta_{t-2})$ for any $t\in \cycleonestart$ and $t-2$ must be in $\cycletwoother$. That implies,
    \begin{align*}
         \sum_{t\in[\overline t,t']} = \sum_{t \in \cycleonestart\cup \cycletwostart \cup \cycletwoother} \theta_t   \le  O(\sqrt{ t'-\overline t}  + (t'-\overline t)\eta),
    \end{align*}
    which completes the proof.
\end{proof}

\begin{lemma}\label{lem:Behavior_explode2nd_easycase}
    Consider any coordinate $2 \le k \le M$. For any constants  $0 < \init$, there is some constant $\restrict>0$ such that for any time step $t$ where $x_{\eta} (\tilde t)$ is in $\flowtraj^\saferadius$,  $G_{\tilde t} \ge \init \eta$, condition \eqref{eq:condition_perturbed} holds and $\abs{\inner{v_k(\tilde  t)}{\tilde x(\tilde  t)}}\ge\restrict \eta^2$, then there is some time $ \overline{t} \le t + O\left(\ln 1/\eta \right)$ such that if for  all time $\tilde t \le t' < \overline{t}$, $\overline{x_{\eta}(t') x_{\eta}(t' + 1) } \subset \flowtraj^\saferadius$, then condition \eqref{eq:condition_perturbed} holds at time $\overline t$ and at least one of the following two conditions hold:
     \begin{enumerate}
         \item $G_{\overline t} \ge 0.99 g_t(\lambda_k(\overline{t}))$.
         
         \item $\theta_{\overline{t}} \ge \Omega( 1).$         
     \end{enumerate}
\end{lemma}

\begin{proof}[Proof of \Cref{lem:Behavior_explode2nd_easycase}]
We will prove by contradiction. Suppose neither of the two condition happens, we will show $\theta_{ t}$ grows exponentially and thus the condition (2) must be false in $O(\log 1/\eta)$ steps.       

First of all, because $\theta_t = O(1)$  and $G_t \le \frac{\eta\lambda_1(t)}{2}$ whenever $\norm{\tilde x_\eta(t)}\le \frac{\eta\lambda_1(t)}{2} +\normerrterm\eta^2$, by \Cref{lem:BehaviorGt_one_setp}, we know if $\norm{\tilde x_\eta(t)}\le \frac{\eta\lambda_1(t)}{2} +\normerrterm\eta^2$, then $\norm{\tilde x_\eta(t+1)}> \frac{\eta\lambda_1(t+1)}{2} +\normerrterm\eta^2$. And thus $\norm{\tilde x_\eta(t+2)}\le \frac{\eta\lambda_1(t+2)}{2} +\normerrterm\eta^2$. In other words, $t\in \cycletwostart \cup \cycletwoother $ for all $t$.   Therefore, for $\tilde t + 2k\in \cycletwostart$ for all natural numbers $k$ with $\tilde  t+2k\le \overline t$. Moreover,  $G_{\tilde  t+2k} \ge G_{\overline t} - O(k\eta^2)=\Omega(\eta)$ by \Cref{lem:BehaviorGt} for $k \le O(\frac{1}{\eta})$.
    
         Now, we can use \Cref{eq:noisynormaliedgdupdate_primal} (\Cref{lem:appr_gradient_norm}) to show that the Normalized GD update is equivalent to update in quadratic model, up to an additional $O(\eta^2)$  error.
        \begin{align*}
            x_{\eta} (t+1) -  x_{\eta} (t) = - \eta  \frac{\nabla^2 L (\Phi(x_{\eta} (t))) [ x_{\eta} (t) - \Phi ( x_{\eta} (t) ) ] }{ \norm{ \nabla^2 L (\Phi(x_{\eta} (t))) [ x_{\eta} (t) - \Phi ( x_{\eta} (t) ) ]   } } + O(\eta^2).   
        \end{align*}
        
Similar to \Cref{lem:incr_in_other_coordinates}, consider the coordinate $k$, we have that 
        \begingroup
        \allowdisplaybreaks
        \begin{align}
            &\frac{ \abs{  \langle v_k(t),  \nabla^2 L (\Phi (x_{\eta} (t)) ) ( x_{\eta} (t+2) - \Phi( x_{\eta} (t) ) ) \rangle } }{ \abs{ \langle v_1(t), \nabla^2 L (\Phi (x_{\eta} (t)) ) ( x_{\eta} (t+2) - \Phi( x_{\eta} (t) ) )  \rangle } } \nonumber \\
            &= \left( 1 - \eta \frac{ \lambda_1(t) - \lambda_k(t)  }{ \lambda_1(t) \eta - \norm{ \nabla^2 L (\Phi (x_{\eta} (t)) ) ( x_{\eta} (t+1) - \Phi( x_{\eta} (t) ) )  } } \right) \nonumber \\
            & \cdot \left( 1 -  \eta \frac{ \lambda_1(t) - \lambda_k(t)  }{  \lambda_1(t) \eta - \norm{ \Tilde{x}_{\eta} (t)  } } \right) \frac{\abs{  \langle v_k(t),  \Tilde{x}_{\eta} (t) \rangle } }{ \abs{ \langle  v_1(t), \Tilde{x}_{\eta} (t)  \rangle } } - O(\eta)  \label{eq:incr_in_angle_gen_eq} \\&
            \ge \left( 1 + \frac{1}{100} \right)  \frac{\abs{  \langle v_k(t),  \Tilde{x}_{\eta} (t) \rangle } }{ \abs{ \langle  v_1(t), \Tilde{x}_{\eta} (t)  \rangle } } - O(\eta  ) \nonumber \\&
            \ge \left( 1 + \frac{1}{200} \right) \frac{\abs{  \langle v_k(t),  \Tilde{x}_{\eta} (t) \rangle } }{ \abs{ \langle  v_1(t), \Tilde{x}_{\eta} (t)  \rangle } }, \nonumber
        \end{align}
        \endgroup
        The third step follows from using the same argument as the one used for the quadratic update in \Cref{lem:incr_in_other_coordinates} and the assumption that $G_t\ge 0.99 g_t(\lambda_k(t))$. The final step holds true because we can pick $\restrict$ as a large enough constant and by assumption $\frac{\abs{  \langle v_k(t),  \Tilde{x}_{\eta} (t) \rangle } }{ \abs{ \langle  v_1(t), \Tilde{x}_{\eta} (t)  \rangle } } \ge \restrict\eta$.

         We then bound $\norm{v_k(\overline t) - v_k(\tilde t)}$  and $\norm{\Phi(x_{\eta} (\overline t) - \Phi(x_{\eta} (\tilde t))}$ by $O(  \eta^2 (\overline t-\tilde t) )$ using \Cref{lem:boundmovementinPhi}. Combining everything, we conclude that at least one of the two assumptions has to break for some $\overline t\le \tilde t+O(\log 1/\eta)$.        
\end{proof}

\begin{lemma}\label{lem:Behavior_angledrop2nd_easycase}
    Consider any coordinate $2 \le k \le M$. For any constants  $0 < \init$, suppose at time step $t$, $x_{\eta} (t)$ is in $\flowtraj^\saferadius$,  $\left(1.01\right) g_t(\lambda_k(t)) \eta \le G_t < 0.5  \eta\lambda_1(t)$ and condition \eqref{eq:condition_perturbed} holds, then there is some time $ \overline{t} \le t + O\left(\ln 1/\eta \right)$ such that if for  all time $t \le t' < \overline{t}$, $\overline{x_{\eta}(t') x_{\eta}(t' + 1) } \subset \flowtraj^\saferadius$, then the following two conditions hold:
$$\abs{ \langle v_k(\overline t), \tilde x_\eta(\overline t) \rangle } \le O( \eta^2).$$    
    and  $$\norm{\nabla^2 L(\Phi(x_{\eta}(\overline{t}))) ( x_{\eta} (\overline{t}) - \Phi ( x_{\eta} (\overline{t}) ) )} \le 0.5 \lambda_1(t) \eta + \normerrterm \eta^2.$$        
\end{lemma}

The proof of \Cref{lem:Behavior_angledrop2nd_easycase} is very similar to the proof of \Cref{lem:Behavior_explode2nd_easycase} and thus we omit the proof. The only difference will be that we need to use \Cref{lem:quadratic_tan_two_step} in place of \Cref{lem:incr_in_other_coordinates}, when we use the result for the quadratic model. 

\subsection{Dynamics in the Top Eigenspace}
Here, we will state two important lemmas that we used for the proof of \Cref{lem:avgGt}, which is about the behavior of the iterate along the top eigenvector. \Cref{lem:BehaviorGt} can be viewed as perturbed version for \Cref{lem:firscoordincr_ondrop} in the quadratic case, and  We assume in all the lemmas, that \Cref{eq:condition_perturbed} holds true for the time under consideration, which we showed in \Cref{lem:phase0}, and also that we start Phase II from a point where the alignment along the top eigenvector is non negligible. 

The following lemmas give the properties of dynamics in the top eigenspace in Phase II for one-step and two-step updates respectively. Recall we use $G_t$ to denote the quantity $\abs{ \inner{v_1(t)}{ \tilde x(t)} }$.

\begin{lemma}[Behavior along the top eigenvector, one step]\label{lem:BehaviorGt_one_setp}
For sufficiently small $\eta$, 
consider any time $t$, such that $x_{\eta}(t) \in \flowtraj^\saferadius$,  $\norm{  \Tilde{x}_{\eta} (t) } \le \frac{1}{2} \eta 
 \lambda_1(t) + \normerrterm \eta^2$ and $G_t\ge \Omega(\eta)$ holds true , the following holds:
        \begin{align*}
            G_{t+1} &\ge  \left(\frac{\lambda_1(t)}{\norm{\tilde x_i(t)}}-1\right)G_t -O(\eta^2)
            \end{align*}
        provided that $\overline{x_{\eta}(t) x_{\eta}(t+1)}, \overline{x_{\eta}(t+1) x_{\eta}(t+2)} \subset \flowtraj^\saferadius$. 	
\end{lemma}

\begin{restatable}[Behavior along the top eigenvector, two steps]{lemma}{BehaviorGt}\label{lem:BehaviorGt}
For sufficiently small $\eta$, 
consider any time $t$, such that $x_{\eta}(t) \in \flowtraj^\saferadius$,  $\norm{  \Tilde{x}_{\eta} (t) } \le \frac{1}{2} \eta 
 \lambda_1(t) + \normerrterm \eta^2$ and $G_t\ge \Omega(\eta)$ holds true, the following holds:
        \begin{align*}
            G_{t+2} &\ge  ( 1 + 2 \min_{2 \le j \le M} \frac{ \lambda_j(t) ( \lambda_1(t) - \lambda_j(t) ) }{ \lambda^2_1(t) } \sin^2 \theta_t ) G_t  - O( ( \theta_{t}+ \eta)\eta^2)\\
            & \ge G_t -O(\eta^3)    
        \end{align*}
        provided that $\overline{x_{\eta}(t) x_{\eta}(t+1)}, \overline{x_{\eta}(t+1) x_{\eta}(t+2)} \subset \flowtraj^\saferadius$. 
\end{restatable}

\begin{proof}[Proof of \Cref{lem:BehaviorGt}]

First note that $\angle (\tilde x_\eta(t), \nabla L(x_\eta(t))) = O(\frac{\norm{\tilde x_\eta(t) - \nabla L(x_\eta(t))}}{\norm{\tilde x_\eta(t)}}) = O(\frac{\eta^2}{G_t})= O(\eta)$, where the last step we use \Cref{lem:appr_gradient_norm}. Let $\delta =  \angle (v_1(t), \nabla L(x_\eta(t))) -\angle (v_1(t), \tilde x_\eta(t))$ and we have $\abs{\delta} \le  \angle (\tilde x_\eta(t), \nabla L(x_\eta(t))) =  O(\eta)$. Therefore, it holds that 
\begin{align*}
&\inner{v_1(t)}{\frac{\nabla L(x_\eta(t))}{\norm{\nabla L(x_\eta(t))}}}	 = \cos \angle(v_1(t), \nabla L(x_\eta(t)))\\
=&\cos \angle(v_1(t), \tilde x_\eta(t)) \cos \delta + \sin \angle(v_1(t), \tilde x_\eta(t)) \sin\delta \\
=& \cos \angle(v_1(t), \tilde x_\eta(t)) + O(\sin \delta \sin \angle(v_1(t), \tilde x_\eta(t)) + \sin^2 \delta) \\
= &\inner{v_1(t)}{\frac{\tilde x_\eta(t)}{\norm{\tilde x_\eta(t)}}}+ O( ( \theta_t+ \eta) \eta)
\end{align*}

    Using the Normalized GD update, we have
    \begin{align}
        &\langle v_1(t), x_{\eta} (t+1) - \Phi( x_{\eta} (t) ) \rangle - \langle v_1(t), x_{\eta} (t) - \Phi( x_{\eta} (t) ) \rangle \nonumber \\
        = & - \eta  \inner{v_1(t)}{\frac{\nabla L(x_\eta(t))}{\norm{\nabla L(x_\eta(t))}}} \nonumber \\
        = & - \eta  \inner{v_1(t)}{\frac{\tilde x_\eta(t)}{\norm{\tilde x_\eta(t)}}} +  O( ( \theta_t+ \eta) \eta^2) \label{eq:noisyfirstcoordinateupate_Gt}
\end{align}
    From \Cref{lem:boundmovementinPhi}, we have $\norm{\Phi(x_{\eta} (t)) - \Phi(x_{\eta} (t+1))} \le O(\secondPhi \eta^2)$, which further implies, $\norm{ \nabla^2 L(\Phi(x_{\eta} (t+1))) - \nabla^2 L(\Phi(x_{\eta} (t))) } \le O( \eta^2 ) $. 
    Thus, we can use \Cref{lem:DavisKahn} to have $\norm{v_1(t) - v_1(t+1)} \le O( \frac{\thirdL \secondPhi \eta^2 }{ \lambda_1(t) - \lambda_2(t) } ) = O( \eta^2 ) $. From \Cref{lem:PhiGt}, we have $\abs{ \langle v_1(t), \Phi(x_{\eta} (t+1)) - \Phi(x_{\eta} (t)) \rangle} \le O( \eta^3)$. Thus we have that 
    \begin{align*}
    	\inner{v_1(t)}{\tilde x_\eta(t+1) -  \tilde x_\eta(t)} = \inner{v_1(t)}{\nabla ^2L (\Phi(x_\eta(t)))(x_\eta(t+1) - x_\eta(t))} + O(\eta^3),
    \end{align*}
    and therefore,
    \begin{align*}
    &  \inner{v_1(t+1)}{\tilde x_\eta(t+1)} - \inner{v_1(t)}{\tilde x_\eta(t)}\\
    = &  \inner{v_1(t)}{\tilde x_\eta(t+1)} - \inner{v_1(t)}{\tilde x_\eta(t)} + O(\eta^3)\\
    = & - \eta\lambda_1(t)  \inner{v_1(t)}{\frac{\nabla L(x_\eta(t))}{\norm{\nabla L(x_\eta(t))}}} +  O( \eta^3)\\
    = & - \eta\lambda_1(t)  \inner{v_1(t)}{\frac{\tilde x_\eta(t)}{\norm{\tilde x_\eta(t)}}} +  O( ( \theta_t+ \eta) \eta^2)
    \end{align*}

Thus we have that
\begin{align*}
G_{t+1} =    \abs{1- \eta \frac{\lambda_1(t)}{\norm{\tilde x(t)}}} G_t +  O( ( \theta_t+ \eta) \eta^2)
\end{align*}

Therefore, we have the following inequality by applying the same argument above to $t+1$:
\begin{align*}
G_{t+2} 
= &\abs{1-\eta\frac{\lambda_1(t+1)}{\norm{\tilde x(t+1)}}}G_{t+1} + O( ( \theta_{t+1}+ \eta) \eta^2)\\
= &\abs{1-\eta\frac{\lambda_1(t+1)}{\norm{\tilde x(t+1)}}}\abs{1-\eta\frac{\lambda_1(t)}{\norm{\tilde x(t)}}}G_t\\
 + &\abs{1-\eta\frac{\lambda_1(t+1)}{\norm{\tilde x(t+1)}}}\cdot O( ( \theta_{t}+ \eta) \eta^2)   + O( ( \theta_{t+1}+ \eta) \eta^2) 	
\end{align*}

By \Cref{lem:dropbyhalf_general}, we know $\norm{\tilde x(t+1)} = \Omega(\eta)$. By \Cref{lem:behaviornormhalf_general}, we know that $\theta_{t+1}\le \theta_{t} + O(\eta)$. Thus 
\begin{align}\label{eq:noisyfirstcoordinateupate_two_step}
G_{t+2}  = 	\abs{1-\frac{\eta\lambda_1(t+1)}{\norm{\tilde x(t+1)}}}\abs{1-\frac{\eta\lambda_1(t)}{\norm{\tilde x(t)}}}G_t + O( ( \theta_{t}+ \eta) \eta^2). 
\end{align}

Next we will show $\abs{ \norm{\Tilde{x}_\eta(t+1)} - 	 \norm{(I- \frac{\eta \nabla^2 L(\Phi(x_\eta(t)))}{\norm{\tilde x_\eta(t)}})\Tilde{x}_\eta(t)} }  = O(\eta^2\theta_t)$.  For convenience, we denote $\nabla ^2 L(\Phi(x_\eta(t)))$ by $H$. First we have that
\begin{align*}
&\abs{ \norm{\tilde x_\eta(t)- \eta H \frac{\nabla L(x_\eta(t))}{\norm{\nabla L(x_\eta(t))}}}^2- 	 \norm{\tilde x_\eta(t)- \eta H \frac{\Tilde{x}_\eta(t)}{\norm{\tilde x_\eta(t)}}}^2} \\
= & \abs{ \inner{ 2 \tilde x_\eta(t) - \eta H \left( \frac{\nabla L(x_\eta(t))}{\norm{\nabla L(x_\eta(t))}}+\frac{\Tilde{x}_\eta(t)}{\norm{\tilde x_\eta(t)}}\right) }{\eta H \left( \frac{\nabla L(x_\eta(t))}{\norm{\nabla L(x_\eta(t))}}-\frac{\Tilde{x}_\eta(t)}{\norm{\tilde x_\eta(t)}}\right) }} \\
= & \abs{ \inner{ 2 H\tilde x_\eta(t) - \eta H^2 \left( \frac{\nabla L(x_\eta(t))}{\norm{\nabla L(x_\eta(t))}}+\frac{\Tilde{x}_\eta(t)}{\norm{\tilde x_\eta(t)}}\right) }{\eta \left( \frac{\nabla L(x_\eta(t))}{\norm{\nabla L(x_\eta(t))}}-\frac{\Tilde{x}_\eta(t)}{\norm{\tilde x_\eta(t)}}\right) }} \\
= &  \abs{ 2 H\tilde x_\eta(t) - \eta H^2 \left( \frac{\nabla L(x_\eta(t))}{\norm{\nabla L(x_\eta(t))}}+\frac{\Tilde{x}_\eta(t)}{\norm{\tilde x_\eta(t)}}\right) }\abs{\eta \left( \frac{\nabla L(x_\eta(t))}{\norm{\nabla L(x_\eta(t))}}-\frac{\Tilde{x}_\eta(t)}{\norm{\tilde x_\eta(t)}}\right) }\cos \alpha\\
= & O(\eta^3 \cos \alpha), 
\end{align*}
where $\alpha$ is the angle between $\frac{\nabla L(x_\eta(t))}{\norm{\nabla L(x_\eta(t))}}-\frac{\Tilde{x}_\eta(t)}{\norm{\tilde x_\eta(t)}}$ and $ 2 H\tilde x_\eta(t) - \eta H^2 \left( \frac{\nabla L(x_\eta(t))}{\norm{\nabla L(x_\eta(t))}}+\frac{\Tilde{x}_\eta(t)}{\norm{\tilde x_\eta(t)}}\right) $. Note that  and that both $\angle (\tilde x_\eta(t), v_1(t)), \angle (\nabla L(x_\eta(t)), v_1(t)) = O(\eta_t + \eta)$, we have that the angle between $\frac{\nabla L(x_\eta(t))}{\norm{\nabla L(x_\eta(t))}}+\frac{\Tilde{x}_\eta(t)}{\norm{\tilde x_\eta(t)}}$ and $ 2 H\tilde x_\eta(t) - \eta H^2 \left( \frac{\nabla L(x_\eta(t))}{\norm{\nabla L(x_\eta(t))}}+\frac{\Tilde{x}_\eta(t)}{\norm{\tilde x_\eta(t)}}\right) $ is at most $O(\eta_t + \eta)$. Further note that $\frac{\nabla L(x_\eta(t))}{\norm{\nabla L(x_\eta(t))}}-\frac{\Tilde{x}_\eta(t)}{\norm{\tilde x_\eta(t)}}$ is perpendicular to $\frac{\nabla L(x_\eta(t))}{\norm{\nabla L(x_\eta(t))}}+ \frac{\Tilde{x}_\eta(t)}{\norm{\tilde x_\eta(t)}}$, we know $\cos \alpha  \le O(\theta_t + \eta)$.
Therefore we have that 
\begin{align*}
&\abs{ \norm{\tilde x_\eta(t)- \eta H \frac{\nabla L(x_\eta(t))}{\norm{\nabla L(x_\eta(t))}}}- 	 \norm{\tilde x_\eta(t)- \eta H \frac{\Tilde{x}_\eta(t)}{\norm{\tilde x_\eta(t)}}}}\\
 = &	\frac{\abs{ \norm{\tilde x_\eta(t)- \eta H \frac{\nabla L(x_\eta(t))}{\norm{\nabla L(x_\eta(t))}}}^2- 	 \norm{\tilde x_\eta(t)- \eta H \frac{\Tilde{x}_\eta(t)}{\norm{\tilde x_\eta(t)}}}^2} }{\abs{ \norm{\tilde x_\eta(t)- \eta H \frac{\nabla L(x_\eta(t))}{\norm{\nabla L(x_\eta(t))}}}+ 	 \norm{\tilde x_\eta(t)- \eta H \frac{\Tilde{x}_\eta(t)}{\norm{\tilde x_\eta(t)}}}} } \\
 = &O(\eta^2 (\eta+\theta_t)).
\end{align*}

By \Cref{lem:normbound_onincrease}, we have that
\begin{align*}
\norm{\Tilde{x}_\eta(t)} +  \norm{(I- \frac{\eta \nabla^2 L(\Phi(x_\eta(t)))}{\norm{\tilde x_\eta(t)}})\Tilde{x}_\eta(t)}  \le \eta \lambda_1(t) \left(1 -  \frac{1}{2\lambda_1(t)}  \min_{2 \le j \le M} \frac{ \lambda_j(t) ( \lambda_1(t) - \lambda_j(t) ) }{ \lambda^2_1(t) } \sin^2 \theta_t \right).
\end{align*}
Thus we have proved a perturbed version of \Cref{lem:normbound_onincrease}, that is, 
\begin{align*}
\norm{\Tilde{x}_\eta(t)} +  \norm{\Tilde{x}_\eta(t+1)}  \le \eta \lambda_1(t) \left(1 -  \frac{1}{2\lambda_1(t)}  \min_{2 \le j \le M} \frac{ \lambda_j(t) ( \lambda_1(t) - \lambda_j(t) ) }{ \lambda^2_1(t) } \sin^2 \theta_t \right)+ O(\eta^2(\eta+\theta_t)).
\end{align*}

Therefore a perturbed version of \Cref{lem:firscoordincr_ondrop} would give us:
\begin{align*}
	(1-\frac{\eta\lambda_1(t+1)}{\norm{\tilde x(t+1)}})(1-\frac{\eta\lambda_1(t)}{\norm{\tilde x(t)}}) \ge  2 \min_{2 \le j \le M} \frac{ \lambda_j(t) ( \lambda_1(t) - \lambda_j(t) ) }{ \lambda^2_1(t) } \sin^2 \theta_t + O((\eta+\theta)\eta).
\end{align*}
The proof of the first inequality is completed by plugging the above equation into \Cref{eq:noisyfirstcoordinateupate_two_step}.

The second inequality is immediate by noting that $\eta\theta^2_t + C^2 \eta^3 \ge 2 C \eta^2\theta_t$ for any $C>0$.
\end{proof}

\subsection{Dynamics in Top Eigenspace When  Dropping Below Threshold}\label{sec:movementtopsmalleta}
In this section, we will show that the projection along the top eigenvector cannot drop below a certain threshold. Formally, we will show the following lemma that predicts the increase in the projection $G_t= \abs{ \langle v_1(t), \tilde x_\eta(t) \rangle }$ along the top eigenvector in $O(\log 1/\eta)$ steps, whenever the projection drops below a certain threshold $\gmax(t):=\max_{k \in [M]} g_t(\lambda_k(t))$.

\begin{lemma}\label{lem:Behavior_explode2nd}
     Denote $\noiseparameter = \eta^{100}$. For any constant  $0 < \init $, there is a constant $\restrict>0$, such that for any step $t$ and  $x_{\eta}(t) \in \flowtraj^\saferadius$ with  the  following conditions hold:
    \begin{enumerate}
        \item $\init \eta \le G_t\le 0.98 \gmax(t) \eta.$
        \item $\abs{ \langle v_i(t), \tilde x_\eta(t)\rangle } \le O( \eta^2),$ for all $2 \le i \le M.$
     \end{enumerate}

     Then, with probability at least $1 - \eta^{12}$,  after  perturbing $x_\eta(t)$ with noise generated uniformly from $B_0(\noiseparameter)$ followed by  $t_{\mathrm{esc}} + 2 = \Theta(\log 1/\eta)$ steps of Normalized GD ($\overline{t} = t + t_{\mathrm{esc}} + 2$), it holds that $\norm{ P_{t, \Gamma}^{(2:M)} \tilde x_\eta(t)}\ge \Omega( \eta^2 )$ provided that $\overline{x_{\eta}(t')x_{\eta}(t' + 1)} \subset \flowtraj^\saferadius$ for all time $t \le t' \le \overline{t}$. 
\end{lemma}

\Cref{lem:Behavior_explode2nd} is a direct consequence of the following lemma.

\begin{lemma}\label{lem:phase1}
Consider any time $t$, with $x_{\eta}(t) \in \flowtraj^\saferadius$. Suppose $x_{\eta} (t)$ satisfies the conditions in \Cref{lem:Behavior_explode2nd}. The constants $c_{\mathrm{esc}}, \gmax(t), \noiseparameter, \restrict,$ and $\init$ have been taken from \Cref{lem:Behavior_explode2nd}.
Define $\mathcal{X}_{\mathrm{stuck}}$ as the region in $B_{x_{\eta} (t)}(\noiseparameter)$ such that starting from any point $u \in \mathcal{X}_{\mathrm{stuck}}$, the points $\{u(\tilde{t})\}_{\tilde{t} \in [t_{\mathrm{esc}}]}$, with $u(0) := u$, obtained using $t_{\mathrm{esc}}$ steps of Normalized GD satisfy:
\begin{align}
    &\norm{P_{t, \Gamma}^{(2:M)} ( u(\tilde{t}) - \Phi( x_{\eta} (t) ) ) } \le \restrict \eta^2, \quad \text{ for all }  \tilde{t} \in [t_{\mathrm{esc}}],  \label{eq:escape_condition}
\end{align}
where $P_{t, \Gamma}^{(2:M)}$ denotes the subspace spanned by $v_2(t), \ldots, v_M(t)$.

Consider two points $u$ and $w$ in $B_{x_{\eta} (t)}(\noiseparameter)$, with the property $w = u + K\eta^{12} \noiseparameter v_k(t)$, \footnote{$\eta^{12}$ can be replaced by any $\eta^p$, and the final success probability in  \Cref{lem:Behavior_explode2nd} becomes $1- \eta^{p-2}$.} where $K\ge 1$ can be arbitrary number and $v_k(t)$ denotes the eigenvector corresponding to the eigenvalue $\lambda_k(t) = \argmax_{\lambda_i(t) \mid 1 \le i \le M} g_t(\lambda_i(t))$. Then, at least one of $u$ and $w$ is not present in the region $\mathcal{X}_{\mathrm{stuck}}$. 
\end{lemma}

We will first prove \Cref{lem:Behavior_explode2nd} and then we turn to the proof of \Cref{lem:phase1}. 
\begin{proof}[Proof of \Cref{lem:Behavior_explode2nd}]
\Cref{lem:phase1} shows that if some point $u\in B_{x_\eta(r)}$ is in $ \mathcal{X}_{\mathrm{stuck}}$, then  it holds that 
\begin{align*}
	\mathcal{X}_{\mathrm{stuck}}\cap \{u+ \lambda v_k(t) \mid \lambda \in \mathbb{R} \} \subset \{u+ \lambda v_k(t) \mid \lambda \in \mathbb{R}, |\lambda|\le \eta^{12} r \}. 
\end{align*}
The other words, $\mathcal{X}_{\mathrm{stuck}}$ is only a thin slice of width at most $\eta^{12}r$ of $B_{x_\eta(t)}(r)$, which implies  $\mathrm{vol}(\mathcal{X}_{\mathrm{stuck}})/ \mathrm{vol}(B_{x_\eta(t)}(t)) = O(\eta^12)$, where $\mathrm{vol}(\cdot)$ denotes the volume of the set. 
\end{proof}

\begin{proof}[Proof of \Cref{lem:phase1}]
    We will prove by contradiction. Consider the two sequences obtained with $t_{\mathrm{esc}}$ steps of Normalized GD, $\{u(\tilde{t}), w(\tilde{t})\}_{\tilde{t} \in [t_{\mathrm{esc}}]}$:
    \begin{align*}
        u(0) = u, \quad w(0) = w, \quad u(\tilde{t}) = u(\tilde{t}) - \eta \frac{\nabla L(u(\tilde{t}))}{\norm{\nabla L(u(\tilde{t}))}}, \quad w(\tilde{t}+1) = w(\tilde{t}) - \eta \frac{\nabla L(w(\tilde{t}))}{\norm{\nabla L(w(\tilde{t}))}}.
    \end{align*}
        
    For convenience, we denote $\nabla^2 L(\Phi(x_{\eta} (t))) [u(\tilde{t}) - \Phi(x_{\eta} (t))]$ by $\tilde u (\tilde t)$ and $\nabla^2 L(\Phi(x_{\eta} (t))) [w(\tilde{t}) - \Phi(x_{\eta} (t))]$ by $\tilde w (\tilde t)$.   Suppose both $\norm{P_{t, \Gamma}^{(2:M)} \tilde u(\tilde{t}) }, \norm{P_{t, \Gamma}^{(2:M)} \tilde v(\tilde{t}) }$ are $O(\eta^2)$,   we will show the following, which indicates the contradiction:
    \begin{align*}
        \norm{ P_{t, \Gamma}^{(2:M)} (u(t_{\mathrm{esc}}) - w(t_{\mathrm{esc}})) } \ge \Omega (\eta^2).    
    \end{align*}
    
    An important claim to note is the following:
    \begin{lemma}\label{lem:condition_perturbed_modifieduw}
    Both the trajectories $\{u(\tilde{t}), w(\tilde{t})\}_{\tilde{t} \le t_{\mathrm{esc}}}$ satisfy a modified version of the alignment condition (\Cref{eq:condition_perturbed}), i.e. for all $1 \le j \le M$: 

    \begingroup 
    \allowdisplaybreaks
    \begin{align*}
       \max\left( \sqrt{ \sum_{i=j}^{M} \langle v_i(t),\tilde w(\tilde t) \rangle^2},  \sqrt{ \sum_{i=j}^{M} \langle v_i(t),\tilde u(\tilde t) \rangle^2 } \right)&\le \lambda_j(t) \eta + O( \normerrterm \eta^2 ).
    \end{align*}
    \endgroup
    \end{lemma}

    Note that the condition has been slightly changed to use $\{v_i(t)\}$ as reference coordinate system and $\Phi(x_{\eta} (t))$ as reference point. The above lemma follows from the fact that both $u(0)$ and $w(0)$ are $\noiseparameter$-close to $x_{\eta} (t)$, which itself satisfies the alignment condition (\Cref{eq:condition_perturbed}). Thus, both $u(0)$ and $w(0)$ initially follow the desired condition. Since, both the trajectories follow Normalized GD updates, the proof will follow from applying the same technique used in the proof of \Cref{lem:phase0}. Another result to keep in mind is the following modified version of \Cref{cor:behaviornormhalf_general}, \Cref{lem:behaviornormhalf_general_modifieduw}.
    \begin{lemma}\label{lem:behaviornormhalf_general_modifieduw}
    If $\norm{\tilde u(\tilde t) } \le \eta \frac{\lambda_1(t)}{2} + \normerrterm \eta^2$, then $\abs{v_1(t)^{\top} \tilde u(\tilde t +1)} \ge  \abs{v_1(t)^{\top}\tilde u(\tilde t) } -  O( \eta^2).$
The same results hold for $\tilde w(\tilde t)$ as well.
    
    If $\norm{\tilde w(\tilde t) } \le \eta \frac{\lambda_1(t)}{2} + \normerrterm \eta^2$, $\norm{\tilde u(\tilde t) } \le \eta \frac{\lambda_1(t)}{2} + \normerrterm \eta^2$, and $z(\translate)$ denotes $\translate u(0) + (1 - \translate) w(0) $ for any $\translate \in [0, 1]$, let $F(x) = x - \eta \frac{\nabla L(x)}{\norm{ \nabla L(x)}}$, we have
    \begin{align*}
        &\abs{v_1(t)^{\top} \nabla^2L (\Phi(x_{\eta} (t))) (F(z(\translate)) - \Phi(x_{\eta} (t))) } \\&\ge \abs{v_1(t)^{\top} \nabla^2L (\Phi(x_{\eta} (t))) (z(\translate) - \Phi(x_{\eta} (t))) } - O( \eta^2).
    \end{align*}
    \end{lemma}
    
    The above lemma uses $\{v_i(t)\}$ as reference coordinate system and $\Phi(x_{\eta} (t))$ as reference point. The above lemma follows from showcasing  Normalized GD updates of $u(\tilde{t})$ and $w(\tilde{t})$ as equivalent to the update in a quadratic model, with an additional noise of $O( \frac{ \thirdL \secondL}{\mineigen}  \eta^2)$, similar to \Cref{eq:noisynormalizedGD}.
    
    Continuing with the proof of \Cref{lem:phase1},
    we first consider the behavior of $u$. 
    Since $u\in \mathcal{X}_{\mathrm{stuck}}$, we have for any time-step $\tilde{t}$:
    \begin{align}
\min_{s\in\{\pm1\}}\norm { \frac{ u(\tilde{t}) - \Phi(x_{\eta} (t)) }{ \norm{ u(\tilde{t})  - \Phi (x_{\eta} (t)) } } -s v_1(t) } \le \eta \label{eq:closenesstov_1}
    \end{align}
    Further, applying the same technique from \Cref{lem:BehaviorGt}, we can show that 
    \begin{align}
        \abs{ \langle v_1(t), \tilde u(\tilde{t}+2) \rangle - \langle v_1(t), \tilde u(\tilde{t}) ) \rangle  }  =O(\eta^3).\label{eq:movementalongv_1}
    \end{align}
    
    Initially, because $u$ was initialized close to $x_{\eta} (t)$, we must have 
    $$\abs{ \langle v_1(t), \tilde u(0) \rangle - \langle v_1(t), \tilde x_{\eta} (t) \rangle } \le O(\noiseparameter).$$ 
    Hence, $\abs{ \langle v_1(t), u(\tilde{t}) - \Phi(x_{\eta} (t))\rangle - \langle v_1(t), u(0) - \Phi(x_{\eta} (t))\rangle } \le O(  \eta^3 t_{\mathrm{esc}})$ for all even $\tilde{t} \in [t_{\mathrm{esc}}]$. With $t_{\mathrm{esc}} \sim O(
     \log 1/\eta)$ , we must have 
     \begin{align}
         0.99 \gmax (t) \ge  \abs{ \langle v_1(t), \tilde u(\tilde{t}))\rangle }
         \ge 0.5 \init \eta, \label{eq:lowerboundutvt}
     \end{align}
     for any $t \le \tilde{t} \le t + t_{\mathrm{esc}}.$ The same argument applies to $w(\cdot)$ as well. By \Cref{eq:escape_condition}, we know $\norm{\tilde u(\tilde t)},\norm{\tilde w(\tilde t)} = o(\eta)$.

    Now, we consider the behavior of $w(\cdot)$ and $u(\cdot)$. Consider an even time step $0 \le \tilde{t} \le t_{\mathrm{esc}}$. From the update rule of $w$ and $u$, we have
        \begin{align*}
            w(\tilde{t}+2) - u(\tilde{t}+2) = F^{(2)}(w(\tilde{t})) - F^{(2)}(u(\tilde{t})),  
        \end{align*}
        where the function $F:\mathbb{R}^{D} \to \mathbb{R}^{D}$, $F(v) = v - \eta \frac{\nabla L(v)}{\norm{\nabla  L(v)} }$ is the one-step update rule of Normalized GD and $F^{(2)} = F\circ F$.
        
    Now, we use taylor expansion of $F$ around $u(\tilde{t})$  to get
        \begin{align*}
            w(\tilde{t}+2) - u(\tilde{t}+2) &= F^{(2)}(  w(\tilde{t}))  - F^{(2)}(  u(\tilde{t}))  
            = \nabla F^{(2)}(  u(\tilde{t}))  (w(\tilde{t}) - u(\tilde{t})) + \mathrm{err},
        \end{align*}
        where $\norm{ \mathrm{err}}$ can be bounded as follows, with $ z(\translate) $ defined as $ \translate u(\tilde{t}) + (1 - \translate) w(\tilde{t})$:
        \begin{align*}
            &\max_{ \translate \in [0,1] } \frac{1}{2} \norm{ \nabla^2  F^{(2)}( z(\translate) )) } \norm{ w(\tilde{t}) - u(\tilde{t}) }^2 \\&
            = \max_{ \translate \in [0,1]} \frac{1}{2} \norm{ \nabla [ \nabla F(F(z(\translate)))) \nabla F(z(\translate)) ]  } \norm{ w(\tilde{t}) - u(\tilde{t}) }^2 \\&
            \le \max_{ \translate \in [0,1]}  \eta \cdot O \left(\frac{1}{\norm{\nabla L(z(\translate))}^2} + \frac{1}{\norm{\nabla L(F(z(\translate)))}^2 } \right)   \norm{ w(\tilde{t}) - u(\tilde{t}) }^2 \\& \cdot \max \left( \norm{\partial^2 (\nabla L) ( z(\translate) )}, \norm{ \nabla^2 L(z(\translate)) }^2, \norm{\partial^2 (\nabla L) ( F(z(\translate)) )}, \norm{ \nabla^2 L(F(z(\translate))) }^2 \right) \\&
            \le \max_{ \translate \in [0,1]}   \left(\frac{1}{\norm{\nabla L(z(\translate))}^2} + \frac{1}{\norm{\nabla L(F(z(\translate)))}^2 } \right)   \cdot O( \eta \norm{ w(\tilde{t}) - u(\tilde{t}) }^2).
        \end{align*}
         Using taylor expansion:
            $\nabla L(z(\translate)) = \nabla^2 L(\Phi(x_{\eta} (t))) (z(\translate) - \Phi(x_{\eta} (t))) + O( \thirdL \norm{ z(\translate) - \Phi(x_{\eta} (t)) }^2 )$ and hence, we must have
            $\norm{ \nabla L(z(\translate))} \ge \Omega(\eta)$. 
            
            With $\norm{\tilde u(\tilde{t}) } = o(\eta)$, we can apply \Cref{lem:behaviornormhalf_general_modifieduw} to show $$\abs{ \langle v_1(t), \nabla^2 L(\Phi(x_{\eta} (t))) [F(z(\translate)) - \Phi(x_{\eta} (t))] \rangle} \allowbreak \ge \abs{ \langle v_1(t), \nabla^2 L(\Phi(x_{\eta} (t))) [z(\translate) - \Phi(x_{\eta} (t))] \rangle} + O(   \eta^2).$$ That implies, $\norm{ \nabla L(F(z(\translate)))} \ge \Omega(\eta)$
            
            Hence, $\mu(\Tilde{t}) = \max_{\translate}  \frac{1}{\norm{\nabla L(z(\translate))}^2}  + \frac{1}{\norm{\nabla L(F(z(\translate)))}^2} \le \Omega(1/\eta^2).$
            
    Thus we conclude
    \begin{equation}
        \norm{w(\tilde{t}+2) - u(\tilde{t}+2) - \partial F^{(2)}(u(\tilde{t})) (w(\tilde{t}) - u(\tilde{t}))} \le O( \frac{1}{\eta} \norm{w(\tilde{t}) - u(\tilde{t})}^2),  \label{eq:exponentialblowup_difference}
    \end{equation}
    where  $ \partial F^{(2)}(u(\tilde{t}))  = A_{\tilde{t}+1}A_{\tilde{t}}$ with
    \begin{align*}
        A_{\tilde{t}} := \partial F(u(\tilde t)) =   I - \eta \left[ I - \frac{ \nabla L ( u(\tilde{t}) ) \nabla L ( u(\tilde{t}) )^{\top}  }{ \norm{ \nabla L ( u(\tilde{t}) ) }^2  }\right] \frac{ \nabla^2 L( u(\tilde{t})  )  }{ \norm{ \nabla L ( u(\tilde{t}) ) } },
    \end{align*}
 and  $\mu(\tilde{t})$ is given by
    \begin{align}
        \mu(\tilde{t}) = \max_{ \translate \in [0,1]: z(\translate) = \translate u(\tilde{t}) + (1 - \translate) w(\tilde{t}) }   \left(\frac{1}{\norm{\nabla L(z(\translate))}^2} + \frac{1}{\norm{\nabla L(F(z(\translate)))}^2 } \right), \label{eq:defmu}
    \end{align}
       
    Now we define $B_{\tilde{t}}$ and claim $A_{\tilde{t}}$ can be approximated as below with $\norm{B_{\tilde{t}}} = O( \eta)$. Furthermore, $\norm{A_{\tilde{t}}} \le O(1)$. 
    \begin{align*}
      B_{\tilde{t}} =   A_{\tilde{t}} - I - \eta \left[ I - v_1(t) v_1(t)^{\top} \right] \frac{ \nabla^2 L( \Phi(x_{\eta} (t))  )  }{ \abs{ \langle v_1(t), \nabla^2 L(\Phi (x_{\eta} (t))) [ u(\tilde{t}) - \Phi( x_{\eta} (t) ) ] \rangle } } ,
    \end{align*}
     The following strategies have been used to obtain the above approximation. First, $\norm{ \nabla^2 L( u(\tilde{t})  ) - \nabla^2 L( \Phi( x_{\eta} (t) )  )} \le O( \norm{u(\tilde{t}) - \Phi(x_{\eta} (t)) } ) = O( \norm{u(\tilde t) - \Phi(u(0))}) + O(\norm{\Phi(u(0)) - \Phi(x_\eta(t)) }) =O(\eta)$.
	Therefore, Using taylor expansion, $\nabla L( u(\tilde{t}) ) = \nabla^2 L( \Phi ( x_{\eta} ({t}) ) ) ( u(\tilde{t}) - \Phi(x_{\eta} (t)) ) + O(  \eta^2 ) = \tilde u(\tilde t)+ O(\eta^2)$. Using the update from \Cref{eq:movementalongv_1} and note $t_{\mathrm{esc}}=O(\log 1/\eta)$, we must have $\norm{ \nabla L( u(\tilde{t}) ) }  \ge \abs{ \langle v_1(t), u(\tilde{t}) - \Phi(x_{\eta} (t)) \rangle } - O(\eta^2) \ge \init \eta - O(  \eta^3 t_{\mathrm{esc}} ) \ge \Omega(\eta)$. Finally we use the condition from \Cref{eq:closenesstov_1} to show that $\frac{\nabla L(u(\tilde{t}))}{ \norm{ \nabla L(u(\tilde{t})) }} \left( \frac{\nabla L(u(\tilde{t}))}{ \norm{ \nabla L(u(\tilde{t})) }} \right)^{\top} = v_1(t) v_1(t)^{\top} + O(\eta)$.
    
    Similarly, we can show that:
    \begin{align*}
        A_{\tilde{t}+1} =   I - \eta \left[ I - v_1(t) v_1(t)^{\top} \right] \frac{ \nabla^2 L( \Phi(x_{\eta} (t))  )  }{ \eta \lambda_1(t) - \abs{ \langle v_1(t), \tilde u(\tilde t) \rangle } }   +  B_{\tilde{t}+1},
    \end{align*}
    with $\norm{A_{\tilde{t}+1}} \le O(1)$ and $\norm{B_{\tilde{t}+1}} \le O(\eta)$.

    Now we define  the following error term  
    \begin{equation}
        \mathrm{err}( \tilde{t} ): = w(\tilde{t}+2) - u(\tilde{t}+2) -\prod_{0 \le i \le \tilde{t}: 2\mid i}H ( u( i ) ) ( w(0) - u(0) )  ,  \label{eq:formaldifference}
    \end{equation}

Finally, we use \Cref{lem:usefulexplodingterm} and \Cref{lem:bounerrterm} to handle the main and error terms in \Cref{eq:formaldifference},
    \begin{align*}
         \abs{ \langle v_k(t), w(t_{\mathrm{esc}}) - u( t_{\mathrm{esc}} ) \rangle } &= \abs{ v_k(t)^{\top} \prod_{0 \le \tilde{t} \le t_{\mathrm{esc}}: 2\mid \tilde{t}} H ( u( \tilde{t}  ) ) ( w(0) - u(0) ) + v_k(t)^{\top} \mathrm{err}( t_{\mathrm{esc}} ) } \\&
         \ge \abs{ v_k(t)^{\top} \prod_{0 \le \tilde{t}  \le t_{\mathrm{esc}}: 2\mid \tilde{t}} H ( u( \tilde{t}  ) ) ( w(0) - u(0) ) } - \norm{ \mathrm{err}( t_{\mathrm{esc}} )  }
         \\&\ge \Omega(\eta^2) - O(\eta^3)= \Omega(\eta^2).
    \end{align*}
which completes the proof of \Cref{lem:phase1}.
    \end{proof}

     \begin{lemma}\label{lem:usefulexplodingterm}
        \begin{align*}
            \abs{ v_k(t)^{\top} \prod_{0 \le \tilde{t} \le t_{\mathrm{esc}}: 2\mid \tilde{t}} H ( u( \tilde t ) ) ( w(0) - u(0) ) } \ge \Omega(\eta^2).
        \end{align*}
    \end{lemma}    
        \begin{lemma}\label{lem:bounerrterm}
        \begin{align*}
            \norm{\mathrm{err}(t_{\mathrm{esc}})} \le O(\eta \norm{w(t)-u(t)}) = O(\eta^3).
        \end{align*}
    \end{lemma} 

    \begin{proof}[Proof of \Cref{lem:usefulexplodingterm}]
	For simplicity of presentation, we have used $M_{\tilde{t}}$ to define   
    {\small
    \begin{align*}
         \left[ I - \eta \left[ I - v_1(t) v_1(t)^{\top} \right] \frac{ \nabla^2 L( \Phi(x_{\eta} (t))  )  }{ \eta \lambda_1(t)  - \abs{ \langle v_1(t),\tilde u(\tilde t)\rangle } }    \right]  \left[ I - \eta \left[ I - v_1(T) v_1(t)^{\top} \right] \frac{ \nabla^2 L( \Phi(x_{\eta} (t))  )  }{ \abs{ \langle v_1(t),\tilde u(\tilde t)\rangle } }   \right].
    \end{align*} 
    }%
    Thus, the term under consideration can be simplified as follows,
    \begin{align*}
        &\prod_{0 \le \tilde{t} \le t_{\mathrm{esc}}: 2\mid \tilde{t}} H ( u( \tilde{t} ) ) ( w(0) - u(0) ) \\
        &= \eta^3 \noiseparameter \prod_{0 \le \tilde{t} \le t_{\mathrm{esc}}: 2\mid \tilde{t}} H ( u( \tilde{t} ) ) v_k (t)\\
        &= \eta^3 \noiseparameter \prod_{0 \le \tilde{t} \le t_{\mathrm{esc}}: 2\mid \tilde{t}} A_{\tilde{t}+1} A_{\tilde{t}} v_k (t) \\
        &= \eta^3 \noiseparameter \prod_{0 \le \tilde{t} \le t_{\mathrm{esc}}: 2\mid \tilde{t}} M_{\tilde{t}} v_k(t) + \mathrm{rem},
    \end{align*}
    where using the bounds on $\{ A_{\tilde{t}}, A_{\tilde{t}+1}, B_{\tilde{t}}, B_{\tilde{t}+1} \}_{0 \le \tilde{t} \le t_{\mathrm{esc}}}$, we have
    \begin{align}
        &\norm{ \mathrm{rem} } \le  \max_{\tilde{t} \le t_{\mathrm{esc}} } (\norm{B_{\tilde{t}}} + \norm{B_{\tilde{t}+1}}) \cdot \max_{\tilde{t} \le t_{\mathrm{esc}} } (\norm{A_{\tilde{t}}} + \norm{A_{\tilde{t}+1}}) \cdot \Big\| \sum_{0 \le \tilde{t} \le t_{\mathrm{esc}}: 2\mid \tilde{t}} \quad \prod_{0 \le j \le t_{\mathrm{esc}}: 2\mid j, j \ne \tilde{t}} M_j  \Big\| \nonumber\\
        &\le O(k\eta^12 \noiseparameter)  \cdot \Big\| \sum_{0 \le \tilde{t} \le t_{\mathrm{esc}}: 2\mid \tilde{t}} \quad \prod_{0 \le j \le t_{\mathrm{esc}}:2\mid j, j \ne \tilde{t}} M_j  \Big\|. \label{eq:boundonBBAA}
    \end{align}
     From the behavior of $u(\tilde{t})$ from \Cref{eq:lowerboundutvt}, we have $\abs{ \langle v_1(t), \tilde u(\tilde{t})\rangle } \le  \left(1 - \frac{1}{200M}\right) \gmax(t) \eta.$ 
     Recall that $\gmax(t)$ was chosen as  $\max_{1 \le k \le M} g_t(\lambda_k(t)).$ It turns out that for the chosen upper bound of $\gmax(t)$, $v_k(t)$ acts as the top eigenvector of $M_{\tilde{t}}$ for any $\tilde{t} \le t_{\mathrm{esc}}$.
     
     For all $j \in [2, M]$ and $\tilde{t} \in [t_{\mathrm{esc}}]$, we have:
     \begin{align*}
         M_{\tilde{t}} v_j(t) = \left[ 1 - \eta \frac{ \lambda_{j}(t) / \lambda_1(t) }{ \eta  - \abs{ \langle v_1(t), u(\tilde{t}) - \Phi( x_{\eta} (t) ) \rangle } }    \right]  \left[ 1 - \eta \frac{ \lambda_{j}(t) / \lambda_1(t) }{ \abs{ \langle v_1(t), u(\tilde{t}) - \Phi( x_{\eta} (t) ) \rangle } }   \right] v_j(t),
     \end{align*}
     with $M_{\tilde{t}} v_1(t) = v_1(t)$. 
     When $\abs{ \langle v_1(t), \tilde u(\tilde{t}) \rangle } \le \gmax(t)$, $ \norm{M_{\tilde{t}} v_kt) }\ge \norm{M_{\tilde{t}} v_1(t)}$, for all $j \ge 2$. Furthermore, $\norm{M_{\tilde{t}} v_j(t)}$ maximizes when $j = k$. Therefore, 
     \begin{align*}
        \norm{M_{\tilde{t}}} &= \abs{ \left[ 1 - \eta \frac{ \lambda_{k}(t) / \lambda_1(t) }{ \eta  - \abs{ \langle v_1(t), u(\tilde{t}) - \Phi( x_{\eta} (t) ) \rangle } }    \right]  \left[ 1 - \eta \frac{ \lambda_{k}(t) / \lambda_1(t) }{ \abs{ \langle v_1(t), u(\tilde{t}) - \Phi( x_{\eta} (t) ) \rangle } }   \right] } \\&
        \ge \abs{ \left[  \lambda_1(t) -  \frac{ \lambda_{k}(t) }{  \lambda_1(t)  - 0.99 \gmax(t) }    \right]  \left[ 1 -  \frac{ \lambda_{k}(t) }{  0.99 \gmax(t) }   \right] },  \quad \text{ for all } \tilde{t} \in [t_{\mathrm{esc}}] ,
     \end{align*}
     since we showed before that $\abs{ \langle v_1(t), \tilde u(\tilde{t}) \rangle } \le 0.99 \gmax(t) \eta$ .
     
     Now, we explain our choice of $t_{\mathrm{esc}}$. We select $t_{\mathrm{esc}}$ s.t. 
     $$
       \abs{ \left \langle v_k(t), K\eta^12 \noiseparameter \prod_{0 \le \tilde{t} \le t_{\mathrm{esc}}: 2\mid \tilde{t}} M_{\tilde{t}} v_k(t) \right \rangle } = \Theta(\eta^2). 
     $$
     That is, we select the time step $\tilde{t}$, where the magnitude of the useful term $\prod_{0 \le \tilde{t} \le t_{\mathrm{esc}}: 2\mid \tilde{t}} M_{\tilde{t}} v_k(t)$ along the eigenvector $v_k(t)$ reaches $c_{\mathrm{esc}} \eta$. With $\gmax(t) = g_t(\lambda_k(t))$, 
     we have $\abs{ \left[ 1 -  \frac{ \lambda_{k}(t) / \lambda_1(t) }{ 1  - 0.99\gmax(t) }    \right]  \left[ 1 -  \frac{ \lambda_{k}(t) / \lambda_1(t) }{ 0.99 \gmax(t) }   \right] } \ge 1.001$ 
     and so, we just need $t_{\mathrm{esc}} \le O(\log (c_{\mathrm{esc}}/\eta) ).$
     
     With this choice of $t_{\mathrm{esc}}$, we must have from \Cref{eq:boundonBBAA}, $\norm{\mathrm{rem}} \le O(\eta^3)$ and therefore
     \begin{align*}
         \abs{\left\langle v_k(t),  \prod_{0 \le \tilde{t} \le t_{\mathrm{esc}}: 2\mid \tilde{t}} H ( u( \tilde{t} ) ) ( w(0) - u(0) ) \right\rangle}  \ge \Omega(\eta^2) - O(\eta^3) \ge \Omega(\eta^2), 
     \end{align*}
      Thus, we have shown that with the appropriate choice of $t_{\mathrm{esc}}$, the magnitude of $\prod_{0 \le \tilde{t} \le t_{\mathrm{esc}}: 2\mid \tilde{t}} H ( u( \tilde{t} ) ) ( w(0) - u(0) )$ can reach at least $\Omega(\eta^2)$ along the eigenvector $v_k(t)$.

     \end{proof}

   \begin{proof}[Proof of \Cref{lem:bounerrterm}]      
   We first recall the definition of the error term:
   \begin{equation}
        \mathrm{err}( \tilde{t} ): = w(\tilde{t}) - u(\tilde{t}) -\prod_{0 \le i \le \tilde{t-2}: 2\mid i}H ( u( i ) ) ( w(0) - u(0) )  ,  \label{eq:formaldifference}
    \end{equation}
	By \Cref{eq:exponentialblowup_difference}, the following property holds for  all $\tilde{t} \ge 0$ for some $\secconstantreplace>0$
    \begin{align*}
        & \norm{\mathrm{err}( \tilde{t} +2 ) -   H( u(\tilde{t}) )\mathrm{err}( \tilde{t}  )} \le  \secconstantreplace( \norm{ w(\tilde{t}) - u(\tilde{t}) }^2/\eta)
    \end{align*}

We will use induction hypothesis to show for all even $t\le t_{\mathrm{esc}}$, $\norm{\mathrm{err}(\tilde t)} \le C \norm{ u(\tilde t) -  w(\tilde t)}^2/\eta$ for some sufficiently large constant $C$. The base case is $t=0$ which holds by definition. Now suppose the induction hypothesis holds for all even $0\le t'\le \tilde t \}$ and below we will show for $\tilde t+ 2$.

First, by induction hypothesis at $\tilde t-2$, we know 
$$\norm{w(\tilde t)-u(\tilde t)} - \norm{\prod_{0 \le i \le \tilde{t}-2: 2\mid i}H ( u( i ) ) ( w(0) - u(0) )  } \le \norm{\mathrm{err}(\tilde t)} \le  C(\norm{w(\tilde t)-u(\tilde t)}^2/\eta ).$$ 
Since $\norm{w(\tilde t)-u(\tilde t)} = O(\eta^2)$, we have $\norm{\mathrm{err}(\tilde t)} \le  O(\eta \norm{w(\tilde t)-u(\tilde t)} )$ and that 
$$\norm{w(\tilde t)-u(\tilde t)} \le (1+O(\eta))\norm{\prod_{0 \le i \le \tilde{t}: 2\mid i}H ( u( i ) ) ( w(0) - u(0) )  }.$$
Meanwhile, we have 
\begin{align*}
&\norm{\mathrm{err}(\tilde t + 2)}\\
\le & \norm{u(\tilde t+2)- w(\tilde t+2)}+ \norm{\prod_{0 \le i \le \tilde{t}+2: 2\mid i}H ( u( i ) ) ( w(0) - u(0) )}\\
\le & \norm{u(\tilde t+2)- w(\tilde t+2)}+ \norm{H(u(\tilde t ))\left( w(\tilde t) -u(\tilde t) -\prod_{0 \le i \le \tilde{t}: 2\mid i}H ( u( i ) ) ( w(0) - u(0) )  \right)}\\ 
\end{align*}

Thus we have that 
\begin{align*}
	&\norm{ H( u(\tilde{t}) )}\norm{w(\tilde t)-u(\tilde t)} \\
	\le & (1+O(\eta))\norm{ H( u(\tilde{t}) )}\norm{\prod_{0 \le i \le \tilde{t}: 2\mid i}H ( u( i ) ) ( w(0) - u(0) )  }\\
	= &(1+O(\eta))\norm{\prod_{0 \le i \le \tilde{t}+2: 2\mid i}H ( u( i ) ) ( w(0) - u(0) )  }\\
	\le & (1+O(\eta))\left( \norm{u(\tilde t+2) - w(\tilde t+2)} + \norm{\textrm{err}(\tilde t+2)} \right).
\end{align*}
Note $\norm{\mathrm{err}(\tilde t)} $ is also $O(\eta^2)$, we have 
$$\norm{ H( u(\tilde{t}) )}^2\norm{w(\tilde t)-u(\tilde t)}^2\le (1+O(\eta))\norm{u(\tilde t+2) - w(\tilde t+2)}^2 + O(\eta^2 )\norm{\textrm{err}(\tilde t+2)} $$
    Denote by $\explosionconst = \min_{\tilde{t} \le t_{\mathrm{esc}}} \norm{M_{\tilde{t}}}$. From previous analysis we know $\explosionconst>1$ and thus we have
\begin{align*}
&\eta \norm{\mathrm{err}(\tilde t + 2)}  \\
\le &\secconstantreplace\norm{w(\tilde t+2)-u(\tilde t+2)}^2 + \norm{ H( u(\tilde{t}) )\mathrm{err}( \tilde{t})}\eta \\
\le &\secconstantreplace \norm{w(\tilde t+2)-u(\tilde t+2)}^2  + C\norm{ H( u(\tilde{t}) )}\norm{w(\tilde t)-u(\tilde t)}^2 \\
\le & \secconstantreplace \norm{w(\tilde t+2)-u(\tilde t+2)}^2 + (\frac{C}{\explosionconst} +O(\eta))\norm{u(\tilde t+2) - w(\tilde t+2)}^2 + O(\eta^2 )\norm{\textrm{err}(\tilde t+2)}.
\end{align*}
Thus we conclude that 
\begin{align*}
    \eta\norm{\mathrm{err}(\tilde t + 2)} \le (\secconstantreplace \frac{C}{\explosionconst}+ O(\eta)) \norm{w(\tilde t+2)-u(\tilde t+2)}^2.
\end{align*}
The proof is completed by picking $C$ large enough such that $\secconstantreplace \frac{C}{\explosionconst}+ O(\eta) \le C$.
\end{proof}

\subsection{Proof  for Operating on Edge of Stability}\label{subsec:proof_ngd_eos}
\begin{proof}[Proof of \Cref{thm:stableness_ngd}]
	According to the proof of \Cref{thm:ngd_phase_2}, we know for all $t$, it holds that $R_j(x_\eta(t))\le O(\eta^2)$. Thus $\stableness{L}{x_\eta(t)}{\eta_t}=\eta_t\cdot\sup_{0\le s\le \eta_t} \lambda_{1}(\nabla^2 L(x_\eta(t)-s\nabla L(x_\eta(t)))) = \eta_t (\lambda_{1}(t)+O(\eta))$, which implies that $[\stableness{L}{x_\eta(t)}{\eta_t}]^{-1} = \frac{\norm{\nabla L(x_\eta(t))}}{\eta\lambda_{1}(t)} +O(\eta) =  \frac{\norm{\Tilde x_\eta(t)}}{\eta\lambda_{1}(t)} +O(\eta)$. The proof for the first claim is completed by noting that $ \frac{1}{\eta}(\norm{\Tilde x_\eta(t)}+ \norm{\Tilde x_\eta(t+1)}) = \lambda_{1}(t) + O(\eta+\theta_t)$ as an analog of the quadratic case.
	
	For the second claim, it's easy to check that $\sqrt{L(x_\eta(t))} = \frac{\norm{\Tilde x_\eta(t)}}{\sqrt{2\lambda_1(t)}} +O(\eta\theta_t)$. Thus have  $\sqrt{L(x_\eta(t))}+ \sqrt{L(x_\eta(t+1))} = \frac{\norm{\Tilde x_\eta(t)}}{\sqrt{2\lambda_1(t)}} + \frac{\norm{\Tilde x_\eta(t+1)}}{\sqrt{2\lambda_1(t+1)}} +O(\eta(\theta_t+\theta_{t+1}))$. Note that $\lambda_1(t)-\lambda_1(t+1) = O(\eta^2)$ and $\theta_{t+1} = O(\theta_t)$, we conclude that $\sqrt{L(x_\eta(t))}+ \sqrt{L(x_\eta(t+1))} = \eta\sqrt{\frac{\lambda_1(\nabla^2 L(x_\eta(t))}{2}}) +O(\eta\theta_t)$.
\end{proof}



    


\section{Some Useful Lemmas About Eigenvalues and Eigenvectors}\label{appsec:helpinglemmas}

\begin{theorem}[Derivative of eigenvalues and eigenvectors of a matrix, Theorem 1 in \citep{magnus1985differentiating}]\label{lem:derivative_eig}
Let $X_0$ be a real symmetric $n \times n$ matrix. Let $u_0$ be a normalized eigenvector associated with an  eigenvalue $\lambda_0$ of $X_0$ with multiplicity $1$. Then a real valued function $\lambda$ and a vector valued function $u$ are defined for all $X$ in some neighborhood $N(X_0) \subset \mathbb{R}^{n \times n}$ of $X_0$, such that 
\begin{align*}
    \lambda(X_0) = \lambda_0, \quad u(X_0) = u_0,
\end{align*}
and 
\begin{align*}
    X u = \lambda u, \quad u^{\top} u = 1, \quad X \in N(X_0).
\end{align*}

Moreover, the functions $\lambda$ and $u$ are $\mathcal{C}^\infty$ on $N(X_0)$ and the differentials at $X_0$ are
\begin{align*}
    d \lambda = u_0^{\top} (dX) u_0, \quad du = (\lambda_0 I_n - X_0)^{\dagger} (dX) u_0.
\end{align*}
\end{theorem}

\begin{theorem}\label{lem:perturb_eigenvalue} [Eigenvalue perturbation for symmetric matrices, Cor. 4.3.15 in \citep{horn2012matrix}]
Let $\Sigma, \hat{\Sigma} \in \mathbb{R}^{p \times p}$ be symmetric, with eigenvalues $\lambda_{1} \geq \ldots \geq \lambda_{p}$ and $\hat{\lambda}_{1} \geq \ldots \geq \hat{\lambda}_{p}$ respectively. Then, for any $i \le p$, we have
\begin{align*}
    \abs{ \lambda_i - \hat{\lambda}_i } \le \norm{\Sigma - \hat{\Sigma}}_2.
\end{align*}
\end{theorem}

The next theorem is the Davis-Kahan $\sin (\theta)$ theorem, that bounds the change in the eigenvectors of a matrix on perturbation. Before presenting the theorem, we need to define the notion of unitary invariant norms. Examples of such norms include the frobenius norm and the spectral norm.

\begin{definition}[Unitary invariant norms]
A matrix norm $\|\cdot \|_{*}$ on the space of matrices in $\mathbb{R}^{p \times d}$ is unitary invariant if for any matrix $K \in \mathbb{R}^{p \times d}$, $\norm{ U K W }_{*} = \norm{ K }_{*}$ for any unitary matrices $U \in \mathbb{R}^{p \times p}, W \in \mathbb{R}^{d \times d}.$
\end{definition}

\begin{theorem}\label{lem:DavisKahn} [Davis-Kahan $\sin (\theta)$ theorem \citep{davis1970rotation}]
 Let $\Sigma, \hat{\Sigma} \in \mathbb{R}^{p \times p}$ be symmetric, with eigenvalues $\lambda_{1} \geq \ldots \geq \lambda_{p}$ and $\hat{\lambda}_{1} \geq \ldots \geq \hat{\lambda}_{p}$ respectively. Fix $1 \leq r \leq s \leq p$, let $d:=s-r+1$ and let $V=\left(v_{r}, v_{r+1}, \ldots, v_{s}\right) \in \mathbb{R}^{p \times d}$ and $\hat{V}=\left(\hat{v}_{r}, \hat{v}_{r+1}, \ldots, \hat{v}_{s}\right) \in \mathbb{R}^{p \times d}$ have orthonormal columns satisfying $\Sigma v_{j}=\lambda_{j} v_{j}$ and $\hat{\Sigma} \hat{v}_{j}=\hat{\lambda}_{j} \hat{v}_{j}$ for $j=r, r+1, \ldots, s .$ 
Define $\eigengap:=\min \left\{\max\{0, \lambda_s- \hat \lambda_{s+1}\}, \max\{0,  \hat \lambda_{r-1}-\lambda_r\}\right\}$, 
 where $\hat{\lambda}_{0}:=\infty$ and $\hat{\lambda}_{p+1}:=-\infty$, we have for any unitary invariant norm $\| \cdot \|_{*}$,
$$
\eigengap \cdot \|\sin \Theta(\hat{V}, V)\|_{*}  \leq \|\hat{\Sigma}-\Sigma\|_{*}.
$$
Here $\Theta(\hat{V}, V) \in \mathbb{R}^{d \times d},$  
with $\Theta(\hat{V}, V)_{j, j} = \arccos \sigma_j$ for any $j \in [d]$ and $\Theta(\hat{V}, V)_{i, j} = 0$ for all $i \ne j \in [d]$. $\sigma_1 \ge \sigma_2 \ge \cdots \ge \sigma_d$ denotes the singular values of $\Hat{V}^{\top} V.$ $[\sin \Theta]_{ij}$ is defined as $\sin (\Theta_{ij})$. 
\end{theorem}

\section{Analysis of $\sqrt{L}$}
The analysis will follow the same line of proof used for the analysis of Normalized GD. Hence, we write down the main lemmas that are different from the analysis of Normalized GD. Rest of the lemmas are nearly the same and hence, we have omitted them. 

The major difference between the results of Normalized GD and GD with $\sqrt{L}$  is in the behavior along the manifold $\Gamma$ (for comparison, see \Cref{lem:PhiGt} for Normalized GD and \Cref{lem:PhiGt_sqrtL} for GD with $\sqrt{L}$).
Another difference between the results of Normalized GD and GD with $\sqrt{L}$ is in the error rates mentioned in  \Cref{thm:ngd_phase_2} and \Cref{thm:PhaseII_sqrtL}. The difference comes from the stronger behavior of the projection along the top eigenvector that we showed for Normalized GD in \Cref{lem:BehaviorGt}, but doesn't hold for GD with $\sqrt{L}$ (see \Cref{lem:BehaviorGt_sqrtL}). This difference shows up in the sum of angles across the trajectory (for comparison, see \Cref{lem:avgGt} for Normalized GD and \Cref{lem:avgGt_sqrtL} for GD with $\sqrt{L}$), and is finally reflected in the error rates.

\subsection{Notations}
The notations will be the same as \Cref{appsec:setup} . However, here we will use $\Tilde{x}_{\eta} (t)$ to denote $\left( 2 \nabla^2 L(\Phi ( x_{\eta} (t) )) \right)^{1/2} (x_{\eta} (t) - \Phi(x_{\eta} (t)))$. We will now denote $Y$ as the limiting flow given by \Cref{eq:nonlog_limiting_flow}.
\begin{align*}
    \nonloglimitingflow
\end{align*}
\subsection{Phase I, convergence}
Here, we will show a very similar stability condition for the GD update on $\sqrt{L}$ as the one  (\Cref{lem:phase0}) derived for Normalized GD. Recall our notation $\Tilde{x}_{\eta} (t) = \sqrt{2 \nabla^2 L(\Phi(x_{\eta} (t))) } (x_{\eta} (t) - \Phi(x_{\eta} (t))).$

\begin{lemma}\label{lem:phase0_sqrtL}

    Suppose $\{x_{\eta}(t)\}_{t\ge 0}$ are iterates of GD with $\sqrt{L}$~\eqref{eq:gd_sqrtL} with a learning rate $\eta$ and $x_\eta(0) = \xinit$.
    There is a constant $C>0$, such that for any constant $\secphaseinit>0$, if at some time $t'$, $x_{\eta}(t') \in \flowtraj^\saferadius$ and satisfies $  \frac{\norm{x_{\eta}(t') - \Phi(x_{\eta}(t'))}}{\eta} \le \secphaseinit$,
    then for all $\bar{t} \ge t' + C \frac{\secondL \secphaseinit}{\mineigen} \log \frac{\secphaseinit \secondL }{ \mineigen }$, the following must hold true for all $1 \le j \le M$:
    \begin{align}  \label{eq:condition_perturbed_proof_sqrtL}
          \sqrt{ \sum_{i=j}^{M} \langle v_i(\Bar{t}), \tilde{x}_{\eta} (\Bar{t}) \rangle^2 }   
                &\le \eta \lambda_j(\Bar{t}) + O(\eta^2),
    \end{align}
    provided that for all steps $t \in \{t,\ldots, \bar{t}-1\}$, $\overline{x_{\eta}(t)x_{\eta}(t+1)} \subset \flowtraj^\saferadius$. 
\end{lemma}

\begin{proof}
    The proof exactly follows the strategy used in \Cref{lem:phase0}. We can use the noisy update formulation from \Cref{lem:noisyupdate_sqrtL} and the bound on the movement in $\Phi$ from \Cref{lem:PhiGt_sqrtL} to get for any time $t$ with $\Bar{t} \ge t \ge t'$ (similar to \Cref{eq:noisynormalizedGD}):
        \begin{align*}
            \tilde{x}_{\eta} (t+1) = \left(I - \eta \frac{\nabla^2 L(\Phi(x_{\eta} (t)))  }{ \norm{ \tilde{x}_{\eta} (t) }} \right)\tilde{x}_{\eta} (t) +  O(\eta^2 ) + O( \norm{ x_{\eta} (t) - \Phi(x_{\eta} (t)) }\eta  ),
        \end{align*}
     and the rest proof are the same.
\end{proof}

Hence, similar to \Cref{lem:dropbyhalf_general}, we can derive the following property that continues to hold true throughout the trajectory, once the condition~\Cref{eq:condition_perturbed_proof_sqrtL} is satisfied:
\begin{lemma}
    There is some constant $\normerrterm>0$, such that if the condition~\Cref{eq:condition_perturbed_proof_sqrtL} holds true and $\norm{ \tilde{x}_{\eta} (t) } > \frac{\eta \lambda_1(t) }{2} $, the following must hold true:
    \begin{align*}
        \norm{\tilde{x}_{\eta} (t+1)} \le \frac{\eta \lambda_1(t) }{2} + \normerrterm \eta^2.
    \end{align*}

\end{lemma}


We also have the counterpart of \Cref{cor:behaviornormhalf_general} with the same proof, which follows from using the noisy update of GD on $\sqrt{L}$ from \Cref{lem:noisyupdate_sqrtL} and using the quadratic update result from \Cref{lem:behaviornormhalf}.
\begin{lemma}\label{lem:behaviornormhalf_general_sqrtL}
    If at time $t$, $\norm{ \tilde{x}_{\eta} (t) } \le \frac{\eta \lambda_1(t) }{2} +  \normerrterm \eta^2$ and stability condition (\Cref{eq:condition_perturbed_proof_sqrtL}) holds true, the following must hold true:
    \begin{align*}
        &\abs{ v_1(t+1)^{\top} \Tilde{x}_{\eta}(t+1)  } \ge \abs{v_1(t)^{\top} \Tilde{x}_{\eta}(t)  } - O(  \eta^2 ) . 
    \end{align*}

\end{lemma}

\subsection{Phase II, limiting flow}

To recall, the limiting flow given by
\begin{align*}
    \nonloglimitingflow
\end{align*}

Let $T_2$ be the time up until which solution to the limiting flow exists.  

\Cref{lem:PhiGt_sqrtL} shows the movement in $\Phi$, which can be informally given as follows: in each step $t$,
\begin{align}
       \Phi( x_{\eta}(t + 1) ) - \Phi( x_{\eta}(t) ) &= -\frac{\eta^2}{8} P_{t, \Gamma}^{\perp} \nabla \lambda_{1}( \nabla^2 L (\Phi(x_{\eta} (t))) )  
         + O( \eta^2 (\theta_t + \norm{x_{\eta}(t) - \Phi(x_{\eta}(t)) })  )
       \label{eq:movement_in_Phi_informal_sqrtL},
\end{align}
provided $\overline{\Phi(x_{\eta}(t))  \Phi(x_{\eta}(t+1))} \in \flowtraj^\saferadius$.

Motivated by this update rule, we show that the trajectory of $\Phi(x_{\eta}(\cdot))$ is close to the limiting flow in \Cref{eq:nonlog_limiting_flow}, for a small enough learning rate $\eta$. The major difference from \Cref{thm:ngd_phase_2} comes from the fact that the total error introduced in \Cref{eq:movement_in_Phi_informal_sqrtL} over an interval $[0, t_2]$ is $\sum_{t=0}^{t_2} O(\eta^2\theta_t + \eta^3)$, which is of the order $O(\eta^{1/2})$ using the result of \Cref{lem:avgGt_sqrtL}.

\paragraph{Average of the angles}
The first lemma shows that the sum of the angles in an interval $[0, t_2]$ of length $\Omega(1/\eta^2)$ is at most $O(t_2 \eta^{1/2}).$
\begin{lemma}\label{lem:avgGt_sqrtL}
   	For any $T_2 > 0$ for which solution of \Cref{eq:nonlog_limiting_flow} exists, consider an interval $[0, t_2]$, with $\Omega(\eta^{-2})=t_2 \le  \lfloor T_2 / \eta^2 \rfloor$. 
    Suppose \Cref{alg:PsqrtL} is run with learning rate $\eta$ for $t_2$ steps, starting from a point $x_{\eta}(0)$ that satisfies  
    (1) $\max_{ j\in[D]} \overline{R}_j(x_\eta(0))\le O(\eta^2)$, and (2) $\abs{v_1(0), x_{\eta}(0) - \Phi(x_{\eta}(0)) } \ge \init \eta$ for some constant $0 < \init $ independent of $\eta$, with $\norm{\Tilde{x}_{\eta} (0)} \le \frac{\sqrt{0.5\lambda^2_1(0)}}{2} \eta + \normerrterm \eta^2$,
   	the following holds true with probability at least $1 - \eta^{10}$:
	\begin{align*}
		 \frac{1}{t_2} \sum_{\ell = 0}^{t_2} \theta_{\ell} \le O\left( \sqrt{\eta } \right),
	\end{align*}
	provided $\eta$ is sufficiently small and for all time $0 \le t \le t_2-1$, $\overline{x_{\eta}( t ) x_{\eta}(t+1)} \subset \flowtraj^\saferadius$.
\end{lemma}

\begin{proof}[Proof of \Cref{lem:avgGt_sqrtL}]
    The proof is very similar to the proof of \Cref{lem:avgGt}, except we replace \Cref{lem:stuck_region} by \Cref{lem:stuck_region_sqrtL} in the analysis of case (B). Hence the final average angle becomes $O(\sqrt{\eta})$. 
\end{proof}

\begin{lemma}\label{lem:stuck_region_sqrtL}
    Consider the setting of \Cref{lem:avgGt_sqrtL}. Consider any time interval $[\overline{t},t']$, where $\overline{t} \le \ell < t'$, $\overline{x_{\eta}( \ell ) x_{\eta}(\ell + 1)} \subset \flowtraj^\saferadius$ and  $  \Omega(\eta) \le G_l:= \abs{ \langle v_1(t), \tilde x_\eta(t) \rangle }\le \frac{\lambda_1(\ell)\eta}{2}-\Omega(\eta)$, we have that
    \begin{align*}
\sum_{t\in[\overline t,t']} \theta_t = \sum_{t \in \cycleonestart\cup \cycletwostart \cup \cycletwoother} \theta_t   \le  O(\sqrt{ t'-\overline t}  + (t'-\overline t)\sqrt{\eta}).
    \end{align*}
\end{lemma}

\begin{proof}
The proof will follow exactly as \Cref{lem:stuck_region}, except we replace \Cref{lem:BehaviorGt} by \Cref{lem:BehaviorGt_sqrtL}, which changes the rate into $O(\sqrt{t'-\overline t}+ (t'-\overline t)\sqrt{\eta})$ 
\end{proof}

\begin{lemma}\label{lem:BehaviorGt_sqrtL}[Behavior along the top eigenvector]
Consider any time $t$, such that $x_{\eta}(t) \in \flowtraj^\saferadius$, where $\norm{  \Tilde{x}_{\eta} (t) } \le \frac{1}{2} \eta 
 \lambda_1(t) + \normerrterm \eta^2$ holds true, then the following holds true:
        \begin{align*}
            G_{t+2} &\ge  ( 1 + \frac{1}{2} \min_{2 \le j \le M} \frac{ \lambda_j(t) ( \lambda_1(t) - \lambda_j(t) ) }{ \lambda^2_1(t) } \sin^2 \theta_t ) G_t  -  O(\eta^2),    
        \end{align*}
        provided $G_t \ge \Omega(\eta)$ and $\overline{x_{\eta}(t) x_{\eta}(t+1)}, \overline{x_{\eta}(t+1) x_{\eta}(t+2)} \subset \flowtraj^\saferadius$. 
\end{lemma}

\begin{proof}
    Here, we will follow a much simpler approach than \Cref{lem:BehaviorGt} to have a weaker error bound. The stronger error bounds in \Cref{lem:BehaviorGt} were due to the very specific update rule of Normalized GD.
    
    First recall $\tilde x_\eta(t) = \sqrt{2\nabla^2 L(\Phi(x))} (x_\eta(t)-\Phi(x_\eta(t)))$. By \Cref{lem:PhiGt_sqrtL}, we have $\norm{\Phi(x_\eta(t+1))-\Phi(x_\eta(t))} = O(\eta^2)$, thus $\tilde x_\eta(t+1)- \tilde x_\eta(t) = \sqrt{2\nabla^2 L(\Phi(x))}(x_\eta(t+1) - x_\eta(t)) = \eta\sqrt{2\nabla^2 L(\Phi(x))}\frac{\nabla L(x_\eta(t))}{2\sqrt{L(x_\eta(t))}}$. 
    From \Cref{lem:noisyupdate_sqrtL}, we have
    \begin{align*}
     \frac{ \nabla L(x_{\eta} (t)) }{ \sqrt{ L( x_{\eta} (t) ) }  } = \frac{ \nabla^2 L ( \Phi( x_{\eta} (t) ) ) ( x_{\eta} (t) - \Phi ( x_{\eta} (t) ) ) }{ \sqrt{ \frac{1}{2} \nabla^2 L ( \Phi( x_{\eta} (t) ) ) [ x_{\eta} (t) - \Phi ( x_{\eta} (t) ),  x_{\eta} (t) - \Phi ( x_{\eta} (t) ) ]  } } + O(  \eta ),
\end{align*}
where we have used the fact that $\norm{\tilde x_\eta(t)} = O(\eta)$.

Hence, the update is similar to the update in a quadratic model, with $\nabla^2 L(\Phi(x_{\eta} (t)))$ guiding the updates with an additional $O(\eta^2)$ perturbation. As a result we also get a $O(\eta^2)$ perturbation in $G_t$. Here we use the assumption $G_t = \Omega(\eta)$ so that GD updates are $O(1)$-lipschitz.
\end{proof}


\subsection{Omitted Proof for Operating on Edge of Stability}\label{subsec:proof_sqrtL_eos}
This proof is similar to that of \Cref{thm:stableness_ngd}.
\begin{proof}[Proof of \Cref{thm:stableness_gd_sqrt_L}]
	If $M=1$, that is, the dimension of manifold $\Gamma$ is $D-1$, we know $\overline{x_\eta(t)x_\eta(t+1)}$ will cross $\Gamma$, making the $\nabla^2 \sqrt{L}$ diverges at the intersection and the first claim becomes trivial. If $M\ge 2$, we have $\nabla^2 \sqrt{L} = \frac{2L\nabla^2 L - \nabla L \nabla L^\top}{4\sqrt{L}^3}$ diverges at the rate of $\frac{1}{\norm{\nabla L}}$. It turns out that using basic geometry, one can show that the distance from $\Phi(x_\eta(t))$ to $\overline{x_\eta(t)x_\eta(t+1)}$ is $O(\eta (\theta_t+\theta_{t+1}))$, thus $\sup_{0\le s\le \eta} \lambda_{1}(\nabla^2 \sqrt{L}(x_\eta(t)-s\nabla \sqrt{L}(x_\eta(t)))) = \Omega(\frac{1}{\eta(\theta_t+\theta_{t+1})})$. The proof of the first claim is completed by noting that $\theta_{t+1} = O(\theta_t)$.
		
	For the second claim, it's easy to check that $\sqrt{L(x_\eta(t))} = \norm{\Tilde x_\eta(t)} + O(\eta)$. The proof for the first claim is completed by noting that $\norm{\Tilde x_\eta(t)}+ \norm{\Tilde x_\eta(t+1)} = \eta\lambda_{1}(t) + O(\eta+\theta_t)$ as an analog of the quadratic case.
\end{proof}

\subsection{Geometric Lemmas for $\sqrt{L}$}
First recall our notations,  $\Tilde{x} = \sqrt{2 \nabla^2 L ( \Phi(x) ) } ( x - \Phi(x) )$ and $\theta = \arctan \frac{ \norm{ P_{\Phi(x), \Gamma}^{(2:M)} \Tilde{x} } }{ \abs{ \langle v_1( x ), \Tilde{x}  \rangle} }$.
\begin{lemma}\label{lem:noisyupdate_sqrtL}
At any point $x \in \flowtraj^{\saferadius}$, we have
\begin{align*}
     \frac{ \nabla L(x) }{ \sqrt{ L( x ) }  } = \frac{ \nabla^2 L ( \Phi( x ) ) ( x - \Phi ( x ) ) }{ \sqrt{ \frac{1}{2} \partial^2 L( \Phi( x ) ) [ x - \Phi ( x ),  x - \Phi ( x ) ]  } } + O( \frac{ \secondL^{1/2} \thirdL  } { \mineigen } \norm{x - \Phi(x)} ).
\end{align*}
And therefore,
\begin{align*}
    \norm{ \frac{ \nabla L(x) }{ \sqrt{ L( x ) }  } } \le  \sqrt{2 \lambda_1( \Phi( x ))} + O( \frac{ \secondL^{1/2} \thirdL  } { \mineigen } \norm{x - \Phi(x)} ) = O(\secondL^{1/2}).
\end{align*}
\end{lemma}

\begin{proof}
    By \Cref{lem:appr_gradient_norm}
\begin{align*}
 \norm{ \nabla L(x) - \nabla^2 L(\Phi(x)) (x - \Phi(x)) }\le \frac{1}{2}\thirdL \norm{ x - \Phi(x) }^2.
\end{align*}
    Since $\Phi(x)$ is a local minimizer of zero loss, we have $\nabla L ( \Phi( x ) ) = 0$, we have that 
    \begin{align*}
        L(x) =  \frac{1}{2} \partial^2 L( \Phi( x ) ) [ x - \Phi ( x ),  x - \Phi ( x ) ]  +O( \thirdL \norm{x - \Phi ( x ) }^3 ). 
    \end{align*}
    By \Cref{lem:orthogonal_is_small}, we know $\partial^2 L( \Phi( x ) ) [ x - \Phi ( x ),  x - \Phi ( x ) ] \ge \Omega(\frac{\norm{x-\Phi(x)}^2}{\mineigen})$ and therefore 
    \begin{align*}
    \frac{\sqrt{\frac{1}{2}\partial^2 L( \Phi( x ) ) [ x - \Phi ( x ),  x - \Phi ( x ) ]}}{\sqrt{L(x)}} = 1 + O(\frac{\thirdL}{\mineigen}\norm{x-\Phi(x)}).	
    \end{align*}

Thus we conclude that 
    \begin{align*}
        \frac{ \nabla L(x) }{ \sqrt{ L( x ) }  } 
        = &\frac{  \nabla^2 L(\Phi(x)) (x - \Phi(x)) + O(\thirdL\norm{x-\Phi(x)}^2) }{ \sqrt{\frac{1}{2}\partial^2 L( \Phi( x ) ) [ x - \Phi ( x ),  x - \Phi ( x ) ]} } 
        \cdot \frac{\sqrt{\frac{1}{2}\partial^2 L( \Phi( x ) ) [ x - \Phi ( x ),  x - \Phi ( x ) ]}}{\sqrt{ L( x ) }}\\
        = & \frac{ \nabla^2 L ( \Phi( x ) ) ( x - \Phi ( x ) ) }{ \sqrt{ \frac{1}{2} \partial^2 L( \Phi( x ) ) [ x - \Phi ( x ),  x - \Phi ( x ) ]  } } + O( \frac{ \secondL^{1/2} \thirdL  } { \mineigen } \norm{x - \Phi(x)}).
    \end{align*}
    
   For the second claim, with $\Tilde{x} = \sqrt{ 2\nabla^2 L (\Phi(x)) } (x - \Phi (x))$, we have that
    \begin{align*}
        \norm{ \frac{ \nabla^2 L ( \Phi( x ) ) ( x - \Phi ( x ) ) }{ \sqrt{ \frac{1}{2} \partial^2 L( \Phi( x ) ) [ x - \Phi ( x ),  x - \Phi ( x ) ]  } } } =  \norm{ \sqrt{ \frac{1}{2} \nabla^2 L (\Phi(x) ) }\frac{ \Tilde{x} } { \norm{   \Tilde{x}  } }  } \le \sqrt{2\lambda_1(x)}. 
    \end{align*}
    
    By \Cref{lem:GF_no_escape}, we have $\norm{x-\Phi(x)}\le \frac{2\mineigen}{\secondL\thirdL}$, thus $\frac{ \secondL^{1/2} \thirdL  } { \mineigen } \norm{x - \Phi(x)} = O(\frac{\mineigen}{\secondL^{1/2}}) = O(\secondL^{1/2})$.
\end{proof}

The following two lemmas are  direct implications of \Cref{lem:noisyupdate_sqrtL}.
\begin{lemma}\label{lem:noisyupdate_sqrtL_v2}
At any point $x \in \flowtraj^{\saferadius}$, we have
\begin{align*}
	 \frac{ (\nabla ^2 L(\Phi(x)))^{-1/2}\nabla L(x) }{ \sqrt{ 2L( x ) }  } = \frac{\tilde x}{\norm{\tilde x}} +  O( \frac{ \secondL^{1/2} \thirdL  } { \mineigen^{3/2} } \norm{x - \Phi(x)} ).
\end{align*}
And therefore,
\begin{align*}
    \norm{ \frac{ (\nabla ^2 L(\Phi(x)))^{-1/2}\nabla L(x) }{ \sqrt{ 2L( x ) }  } }=O(1).
\end{align*}
\end{lemma}
\begin{lemma}\label{lem:approx_travel1_sqrtL}
    Consider any point $x \in \flowtraj^{\saferadius}$.  Then, 
    \begin{align*}
        \norm { \sign\left( \left\langle v_1(x), \frac{ \nabla L(x) }{ \sqrt{ L( x ) }  } \right\rangle \right)  v_1( x )-  \frac{ (\nabla ^2 L(\Phi(x)))^{-1/2}\nabla L(x) }{ \sqrt{ 2L( x ) }  }} \le  \theta + O( \frac{  \secondL^{1/2} \thirdL  } { \mineigen^{3/2} } \norm{x - \Phi(x)} ),
    \end{align*}
    where  $\theta = \arctan \frac{ \norm{ P_{\Phi(x), \Gamma}^{(2:M)} \Tilde{x} } }{ \abs{ \langle v_1( x ), \Tilde{x}  \rangle} }$, with $\Tilde{x} = \sqrt{2 \nabla^2 L ( \Phi(x) ) } ( x - \Phi(x) ).$
\end{lemma}

\begin{lemma}\label{lem:PhiGt_sqrtL}
For any $\overline{xy} \in \flowtraj^\saferadius$ where $y= x - \eta \nabla \sqrt{L (x)}$ is the one step update on $\sqrt{L}$ loss from $x$, we have 
\begin{align*}
    \Phi( y ) - \Phi( x ) 
    = &-\frac{\eta^2}{8} P^\perp_{x, \Gamma} \nabla  (  \lambda_{1} ( \nabla^2 L(x) ) ) \\
    + & O(\eta^2 \secondL \secondPhi \theta) + O( \frac{  \secondL^{3/2} \thirdL \secondPhi  } { \mineigen^{3/2} } \norm{x-\Phi(x)} \eta^2 ) + O( \secondL \thirdPhi \norm{x-\Phi(x)} \eta^2 ).
\end{align*}
    Here $\theta = \arctan \frac{ \norm{ P_{\Phi(x), \Gamma}^{(2:M)} \Tilde{x} } }{ \abs{ \langle v_1( x ), \Tilde{x}  \rangle} }$, with $\Tilde{x} = \sqrt{ 2\nabla^2 L ( \Phi(x) ) } ( x - \Phi(x) ).$ 
\end{lemma}

\begin{proof}[Proof of \Cref{lem:PhiGt_sqrtL}]
We outline the major difference from the proof of \Cref{lem:PhiGt}.
Using taylor expansion for the function $\Phi$, we have
\begin{align}
    &\Phi( y ) - \Phi( x ) \nonumber\\&= \partial \Phi( x ) \left(  y - x  \right) + \frac{1}{2} \partial^2 \Phi( x ) [ y - x, y - x ] + \mathrm{err}  \nonumber\\&
    = \partial \Phi( x ) \left(  -\eta \frac{ \nabla L ( x ) }{ 2\sqrt{ L ( x ) } }  \right) + \frac{\eta^2}{2} \partial^2 \Phi( x ) \left[ \frac{ \nabla L ( x ) }{2 \sqrt{  L ( x ) } }, \frac{ \nabla L ( x ) }{ 2\sqrt{ L ( x ) }  } \right] + O(\thirdPhi\eta^3 \norm{\frac{ \nabla L(x) }{ \sqrt{ L( x ) }  }}^3)   \nonumber\\&
    = \frac{\eta^2}{2} \partial^2 \Phi( x ) \left[ \frac{ \nabla L ( x ) }{2 \sqrt{ L ( x ) } }, \frac{ \nabla L ( x ) }{ 2 \sqrt{  L ( x ) } } \right] + O(  \thirdPhi \secondL^{3/2} \eta^3 ), \nonumber
\end{align}
where in the final step, we used the property of $\Phi$ from \Cref{lem:Phi_grad_dot} to kill the first term and use the bound on $\frac{ \nabla L(x) }{ \sqrt{ L( x ) }  }$ from \Cref{lem:noisyupdate_sqrtL} for the third term.

Since the function $\Phi \in \mathcal{C}^3$, hence
    $\partial^2 \Phi( x ) = \partial^2 \Phi( \Phi( x ) ) + O(\thirdPhi \norm{ x - \Phi ( x ) } )$.

Also, at $\Phi(x)$, since $v_1(x)$ is the top eigenvector of the hessian $\nabla^2 L$, we have from \Cref{cor:nabla2_tangentperp},
\begin{align*}
    & \partial^{2} \Phi( \Phi( x ) )\left[ v_1(x) v_1(x)^{\top} \right] \nonumber =- \frac{1}{ 2 \lambda_1(x) }\partial \Phi( \Phi( x ) ) \partial^{2}(\nabla L) (\Phi(x)) [v_1(x), v_1(x)]. 
\end{align*}

From \Cref{lem:approx_travel1_sqrtL}, we have
\begin{align*}
\norm{\lambda_1(x) v_1(x) v_1(x)^\top - \frac{\nabla L(x) (\nabla L(x))^\top}{2L(x)}} \le \secondL\theta + O( \frac{  \secondL^{3/2} \thirdL  } { \mineigen^{3/2} } \norm{x - \Phi(x)} )
\end{align*}
where recall our notation of $ \theta = \arctan \frac{ \norm{ P_{\Phi(x), \Gamma}^{(2:M)} (x - \Phi(x)) } }{ \abs{ \langle v_1(x), x - \Phi( x ) \rangle } }$.

With further simplification, it turns out that
\begin{align*}
    \Phi( y ) - \Phi( x ) 
    = &-\frac{\eta^2}{8}\partial^{2} \Phi( \Phi( x ) )\left[ \lambda_1(x) v_1(x) v_1(x)^{\top} \right] \\
    + & O(\eta^2 \secondL \secondPhi \theta) + O( \frac{  \secondL^{3/2} \thirdL \secondPhi  } { \mineigen^{3/2} } \norm{x-\Phi(x)} \eta^2 ) + O( \secondL \thirdPhi \norm{x-\Phi(x)} \eta^2 ).
\end{align*}
The proof is completed by using \Cref{lem:riemanniansinglestep_no_log}.
\end{proof}

\section{Additional Experimental Details} \label{sec:add_expt_details}

\subsection{Experimental details}

\paragraph{For \Cref{fig:ngd_traj_3d}: } For running GD on $\sqrt{L}$, we start from $(x, y) = (14.7, 3.)$, and use a learning rate $\eta = 0.5$. For running Normalized GD on $L$, we start from $(x, y) = (14.7, -3)$, and use a learning rate $\eta = 5$.

\paragraph{For \Cref{fig:Quadratic}: } We start Normalized GD from $\langle v_1, \Tilde{x}(0) \rangle = 10^{-4}, \langle v_2, \Tilde{x}(0) \rangle =  0.45$. We use a learning rate of $1$ for the optimization updates. 


\subsection{Implementation Details for Simulation  for the Limiting Flow of Normalized GD}
We provide the code for running a single step of the riemannian flow~\eqref{eq:log_limiting_flow} corresponding to Normalized GD. The pseudocode is given in \Cref{alg:riemmanian_NGD}. 

\paragraph{Loss setting: }  The algorithm described in \Cref{alg:riemmanian_NGD} works for the following scenario. The loss $L$ is equal to the average  of $n$ loss functions $\ell_i:\mathbb{R}^{D} \to \mathbb{R}^{+}$ and each $\ell_i$ is defined as follows. 
Suppose we use $n$ functions $f_i : \mathbb{R}^{D} \to \mathbb{R}$, that share a common parameter $x \in \mathbb{R}^D$, to approximate $n$ true labels $\{b_i \in \mathbb{R} \}_{i=1}^{n}$  
Then, we define each $\ell_i(x) = \overline{\ell}(f_i(x), b_i)$, where $\overline{\ell}: \mathbb{R} \times \mathbb{R} \to \mathbb{R}^{+}$ denotes a general loss function, that takes in the prediction of a function and the true label and returns a score. $\overline{\ell}$ should have the following properties:
\begin{enumerate}
    \item $\overline{\ell}(y, b) \ge 0$ for all $y,b\in\RR$. $\overline{\ell}(y, b) = 0$ iff $y=b$.
    \item $\nabla_y^2 \overline{\ell}(y, b) > 0$  for all $y\in \RR$.
\end{enumerate}
Example of such a loss function $\overline{\ell}$ is the $\ell_2$ loss function.
The scenario described above contains regression tasks. Moreover, it can also represent binary classification tasks, since binary classification can be viewed as regression with $0,1$ label.

We can also represent the multi-class classification tasks, which we use for our experiments. Consider the setting, where we are trying to train some function (e.g. a neural network) $f: \mathbb{R}^{D} \times \mathbb{R}^{d} \to \mathbb{R}^{|C|}$ with the parameter space in $\mathbb{R}^{D}$, input examples from $\mathbb{R}^d$, and the set of classes being denoted by $C$. For each class $c$, we can think of $f_c(x, a)$ as the likelihood score for label $c$ to an input $a$ returned by the function with parameter $x$.
If $S = \{ (a, b) \in \mathbb{R}^{d} \times C \}$ denotes the set of all input and label pairs in the training set, we define our loss function as 
\begin{align*}
    L(x) = \frac{1}{|S| |C|}\sum_{(a, b) \in S} \sum_{c \in C} \abs{ f_c(x, a) - \mathbb{I}(b = c) }^2.
\end{align*}
Thus, each $\ell_i$ in \Cref{alg:riemmanian_NGD} represents one of the terms  $\{ \abs{ f_c(x, a) - \mathbb{I}(b = c) }^2 \}_{ (a, b) \in S, c \in C }$ in the multi-class classification setting.




\paragraph{Riemannian gradient update details: } Each update comprises of three major steps: a) computing $\nabla^3 L(x)[v_1(x), v_1(x)]$, b) a projection onto the tangent space of the manifold, and c) few steps of gradient descent with small learning rate to drop back to manifold. 

To get $v_1(x)$, we follow a power iteration on $\nabla^2 L(x)$. We use $100$ iterations to compute $v_1(x)$ at each point $x$. For the plots in \Cref{fig:MNIST,fig:MNIST_paramdiff}, we run \Cref{alg:riemmanian_NGD} starting from a point $x(0)$ with $L(x(0)) = 3.803 \times 10^{-3}$, with $\eta=10^{-2}$, $\eta_{\mathrm{proj}} = 10^{-2}$, and $T_{\mathrm{proj}} = 10^3.$ We use a learning rate of $10^{-2}$ for the Normalized GD trajectory.


 \renewcommand{\algorithmiccomment}[1]{//#1}

\begin{algorithm}[!htbp]
   \caption{Simulation for the limiting flow~\eqref{eq:log_limiting_flow} of Normalized GD}
   \label{alg:riemmanian_NGD}
\begin{algorithmic}

   \STATE {\bfseries Input:} $n$ loss functions $\ell_i: \mathbb{R}^{D} \to \mathbb{R}^+$, initial point $x(0)$ with $L(x(0)) \approx 0$, maximum number of iteration $T$, LR $\eta$, Projection LR $\eta_{\mathrm{proj}}$, maximum number of projection iterations $T_{\mathrm{proj}}$.
   
   \STATE Define $L(x)$ as $\frac{1}{n} \sum_{i=1}^{n} \ell_i ( x )$   and $P_{ x, \Gamma }$ as projection matrix onto the subspace spanned by $\nabla f_1(x), \cdots, \nabla f_n(x)$ for any $x \in \mathbb{R}^D$.
   \FOR{$t=0$ {\bfseries to} $T-1$}
   \STATE Compute $v_1$, the top eigenvector of $\nabla^2 L( x(t) ).$
    \STATE Compute $\nabla \lambda_1 ( x(t) ) = \nabla^3 L ( x(t) ) [v_1, v_1].$ \qquad \COMMENT{This is by \Cref{lem:derivative_eig}.}
    \STATE Compute $P_{ x(t), \Gamma } \nabla \lambda_1 ( x(t) )$ by solving least square.
    \STATE $y(0) \gets x(t) - \frac{\eta}{\lambda_1 ( x(t) ) } (I-P_{ x(t), \Gamma }) \nabla \lambda_1 ( x(t) )$.
    
    \FOR{$\Tilde{t}=0$ {\bfseries to} $T_{\mathrm{proj}}-1$} 
        \STATE $y(\Tilde{t}+1) = y(\Tilde{t}) - \eta_{\mathrm{proj}} \nabla  L( y(\Tilde{t}) ) $.  \qquad \COMMENT{ Inner loop: project GD back to manifold.}
            \ENDFOR
    \STATE $x(t+1) \gets y(T_{\mathrm{proj}}).$
   \ENDFOR
\end{algorithmic}
\end{algorithm}

\end{document}